\documentclass[sigconf, authorversion, nonacm]{acmart}
\usepackage[show]{macros}
\usepackage{tcolorbox}

\usepackage{graphicx}
\usepackage{tikz}
\usepackage{tikz-network}
\usepackage{graphics}
\usepackage{tikz,pgfplots,pgfplotstable}
\usepackage{mathtools}
\pgfplotsset{compat=newest}

\def\GHOUL{GOSDTwG}

\AtBeginDocument{%
  \providecommand\BibTeX{{%
    \normalfont B\kern-0.5em{\scshape i\kern-0.25em b}\kern-0.8em\TeX}}}

\begin{document}
\setlength{\textfloatsep}{3pt}
\everypar{\looseness=-1}

\title{Fast Optimization of Weighted Sparse Decision Trees for use in \\ Optimal Treatment Regimes and Optimal Policy Design}

\author{Ali Behrouz}
\affiliation{%
  \institution{University of British Columbia}
  \city{Vancouver}
  \state{British Columbia}
  \country{Canada}
}
\email{alibez@cs.ubc.ca}

\author{Mathias L\'ecuyer}
\affiliation{%
  \institution{University of British Columbia}
  \city{Vancouver}
  \state{British Columbia}
  \country{Canada}
}
\email{mathias.lecuyer@ubc.ca}

\author{Cynthia Rudin}
\affiliation{%
  \institution{Duke University}
  \city{Durham}
  \state{North Carolina}
  \country{USA}
}
\email{cynthia@cs.duke.edu}

\author{Margo Seltzer}
\affiliation{%
  \institution{University of British Columbia}
  \city{Vancouver}
  \state{British Columbia}
  \country{Canada}
}
\email{mseltzer@cs.ubc.ca}



\begin{abstract}
Sparse decision trees are one of the most common forms of interpretable models. While recent advances have produced algorithms that fully optimize sparse decision trees for \textit{prediction}, that work does not address \textit{policy design}, because the algorithms cannot handle weighted data samples. Specifically, they rely on the discreteness of the loss function, which means that real-valued weights cannot be directly used. For example, none of the existing techniques produce policies that incorporate inverse propensity weighting on individual data points. We present three algorithms for efficient sparse weighted decision tree optimization. The first approach directly optimizes the weighted loss function; however, it tends to be computationally inefficient for large datasets. Our second approach, which scales more efficiently, transforms weights to integer values and uses data duplication to transform the weighted decision tree optimization problem into an unweighted (but larger) counterpart. Our third algorithm, which scales to much larger datasets, uses a randomized procedure that samples each data point with a probability proportional to its weight.
We present theoretical bounds on the error of the two fast methods and show experimentally that these methods can be two orders of magnitude faster than the direct optimization of the weighted loss, without losing~significant~accuracy. 
\end{abstract}

\maketitle
 
\section{Introduction}
\label{sec:introduction}


Sparse decision trees are a leading class of interpretable machine learning models that are commonly used for policy decisions \citep[e.g.,][]{ernst2005tree,silva2020optimization,dhebar2020interpretable}. Historically, decision tree optimization has involved greedy tree induction, where trees are built from the top down \cite{Quinlan93,breiman1984classification,dobkin1997induction}, but more recently there have been several approaches that fully optimize sparse trees to yield the best combination of performance and interpretability \cite{ farhangfar2008fast, nijssen2007mining, bertsimas2017optimal,gunluk2021optimal}. Optimization of sparse optimal trees is NP-hard, and many previous works have essentially leveraged the fact that the loss takes on a discrete number of values to provide a computational advantage \cite{Aghaei2019fairness, aghaei2020learning, McTavishZhong, lin2020generalized}. However, if one were to try to create a \textit{policy tree} or estimate causal effects using one of these algorithms, it would become immediately apparent that such algorithms are not able to handle weighted data, because the weights do not come in a small number of discrete values. This means that common weighting schemes, such as inverse propensity weighting or simply weighting some samples more than others \citep{linden2018estimating, cieslak2008learning}, are not directly possible with these algorithms.

For example, let us consider developing a decision tree for describing medical treatment regimes. Here, the cost for misclassification of patients in different stages of the disease could be different. To create an optimal policy, we would weight the loss from each patient and minimize the sum of the weighted losses. While it is possible to approximately optimize this sum using CART's suboptimal greedy splitting procedures \citep{breiman1984classification}, there is not a way to do it with the current fastest optimal decision tree method, GOSDT \cite{lin2020generalized}.


We extend the framework of  GOSDT-with-Guesses \citep{McTavishZhong} to support weighted samples. GOSDT-with-Guesses produces sparse decision trees with closeness-to-optimality guarantees in seconds or minutes for most datasets; we refer to this algorithm as \textit{\GHOUL}. Our work introduces three approaches to allow weighted samples, where the first one is slow, the second one is fast, and the third one is fast and scales to large dataset sizes through the use of sampling.

In more detail, a key contributor to \GHOUL's performance is its use of bitvectors to speed up the computation of the loss function.
However, the introduction of weights requires a vector multiplication between the weights and this bitvector representation, which introduces a runtime penalty of one to two orders of magnitude.
We demonstrate this effect in our first (direct) approach. Our second approach introduces a normalization and data duplication technique to mitigate the slowdown due to having real-valued weights. Here, we transform the weights to small integer values and then duplicate each sample by its transformed weights. In a third approach to this problem, which scales to much larger sample sizes, we propose a stochastic procedure, where we sample each data point with a probability proportional to its weight. Our experimental results show that: (1) the second and third techniques decrease run time by up to two orders of magnitude relative to that achieved by the slower direct weighted computation (the direct approach), (2) we can bound our accuracy loss from using the second approach rather than the first one;  and (3) the proposed weighted optimal decision tree technique can outperform natural baselines in terms of running time, sparsity, and accuracy.

\section{Related Work}
\label{sec:RelatedWork}
Decision trees are one of the most popular forms of interpretable models \citep{rudin2021interpretable}. While full decision tree optimization is NP-hard~\cite{laurent1976constructing}, it is possible to make assumptions, e.g., feature independence, that simplify the hard optimization to cases where greedy methods suffice~\cite{klivans2006toward}. However, these assumptions are unrealistic in practice. Some other approaches~\cite{narodytska2018learning, hu2020learning} assume that the data can be perfectly separated with zero error and use SAT solvers to find optimal decision trees; however, real data are generally~not~separable.

Recent work has addressed optimizing accuracy with soft or hard sparsity constraints on the tree size. Such decision tree optimization problems can be formulated using mixed integer programming (MIP)~\cite{bertsimas2017optimal, verwer2019learning,vilas2021optimal,gunluk2021optimal,ErtekinRu18,aghaei2020learning}, but MIP solvers tend to be slow. To improve the scalability of decision tree optimization, several studies have produced customized dynamic programming algorithms that incorporate branch-and-bound techniques. In particular, analytical bounds combined with bit-vector-based computation have been used to efficiently reduce the search space and improve run time \citep{hu2019optimal,AngelinoEtAl2018,ChenRu2018}. \citet{lin2020generalized} extend this approach to use dynamic programming, which leads to even more improved scalability. \citet{demirovic2020murtree} introduce constraints on both depth and the number of nodes to improve scalability. Recently, \citet{McTavishZhong} proposed smart guessing strategies, based on knowledge gleaned from black-box models, that can be applied to any optimal branch-and-bound-based decision tree algorithm to reduce the run time by multiple orders of magnitude. While these studies focus on improving running time and accuracy, they handle only uniform sample importance and do not consider weighted data points. Our work neatly fills this gap; our weighted objective function, data duplication method, and sampling approach enable us to find optimal decision trees for these problems quickly.

Several studies focus on learning tree- and list-based treatment regimes from data \citep{lakkaraju2017learning, zhang2015using, wang2015causal, laber2015tree, CuiEtAL2017, DoubledayEtAl2018, SunWa2021}.  However, none of these methods fully optimize the policy because it was not known at the time how to perform  optimization of the type we use in this work.


\vspace{-1ex}
\section{Methodology}
Let $\{(\x_i, y_i, w_i)\}_{i=1}^N$ represent our training dataset, where $\x_i$ are $M$-vectors of features, $y_i \in \{0,1, \dots, K\}$ are labels, $w_i \in \mathbb{R}^{\geq 0}$ is the weight associated with data $\x_i$, and $N$ is the size of the dataset. Also, let $\x$ be the $N\times M$ covariate matrix, $\w$ be the $N$-vector of weights, and $\y$ be the $N$-vector of labels, and let $x_{ij}$ denote the $j$-th feature of $\x_i$. To handle continuous features, we binarize them either by using all possible split points \citep{hu2019optimal} to create dummy variables or by using a subset of these splits as done by \citet{McTavishZhong}.  
We let $\tilde{\x}$, the binarized covariate matrix, be notated as $\tilde{\x}_{ij} \in \{0, 1\}$.

\subsection{Objective}
Let $\t$ be a decision tree that gives predictions $\{\hat{y}_i^\t\}_{i=1}^N$. The weighted loss of the tree $\t$ on the training dataset is:
\begin{equation}\label{eq:misclassification}
    \mathcal{L}_\w(\t,\tilde{\x},\y) = \frac{1}{\sum_{i = 1}^{N} w_i}\sum_{i=1}^N \mathds{1}[y_i \neq \hat{y}_i^\t] \times w_i \: .
\end{equation}

To achieve interpretability and prevent overfitting, we provide the options to use either soft sparsity regularization on the number of leaves, hard regularization on the tree depth,~or~both~\citep[see][]{McTavishZhong}:
\begin{equation}\label{eq:objective}
    \underset{\t}{\text{minimize}} \:\: \mathcal{L}_\w(\t,\tilde{\x},\y) + \lambda H_\t \quad s.t. \:\: \text{depth}(\t) \leq d \:,
\end{equation}
where $H_\t$ is the number of leaves in the tree $\t$ and $\lambda$ is a regularization parameter. We define $R_{\w}(\t,\tilde{\x},\y) = \mathcal{L}_\w(\t,\tilde{\x},\y) + \lambda H_\t$. We might refer to $\mathds{1}[y_i \neq \hat{y}_i^\t]$ as $I_i(\t)$, for simplicity. While in practice, depth constraints between 2 and 5 are usually sufficient, \citet{McTavishZhong} provide theoretically-proven guidance to select a depth constraint so that a single tree has the same expressive power (VC dimension) as an ensemble of smaller trees (e.g., a random forest or a boosted decision tree model). The parameter $\lambda$ trades off between the weighted training loss and the number of leaves in the tree.

\subsection{Learning Weighted Trees}
We present three approaches for handling sample weights. The first is the \emph{direct approach}, where we calculate the weighted loss directly. Implementing this approach requires multiplying each misclassification by its corresponding weight, which is computationally expensive in any algorithm that uses bitvectors to optimize loss computation. This overhead is due to replacing fast bitvector operations with slower vector multiplications. The direct approach slows \GHOUL{} down by two orders of magnitude.
To avoid this computational penalty, our second approach, \emph{data-duplication}, involves a transformation of the weights; specifically, we normalize, scale, and round the weights to be small integer values. We then duplicate samples, where the number of duplicates is the value of the rounded weights, and use this larger unweighted dataset to learn the tree. This method avoids costly vector multiplications and does not substantially increase run time compared to the unweighted \GHOUL{}; note that \GHOUL{} scales extremely well with the sample size due to the bit-vector computations, so duplication does not add much to the computational cost. Finally, to scale to larger datasets, we present a randomized procedure, called \emph{weighted sampling}, where we sample each data point with a probability proportional to its weight. This process introduces variance (not bias) and scales to large numbers of samples.

\head{Direct Approach}
We begin with the branch-and-bound algorithm of~\citet{McTavishZhong} and adapt it to support weighted samples. Given a reference model $T$, they prune the search space using three ``guessing'' techniques: \textbf{(1)} guess how
to transform continuous features into binary features, \textbf{(2)} guess tree depth for sparsity-regularized models,
and \textbf{(3)} guess tight lower bounds on the objective for subsets of points to allow faster time-to-completion. It is straightforward to see that the first two techniques apply directly to our weighted loss function. However, we need to adapt the third guessing technique to have an effective and tight lower bound for the weighted loss function. Let $\hat{y}_i^T$ be the predictions of the reference model (perhaps a boosted decision tree model) on training observation $i$. Let $s_a$ be the subset of training observations that satisfy a boolean assertion $a$: 
\begin{align*}
    s_{a} &:= \left\{i:a(\tilde{\x}_i) = \textrm{True}, i \in \{1,...,N\}  \right\} \\
    \tilde{\x}({s_{a}}) &:= \left\{\tilde{\x}_i: i \in s_{a} \right\}\\
    \y({s_{a}}) &:= \left\{y_i: i \in s_{a} \right\}\\
    \w({s_{a}}) &:= \left\{w_i: i \in s_{a} \right\}.
\end{align*}
Motivated by \citet{McTavishZhong}, we define our guessed lower bound as follows: 
\begin{equation}\label{eq:lbguess}
lb_\textrm{guess}(s_a):=\frac{1}{\sum_{i = 1}^{N} w_i}\sum_{i \in s_a} \mathds{1}[y_i \neq \hat{y}_i^T] \times w_i + \lambda \:.
\end{equation}
\autoref{eq:lbguess} is a lower bound guess for $R_{\w}(t,\tilde{\x}(s_a),\y(s_a))$ because we assume that the (possibly black box) reference model $T$ has loss at most that of tree $t$ on data $s_a$ and we know that any tree has at least one node (hence the regularization term's lower~bound~of~$\lambda\times 1$).

Accordingly, in the branch-and-bound algorithm, to optimize the weighted loss function introduced in \autoref{eq:objective}, we consider a subproblem to be solved if we find a subtree that achieves an objective less than or equal to its $lb_\textrm{guess}$. If we find such a subtree, our training performance will be at least as good as that of the reference model. For a subset of observations $s_a$, we let $t_{a}$ be the subtree used to classify points in $s_a$, and $H_{t_{a}}$ be the number of leaves in that subtree. We can then define the subset's contribution to the objective: 
\begin{align*}
&R_{\w(s_a)}(t_a,\tilde{\x}(s_a), \y(s_a)) \\&= \frac{1}{\sum_{i = 1}^{N} w_i}\sum_{i \in s_a} \mathds{1}[y_i \neq \hat{y}_i^{t_{a}}] \times w_i + \lambda H_{t_{a}}.
\end{align*}
For any dataset partition $A$, where $a\in A$ corresponds to the data handled by a given subtree of $t$:
\begin{equation*}
R_{\w}(t, \tilde{\x}, \y) = \sum_{a \in A} R_{\w(s_a)}(t_a, \tilde{\x}(s_a), \y(s_a)) \:. 
\end{equation*}
By introducing the abovementioned lower bound guess, we can now replace the lower bound of \citet{McTavishZhong} with our lower bound and proceed with branch-and-bound. Their approach is provably close to optimal when the reference model makes errors similar to those made in an optimal tree. We now show that our approach using the weighted lower bound is also close to optimal. Let $\MC$ be the set of observations incorrectly classified by the reference model $T$, i.e., $\MC = \{i |y_i \neq \hat{y}_i^T\}$, and $t_{\textrm{g}}$ be a tree returned from our lower-bound guessing algorithm. We have:

\begin{theorem} \label{thm:glb} (Performance Guarantee).
Let $R(t_{\textrm{g}}, \tilde{\x}, \y)$ denote the objective of $t_\textrm{g}$ on the full binarized dataset $(\tilde{\x}, \y)$ for some per-leaf penalty $\lambda$. Then for any decision tree $t$ that satisfies the same depth constraint $d$, we have: \begin{align*}R(t_{\textrm{g}}, \tilde{\x}, \y) \leq &\frac{1}{\sum_{i = 1}^{N} w_i} \left(\sum_{i \in \MC} w_i \right. \\ &\left.+ \sum_{i \in \CC} \mathds{1}[y_i \neq \hat{y}_i^{t}] \times w_i \right)  + \lambda H_{t} \:.\end{align*}
That is, the objective of the guessing model $t_{\textrm{g}}$ is no worse than the union of errors made by the reference model and tree $t$.
\end{theorem}

Hence, the model $t_{\textrm{g}}$ from our lower bound guessing algorithm achieves a weighted objective that is as good as the error of the reference model (which should be small) plus (something smaller than) the error of the best possible tree of the same depth. The proof can be found in Appendix~\ref{app:theorem_1}.

\head{Motivation for Data Duplication} While it might seem counter-intuitive that increasing the size of the dataset by replicating data would be substantially faster than the direct approach, it actually is. Let us consider the computation involved in these two options. In decision tree optimization, evaluation of the objective is performed repeatedly. Any small improvement in that evaluation step leads to a large improvement (possibly orders of magnitude) in the overall computational speed of the algorithm.
In the direct approach, computing the objective~\eqref{eq:objective} requires computing the inner product  $\w \cdot \mathcal{I}$, where $\mathcal{I}_i = \mathds{1}[y_i \neq \hat{y}_i^\t]$. In the unweighted case, as all weights are 1, this computation can be performed using bit operations, which are extremely fast. In the weighted case, we resort to standard inner products, which are two orders of magnitude slower (see Section \ref{sec:experiments}). The data-duplication approach allows us to use bit-vectors as in the unweighted case, preserving fast computation.

\begin{algorithm}[t] 
    \small
    \SetKwInOut{Input}{Input}
    \SetKwInOut{Output}{Output}
    \Input{Dataset $\mathbf{x},\y$ and weights $\w$, duplication factor $p<100$}
    \Output{Duplicated dataset $\tilde{X}$, $\tilde{y}$}
    $\tilde{X} \leftarrow \emptyset; \tilde{y} \leftarrow \emptyset$;\\
    Define $\tilde{w}_i = \textrm{round}\left(p\cdot \left(\frac{w_i}{\sum_{i = 1}^{N} w_i}\right)\right)$;\\
    \For{$x_i \in \x$}{
        \For{$i = 1, 2, \dots, \tilde{w}_i$}{
            $\tilde{X} \leftarrow \tilde{X} \cup {\{x_i\}}$;\\
            $\tilde{y} \leftarrow \tilde{y} \cup {\{y_i\}}$;
        }
    }
    \Return $\tilde{X}$, $\tilde{y}$
    \caption{Data Duplication}
    \label{alg:data_duplication}
\end{algorithm}

\head{Data-duplication Algorithm}
The data-duplication algorithm is shown in Algorithm~\ref{alg:data_duplication}. Given an integer $p > 0$, we first normalize all weights and scale them to $(0, 1]$. Then we multiply each normalized weight by $p$ and round them to integers. Given the scaled integer weights, we then duplicate each sample, $\x_i$, by its corresponding integer weight, $\hat{w}_i$. Once the data are duplicated, \emph{we can use any optimal decision tree technique}. In our experiments, we show that if we choose the value of $p$ appropriately, this method can speed up the training process significantly without losing too much accuracy. After data-duplication, there are no longer weights associated with the samples, and we can use the fast bit-vector computations from the unweighted case. Even in the presence of substantially more data, since the bit-vector computation scales sublinearly with the number of samples, it is much faster than the direct approach.

\head{Correctness of Data Duplication} 
One might ask whether the data duplication approach leads to suboptimal solutions because its loss function is an approximation to the weighted loss. As it turns out, as long as the weights do not change very much when rounding to integers, the minimum of the data duplication algorithm's objective is very close to the minimum of the original weighted objective. 

Recall the objective 
\[R(t) := \frac{1}{\sum_{i = 1}^{N} w_i}\sum_i w_i I_i(t) + \lambda \#\textrm{leaves}.\] Define the objective with the approximate weights as
\[\Rt(t) := \frac{1}{\sum_{i = 1}^{N} \tilde{w}_i}\sum_i \tilde{w}_i I_i(t) + \lambda \#\textrm{leaves}.\] When we rounded, we ensured that the weights did not change much. That is, we know by design that $\|\mathbf{w}-\mathbf{\tilde{w}}\|_\infty\leq \epsilon$. Note that multiplying $w_i$s by a scalar cannot change the value of the objective function. Accordingly, normalizing and also scaling weights by $p$ do not change the value of $R(t)$. Therefore, without loss of generality, we can assume that $w_i$s are weights right before rounding.

\begin{theorem} Let $t^*$ be a minimizer of the objective as $t^* \in \argmin_t R(t)$, and $\tt$ be a minimizer of the approximate loss function as $\tt\in \argmin_t \Rt(t)$. If $\|\mathbf{w}-\mathbf{\tilde{w}}\|_\infty\leq \epsilon$, we have:
\begin{equation*}
|R(t^*) - \Rt(\tt)|\leq \\max\left\{ \frac{(\zeta - 1) \psi + \epsilon}{\zeta}, \frac{(\eta - 1) \psi + \epsilon}{\eta} \right\},
\end{equation*}
where $\eta = \max_{1 \leq i \leq N} \left\{\frac{w_i}{\tilde{w}_i} \right\}$, $\zeta = \max_{1 \leq i \leq N} \left\{\frac{\tilde{w}_i}{w_i} \right\}$, and $\psi = \frac{\max_i\left\{w_i, \tilde{w}_i\right\}}{\min_i \left\{ w_i, \tilde{w}_i \right\}}$.
\end{theorem}
Note that, in other words, we provably will not lose substantial performance when using the rounded solution, as long as we did not change the weights very much when rounding. The value of $\eta$ and $\zeta$ are usually small and near 1. If the value of $\psi$ is very large, then the direct approach is more efficient, and we do not need to use data duplication. Accordingly, when we use data duplication, the value of $\psi$ should also be small. The proof is in Appendix~\ref{app:theorem_2}.

\head{Weighted Sampling}
When the ratio of the biggest weight over the smallest weight is large, the data duplication approach might be inefficient, in that it could require us to create a huge number of samples. To address this issue, we present a stochastic sampling process as a pre-processing step.  We sample each data point based on its weight. Given an arbitrary number $r$, we sample $S = r \times N$ data points such that the probability of choosing $\x_i$ is $\frac{\w_i}{\sum_{i = 1}^{N} \w_i}$. After this step, we can use any unweighted optimal decision tree algorithm on the sampled dataset.

\head{Quality Guarantee of Weighted Sampling}
Let $\tilde{\mathcal{L}}(.)$ be the loss function on the sampled dataset, it is not hard to see that $\mathbb{E}[\tilde{\mathcal{L}}] = \mathcal{L}_{\w}$, where $\mathcal{L}_{\w}$ is the value of the misclassification (Eq.~\ref{eq:misclassification}) on the weighted dataset. Based on this fact, we have the following theorem:

\begin{theorem}\label{thm:weighted_sampling}
Given a weighted dataset $D = \{(\x_i, y_i, w_i)\}_{i=1}^N$, an arbitrary positive real number $r > 0$, an arbitrary positive real number $\varepsilon > 0$, and a tree $\t$, if we sample $S = r \times N$ data points from $D$, $\tilde{D} = \{(\tilde{\x}_i, \tilde{y}_i)\}_{i=1}^S$, we have:
\begin{equation*}
    \mathbb{P}\left( | \tilde{\mathcal{L}}(\t, \tilde{\x}, \tilde{\y}) -  \mathcal{L}_{\w}(\t, \x, \y)| \geq \varepsilon \right) \leq 2 \exp\left( - \frac{2\varepsilon^2}{S} \right)
\end{equation*}
\end{theorem}

\section{Experiments}
\label{sec:experiments}

Our evaluation addresses the following questions: 
\begin{enumerate}
 \item When is the direct approach more efficient than data-duplication and weighted sampling?    
 \item In practice, how well do the second and third proposed methods perform relative to the direct approach?
 \item How sparse and fast are our weighted models relative to state-of-the-art optimal decision trees?
 \item How can our approach be used for policy making?

\end{enumerate}

 We use sparsity as a proxy for interpretability, because it can be quantified thus providing an objective means of comparision~\cite{rudin2021interpretable}.
\begin{table}[t]
    \small
    \centering
    \caption{Datasets}
    \resizebox{0.35\textwidth}{!} {
    \centering
    \begin{tabular}{l|c|c|c}
    \toprule
    Dataset & samples & features & binary features \\\hline\hline
    Lalonde & 723 & 7 & 447  \\\hline
    Broward & 1954 & 38 & 588
    \\\hline
    Coupon & 2653 & 21 & 87  \\ \hline
    Diabetes & 5000 & 34 & 532
    \\\hline
    COMPAS & 6907 & 7 & 134  \\\hline
    FICO & 10459 & 23 & 1917 \\\hline
    Netherlands & 20000 & 9 & 53890 \\
    \bottomrule
    \end{tabular}
    }
    \label{tab:data}
\end{table}

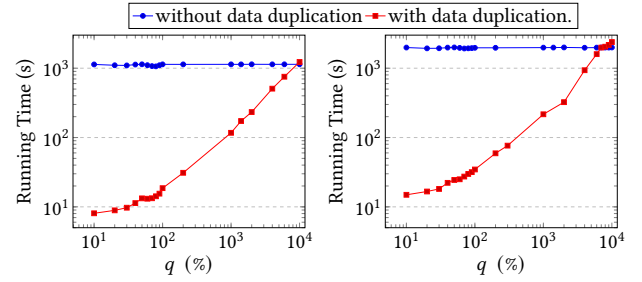
\begin{figure}[t]
\begin{center}
\begin{tikzpicture}[scale=0.45]
\begin{axis}[
    font=\Huge,
    ylabel={Running Time (s)},
    xlabel={$q \:\: (\%)$},
    ymode=log,
    xmode=log,
    log origin=infty,
    xmin=0, xmax=12000,
    ymin=5, ymax=3000,
    ymajorgrids=true,
    grid style=dashed,
    legend columns=2,
    legend style={at={(0.24, 1.1)},anchor=west},
]

\addplot
coordinates {(0, 1134.83) (10, 1133.56) (20, 1102.95) (30, 1096.84) (40, 1128.87) (50, 1138.19) (60, 1106.01) (70, 1072.92) (80, 1063.44) (90, 1103.79) (100, 1136.62) (200, 1136) (1000, 1137) (1400, 1136) (2000, 1139) (4000, 1137) (6000, 1138) (10000, 1138)};
\addlegendentry{without data duplication}

\addplot
coordinates {  (0, 7.916) (10, 8.072) (20, 8.9) (30, 9.72) (40, 11.321) (50, 13.26) (60, 13.1) (70, 13.3) (80, 14.2) (90, 15.5) (100, 18.6) (200, 31) (1000, 117) (1400, 173) (2000, 233) (4000, 507) (6000, 754) (10000, 1229)};
\addlegendentry{with data duplication.}

\end{axis}
\end{tikzpicture}
\hspace{-26ex}
\begin{tikzpicture}[scale=0.45]
\begin{axis}[
    font=\Huge,
    ylabel={Running Time (s)},
    xlabel={$q \:\: (\%)$},
    ymode=log,
    xmode=log,
    log origin=infty,
    xmin=0, xmax=12000,
    ymin=5, ymax=3000,
    ymajorgrids=true,
    grid style=dashed,
]

\addplot
coordinates {(0, 1985.9525) (10, 1986.73) (20, 1933.7625) (30, 1939.5575) (40, 1982.53) (50, 1994.8425) (60, 1963.5175) (70, 1924.87) (80, 1936.2874) (90, 1949.045) (100, 1969.82) (200, 1970.5) (300, 0) (400, 0) (500, 0) (600, 0) (700, 0) (800, 0) (900, 0) (1000, 1981.7173) (1400, 1984.88) (2000, 1993.61) (4000, 1974.44) (5000, 0) (6000, 1982.9425) (7000, 1991.25) (8000, 1986.49) (9000, 1984.675) (10000, 1990.76)};

\addplot
coordinates {(0, 14.8) (10, 14.92) (20, 16.65) (30, 18.13) (40, 22.2) (50, 24.42) (60, 24.975) (70, 27.345) (80, 29.825) (90, 31.675) (100, 34.595) (200, 59.2) (300, 76.39) (1000, 216.45) (2000, 323.75) (3000, 0) (4000, 935.176) (5000, 0) (6000, 1595.27) (7000, 1986.05) (8000, 2043.84) (9000, 2169.37) (10000, 2383.34)};

\end{axis}
\end{tikzpicture}
\vspace{-2ex}
\caption{Training time of the model with and without data duplication on different machines.}
\label{fig:run_time_p}
\end{center}
\end{figure}

\begin{figure*}
\includegraphics[width=0.245\textwidth]{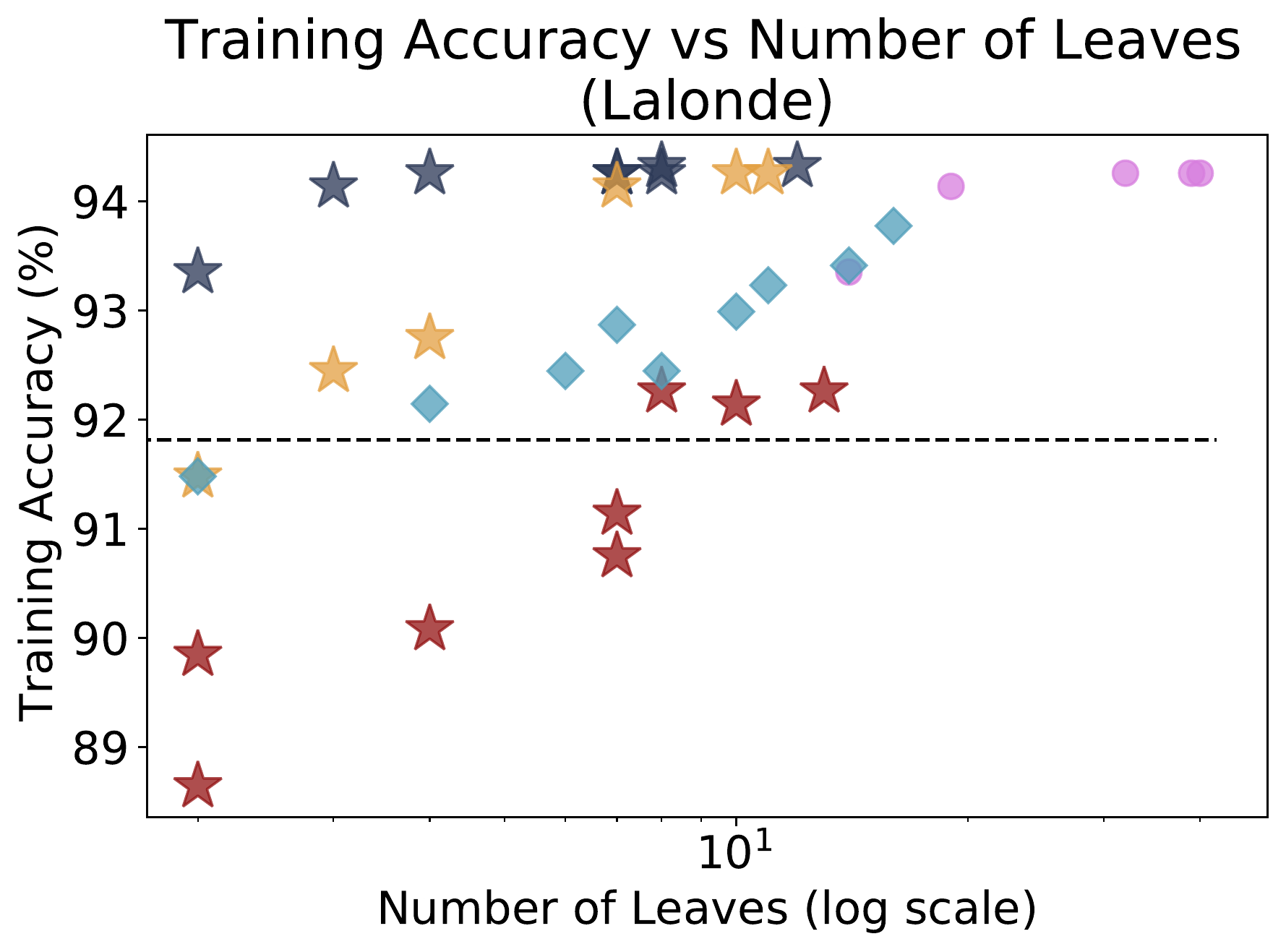}
\includegraphics[width=0.245\textwidth]{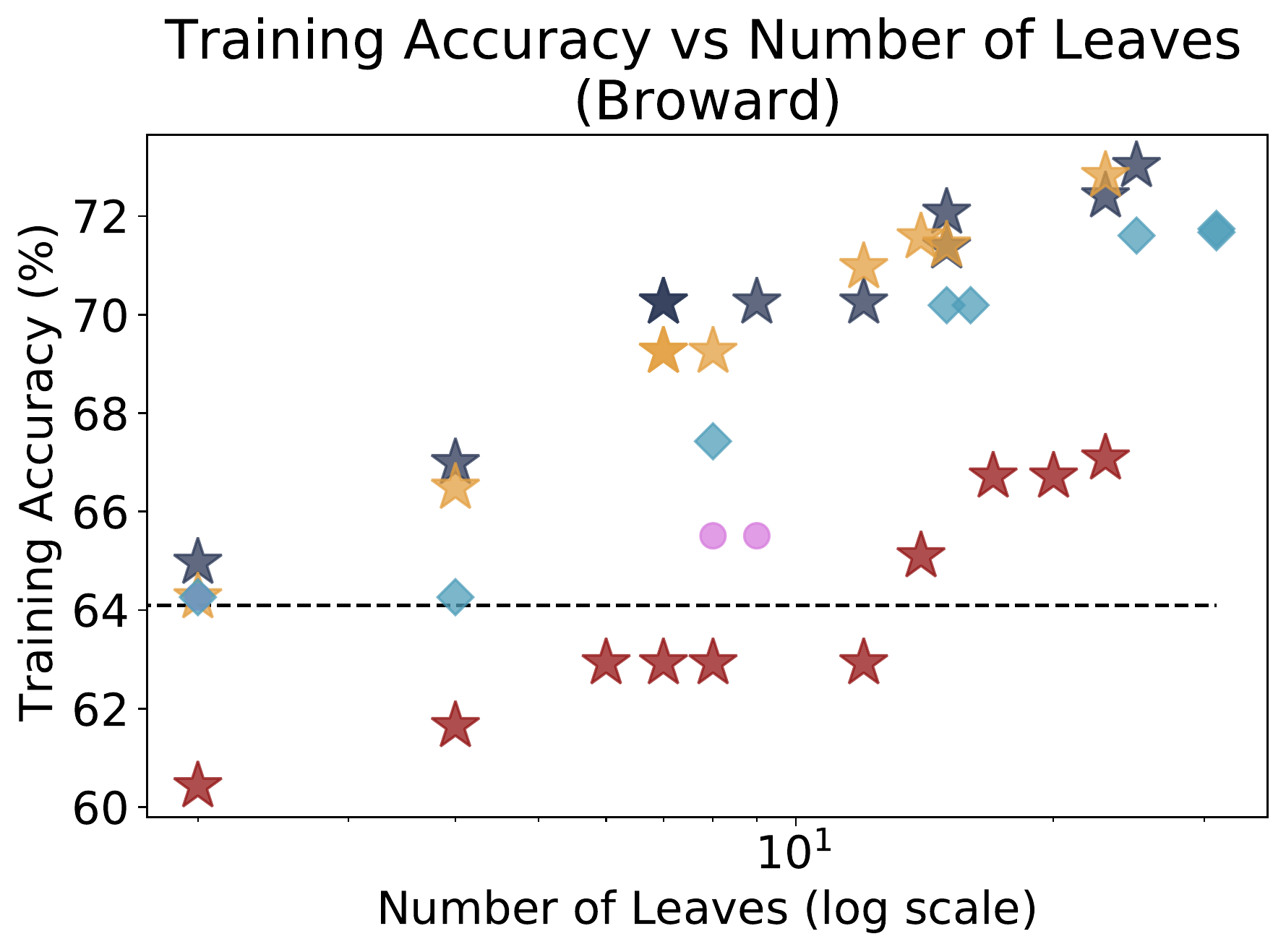}
\includegraphics[width=0.245\textwidth]{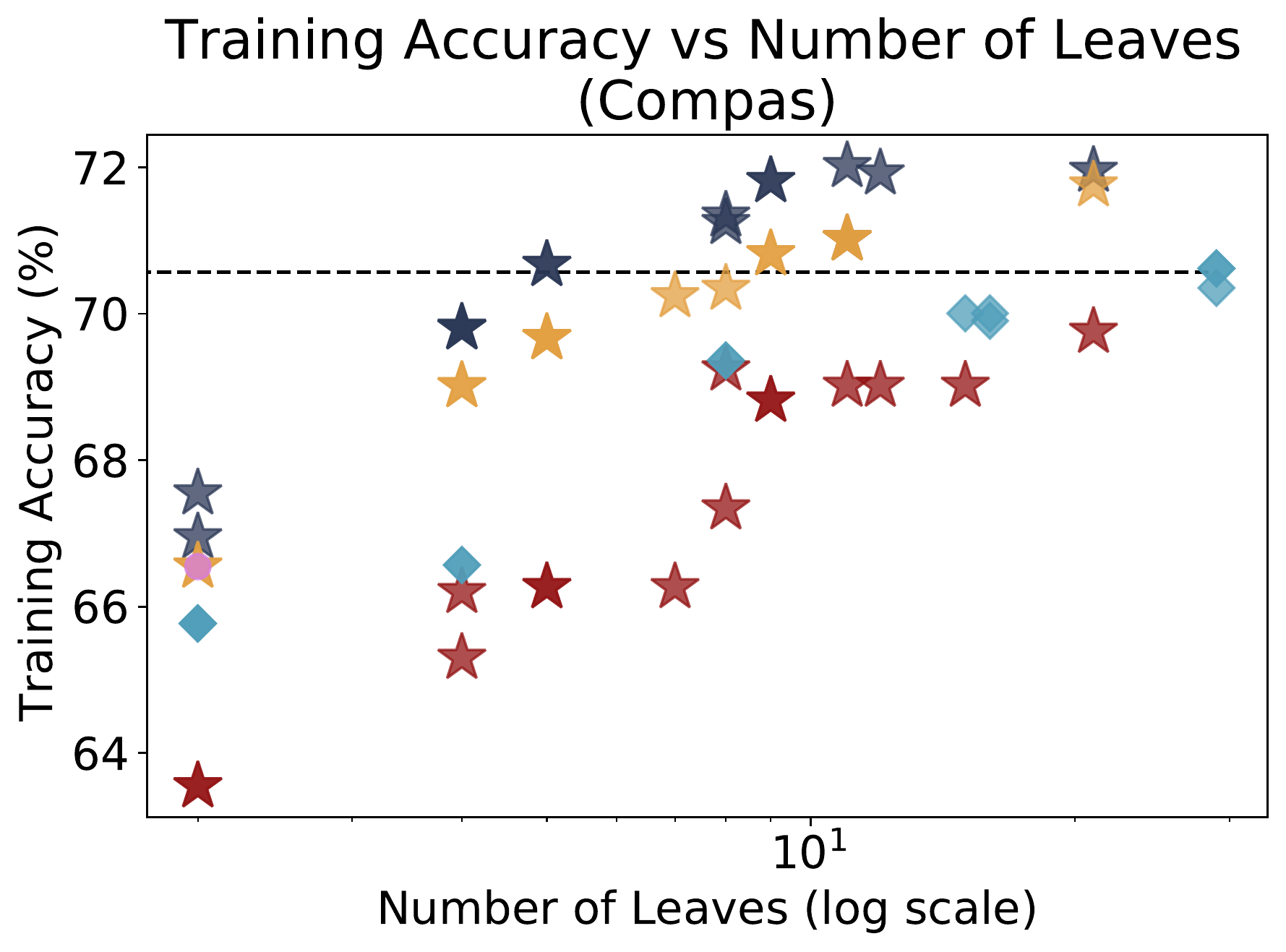}
\includegraphics[width=0.245\textwidth]{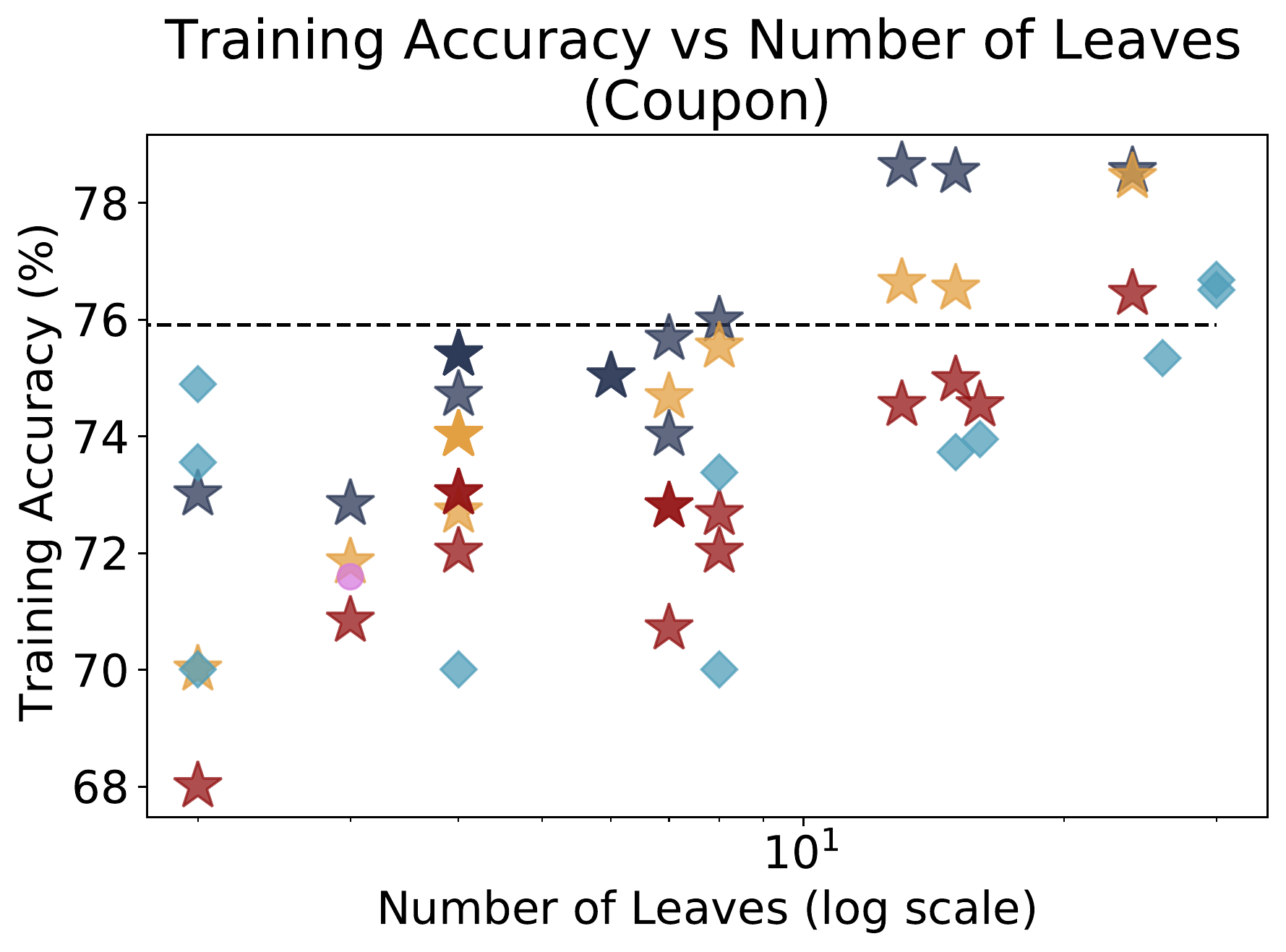}\\
\includegraphics[width=0.245\textwidth]{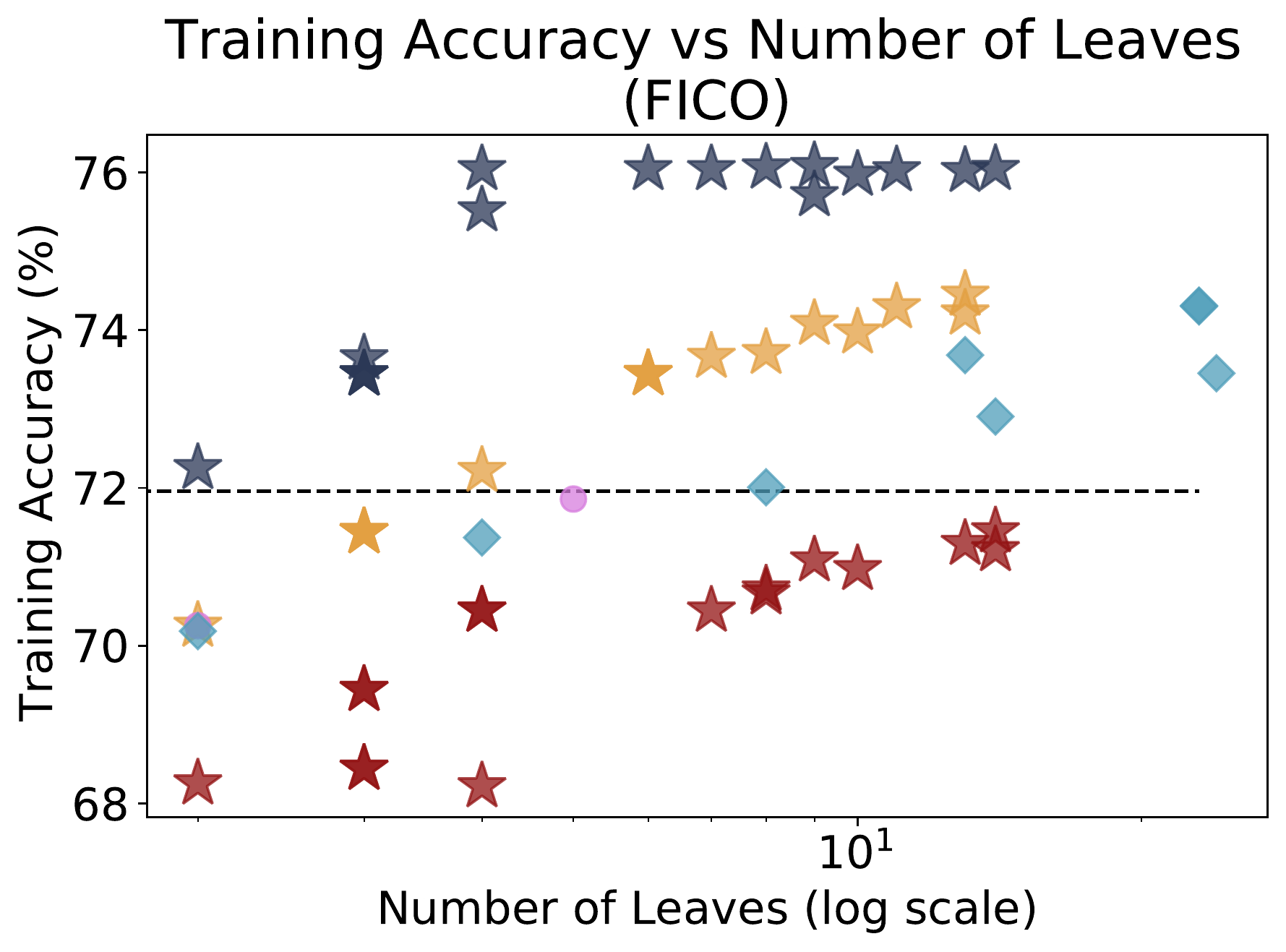}
\includegraphics[width=0.245\textwidth]{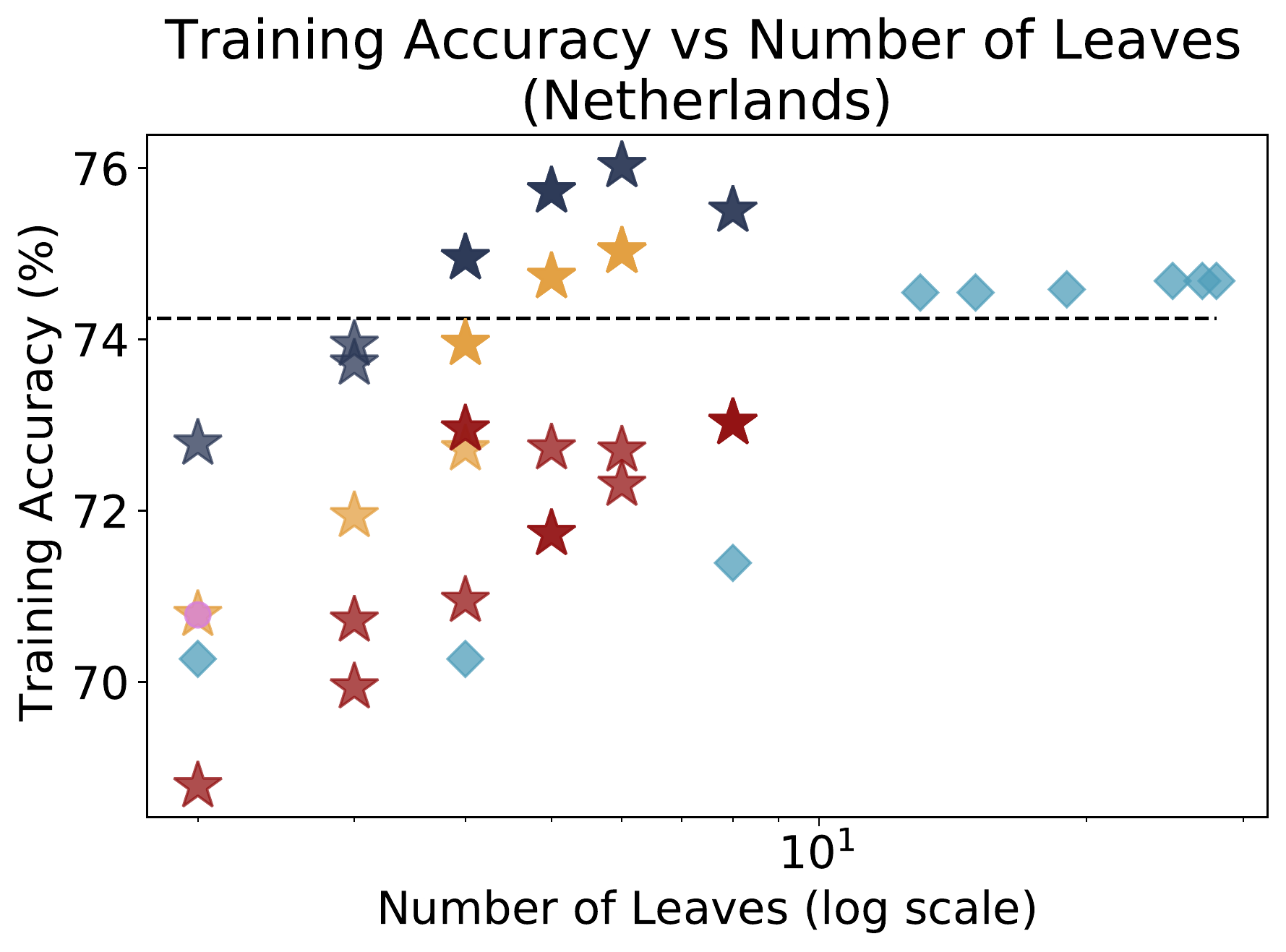}
\includegraphics[width=0.245\textwidth]{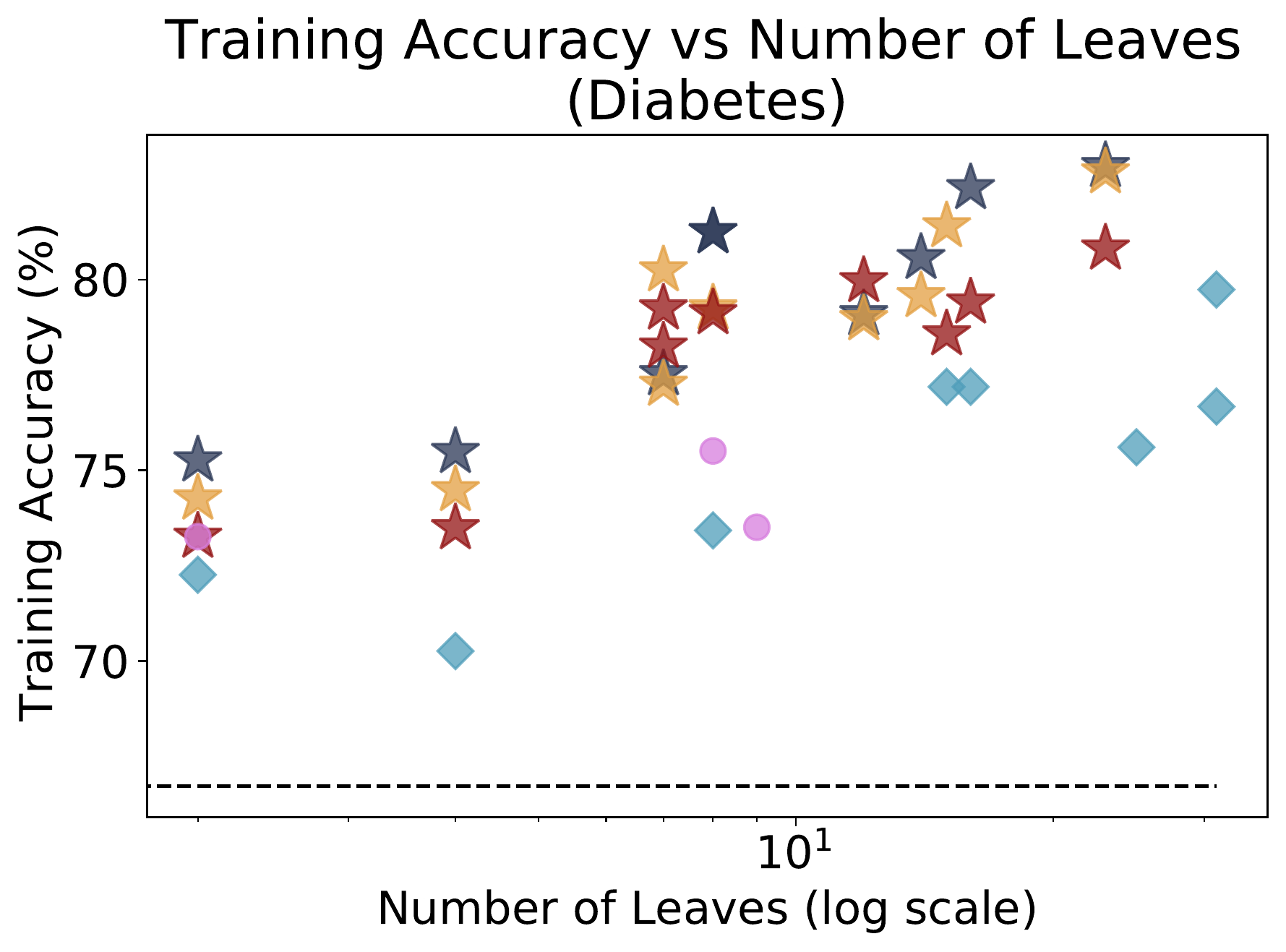}~
\includegraphics[width=0.245\textwidth]{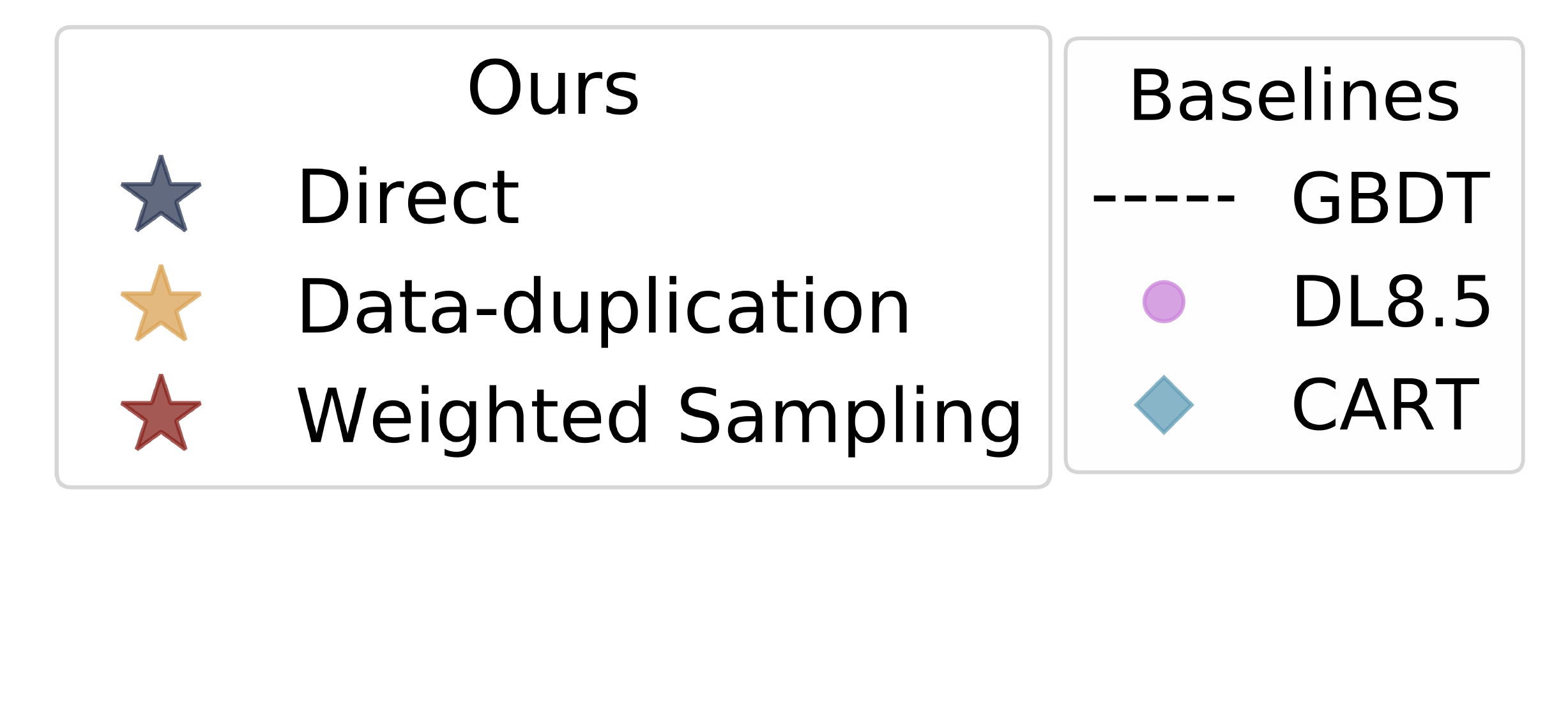}
\vspace{-2ex}
\caption{\textbf{Sparsity vs. training accuracy: } All methods but CART and GBDT use guessed thresholds. GBDT and DL8.5 use data duplication. DL8.5 frequently times out, so there are fewer markers for it. \GHOUL{} achieves the highest accuracy for every level of sparsity.}
\label{fig:sparsity-train-accuracy}
\end{figure*}

\begin{figure*}
\includegraphics[width=0.245\textwidth]{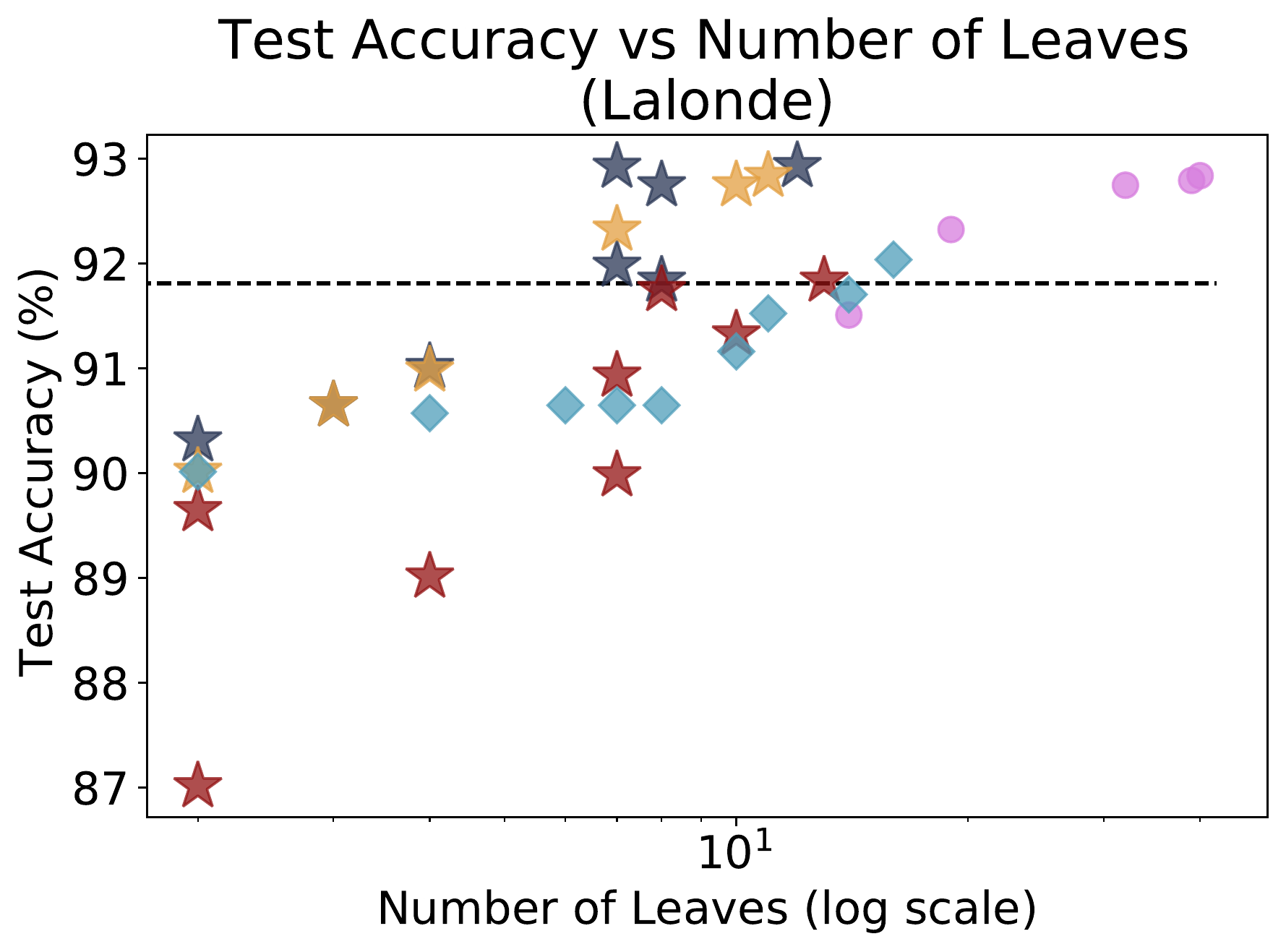}
\includegraphics[width=0.245\textwidth]{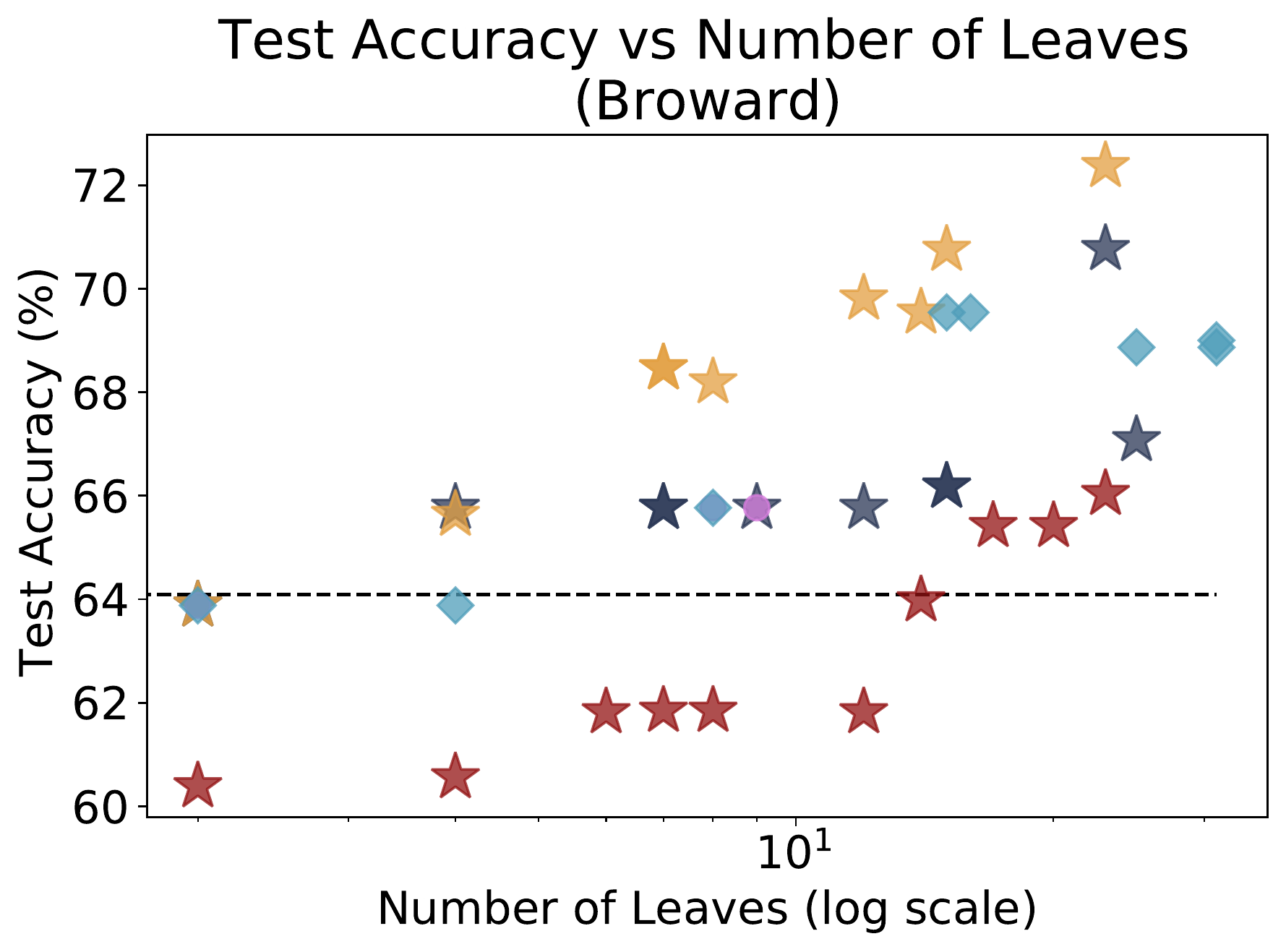}
\includegraphics[width=0.245\textwidth]{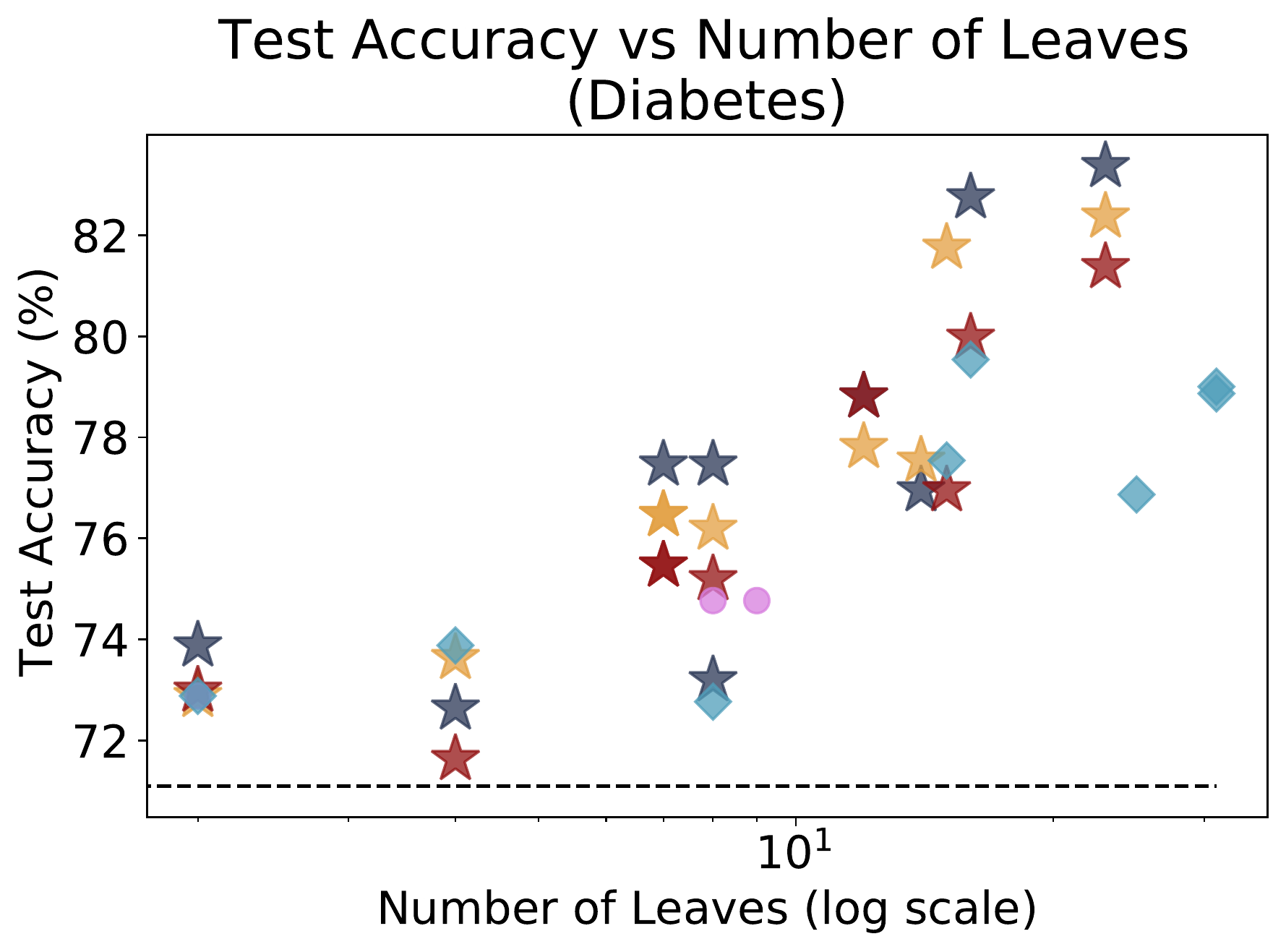}
\includegraphics[width=0.245\textwidth]{legend.pdf}
\vspace{-2ex}
\caption{\textbf{Sparsity vs. test accuracy: } All methods but CART and GBDT use guessed thresholds. GBDT and DL8.5 use data duplication. DL8.5 frequently times out, so there are fewer markers for it. \GHOUL{} achieves the highest accuracy for every level of sparsity.\GHOUL{} also achieves the highest test accuracy for \textbf{\textit{almost}} every level of sparsity. }
\label{fig:sparsity-test-accuracy}
\end{figure*}

\subsection{Datasets}
In our experiments, we use seven publicly available real-world datasets; Table~\ref{tab:data} shows sizes of these datasets: The Lalonde dataset \cite{Lalonde_reference, dehejia1999causal}, Broward~\cite{wang2020pursuit}, the coupon dataset, which was collected on Amazon Mechanical Turk via a survey~\cite{wang2015or}, Diabetes~\cite{Diabetes_dataset}, which is a health care related dataset, the COMPAS~\cite{LarsonMaKiAn16}, the Fair Isaac (FICO) credit risk dataset~\cite{competition} from the Explainable ML Challenge, and Netherlands~\cite{tollenaar2013method} datasets, which are recidivism datasets. Unless stated otherwise, we use inverse propensity score with respect to one of the features as our weights. For more details about datasets and weights see Appendix~\ref{app:datasets}.

For each dataset, we ran the experiments with different depth bounds and regularization, and each point in each plot shows the results for one setting. Table \ref{tab:tree_config} (Appendix~\ref{app:exp_detail}) lists the configurations used for each dataset when training decision trees.

\subsection{Baselines}
We compared our proposed methods with the following baseline models: (1) CART~\cite{breiman1984classification}, (2) DL8.5~\cite{aglin2020learning}, and (3) Gradient Boosted Decision Trees (GBDT)~\citep{freund95, friedman2001greedy}.
CART and GBDT can both handle weighted datasets, so we use their default weighted implementation as the baselines. As DL8.5 does not supported weighted datasets, we use the data-duplication approach with it.

\subsection{Results}

\head{Data duplication}
We begin by demonstrating how much the direct approach penalizes runtime relative to the data-duplication approach.
We use the unweighted FICO dataset and randomly pick $q\%$ of the data points, $\mathcal{S}$,  to double weight, duplicate them and add them to the data set, producing a dataset of size $(1 + \frac{q}{100}) \times N$ , where $N$ is the size of the original data set, $|S|$. We then compare this to the running time of the direct approach in which we assign weight of 2 to all double-weighted samples and weight of 1 to all remaining samples. We run this experiment on two different machines, with different processors and RAMs, in order to show the reliability of the results on different machines. The detailed property of machines can be found in Appendix~\ref{app:machine_property}. Figure~\ref{fig:run_time_p} shows the result of this experiment. We find that when the size of the duplicated dataset is less than 100 times the original dataset, the data-duplication approach is always faster.

\begin{figure*}
\includegraphics[width=0.245\textwidth]{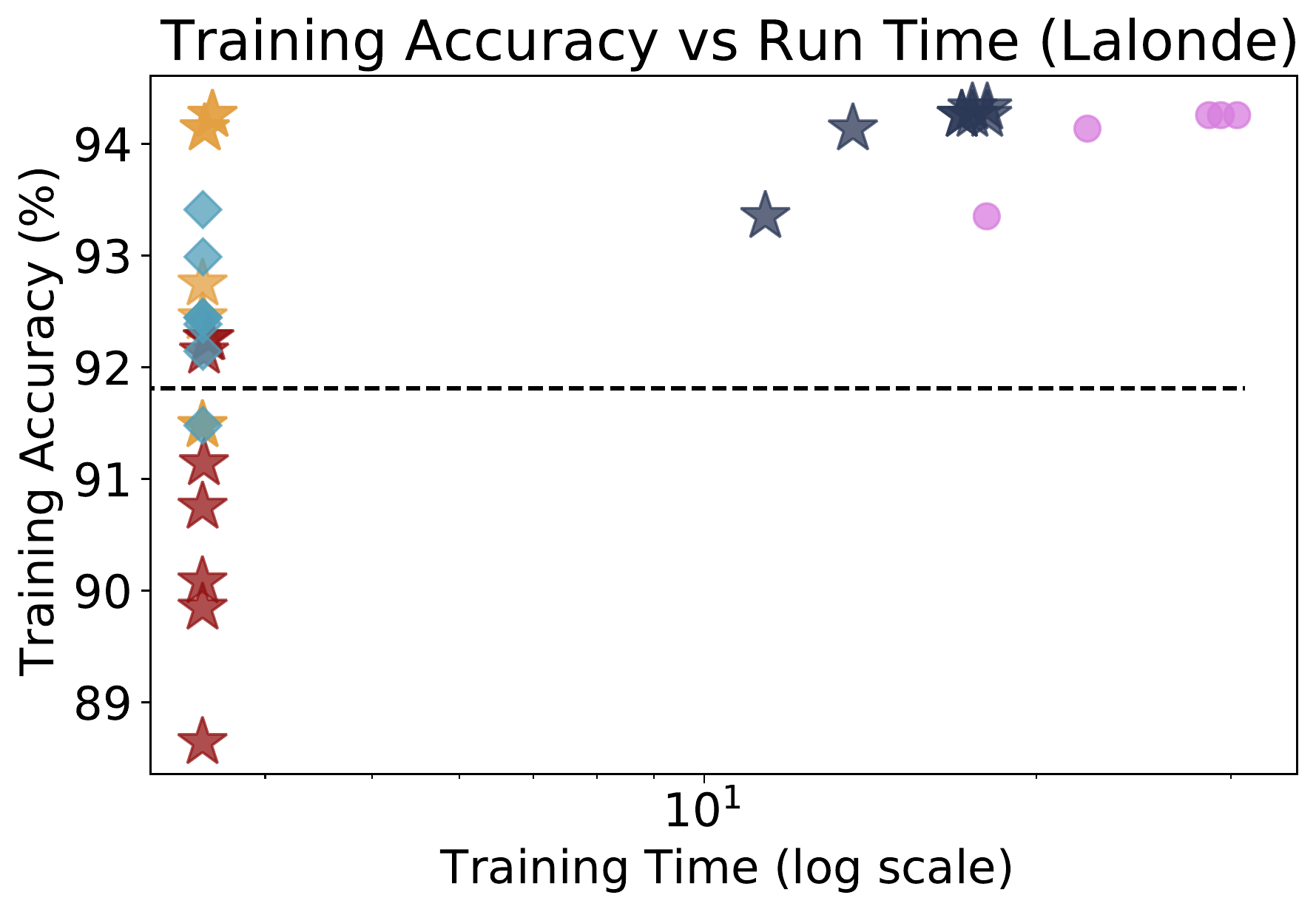}
\includegraphics[width=0.245\textwidth]{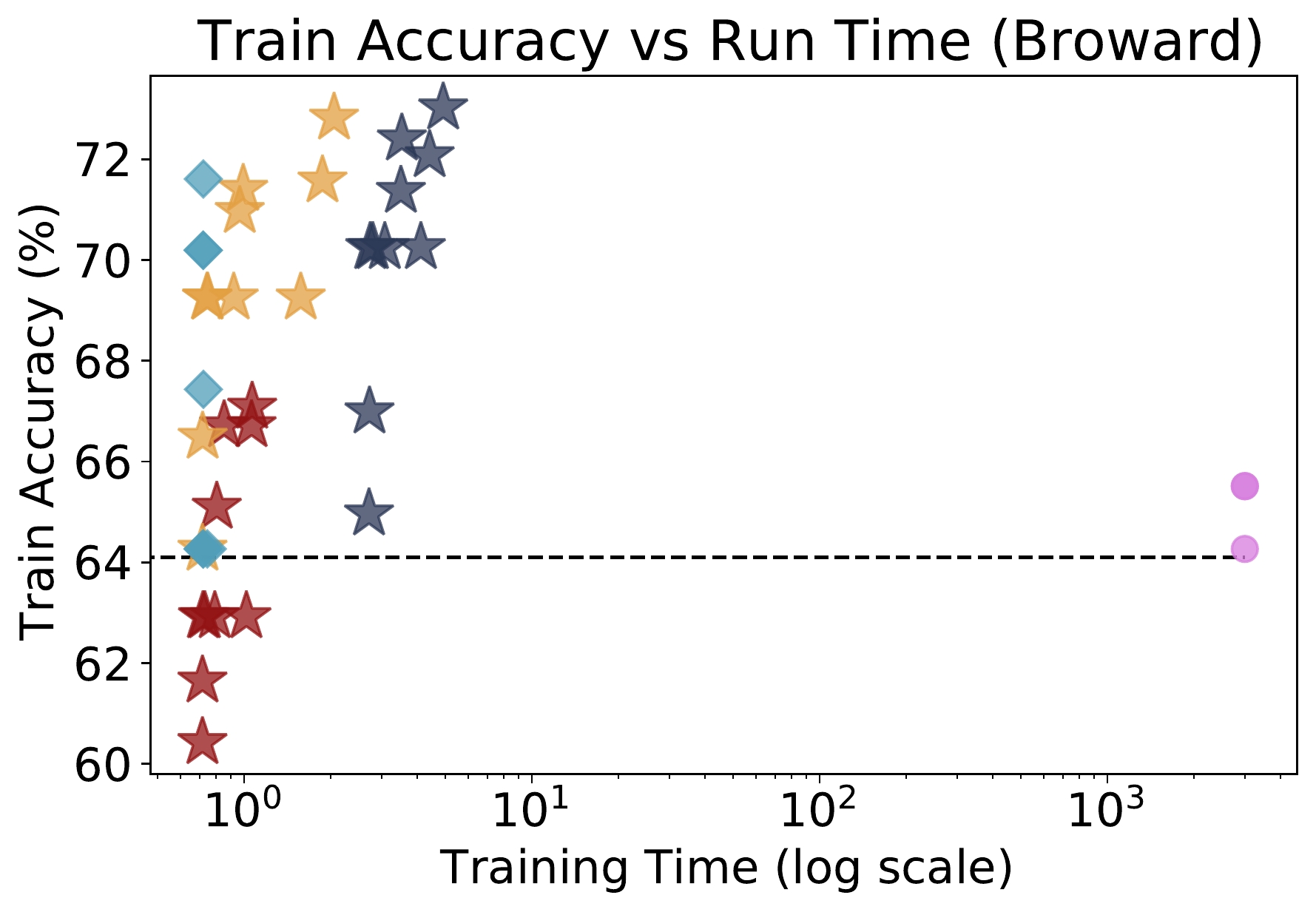}
\includegraphics[width=0.245\textwidth]{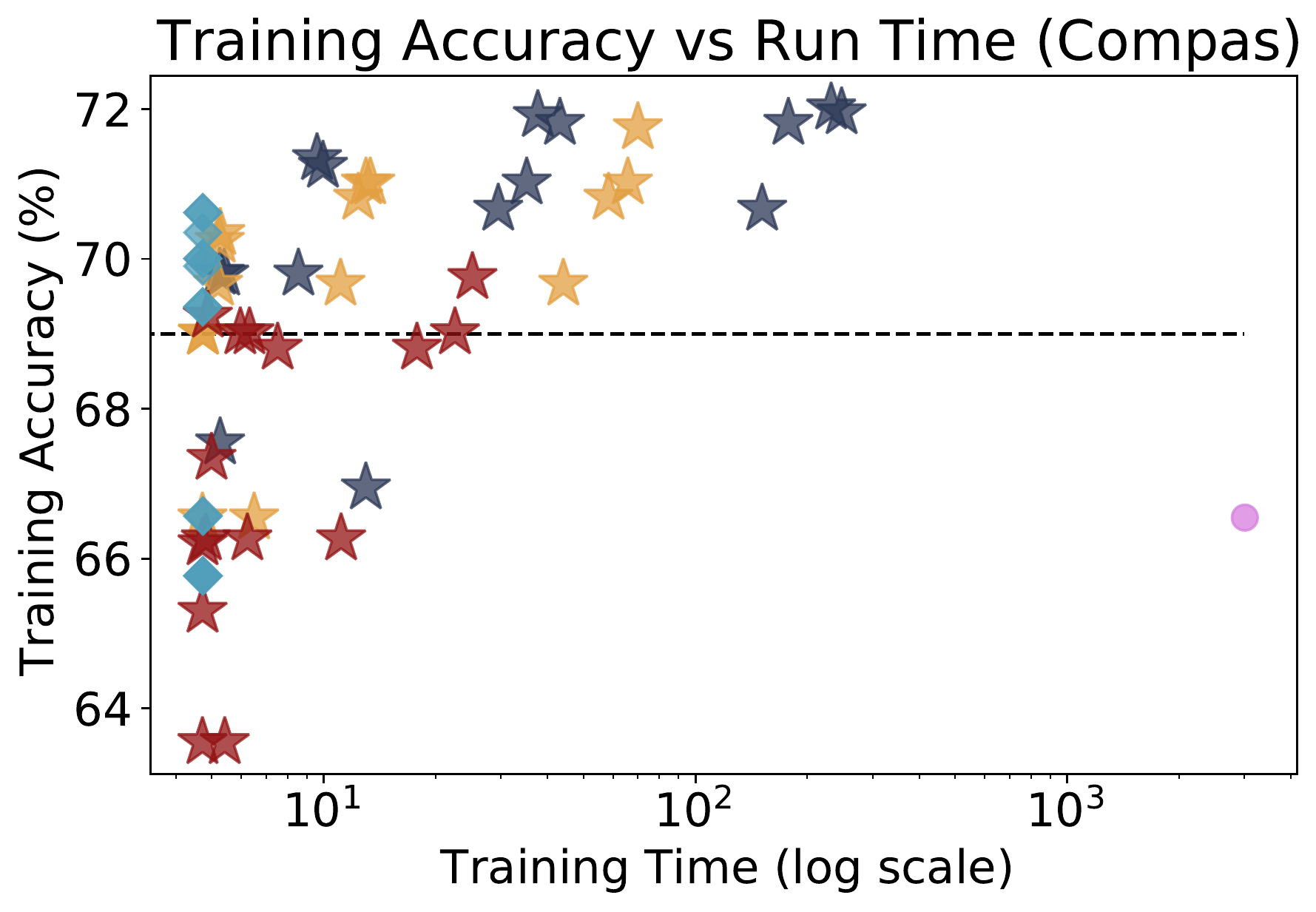}
\includegraphics[width=0.245\textwidth]{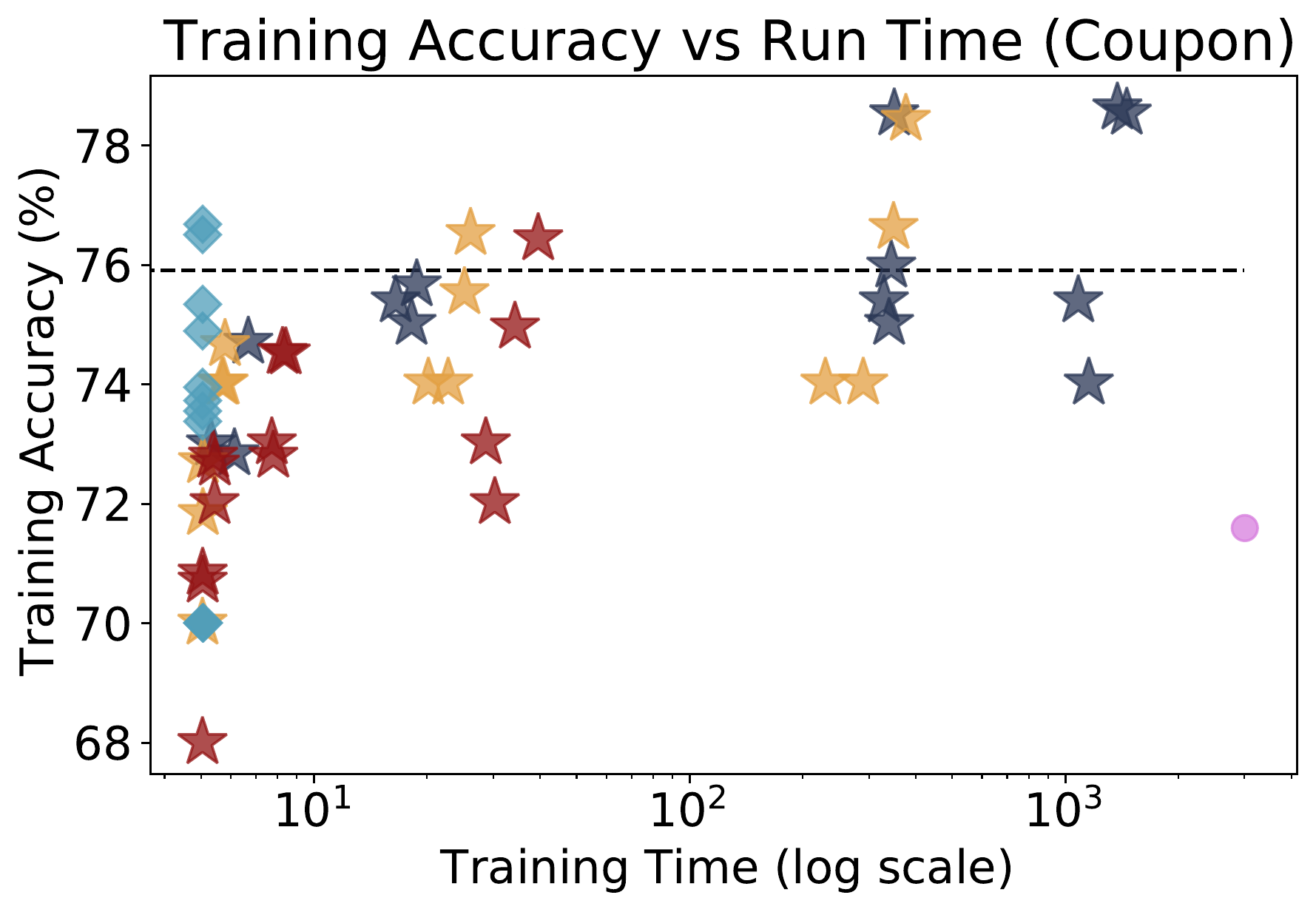}\\
\includegraphics[width=0.245\textwidth]{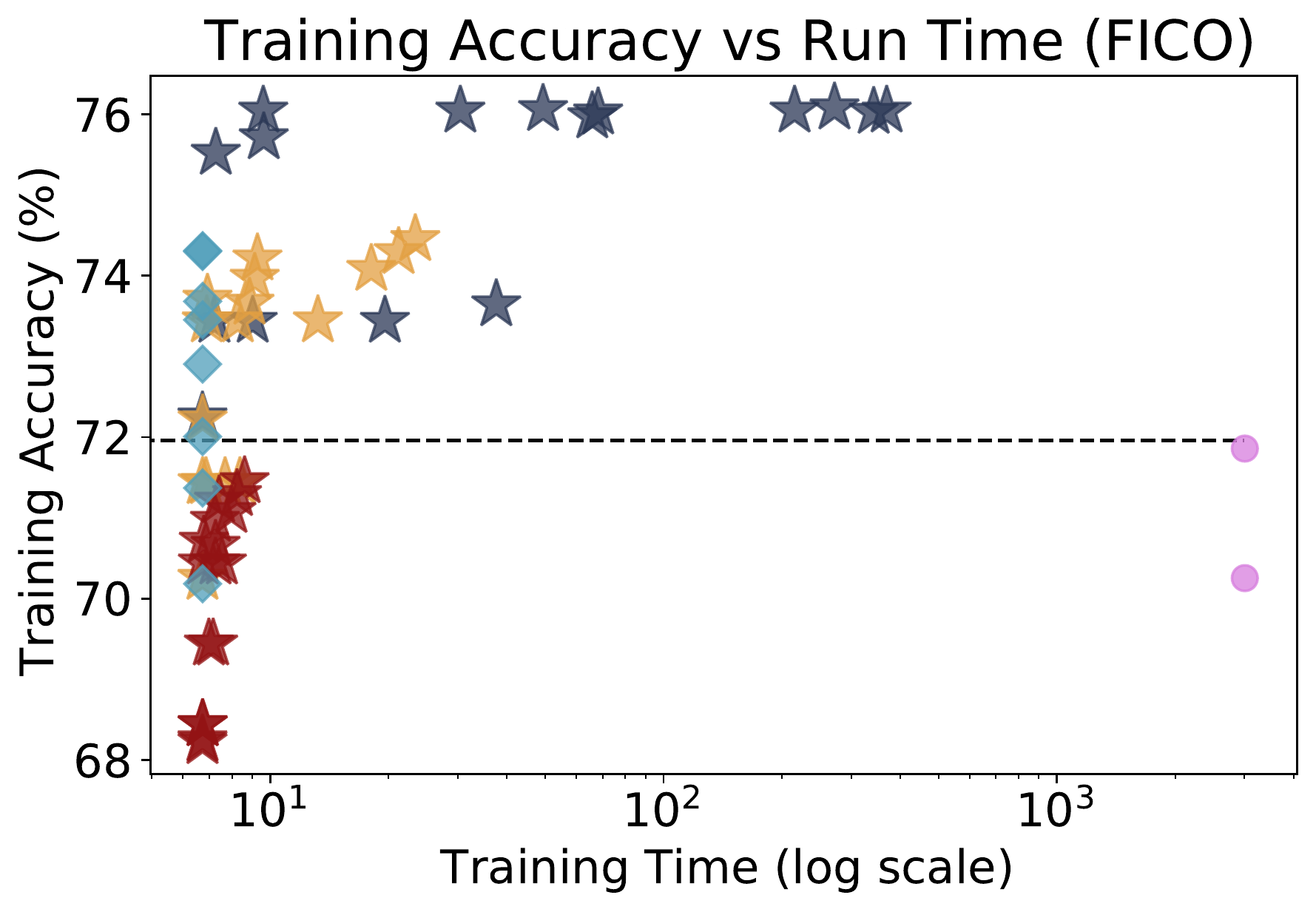}
\includegraphics[width=0.245\textwidth]{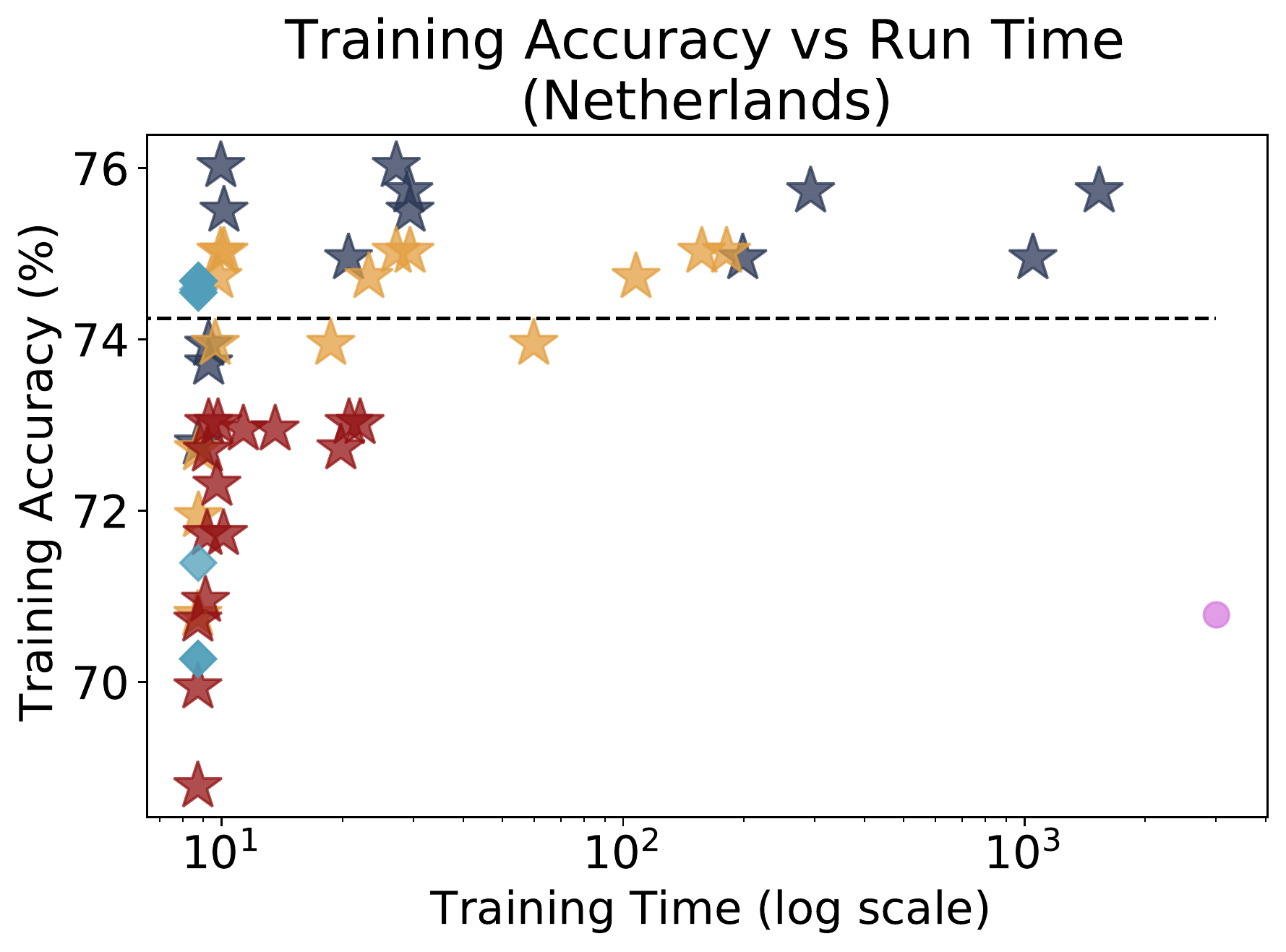}
\includegraphics[width=0.245\textwidth]{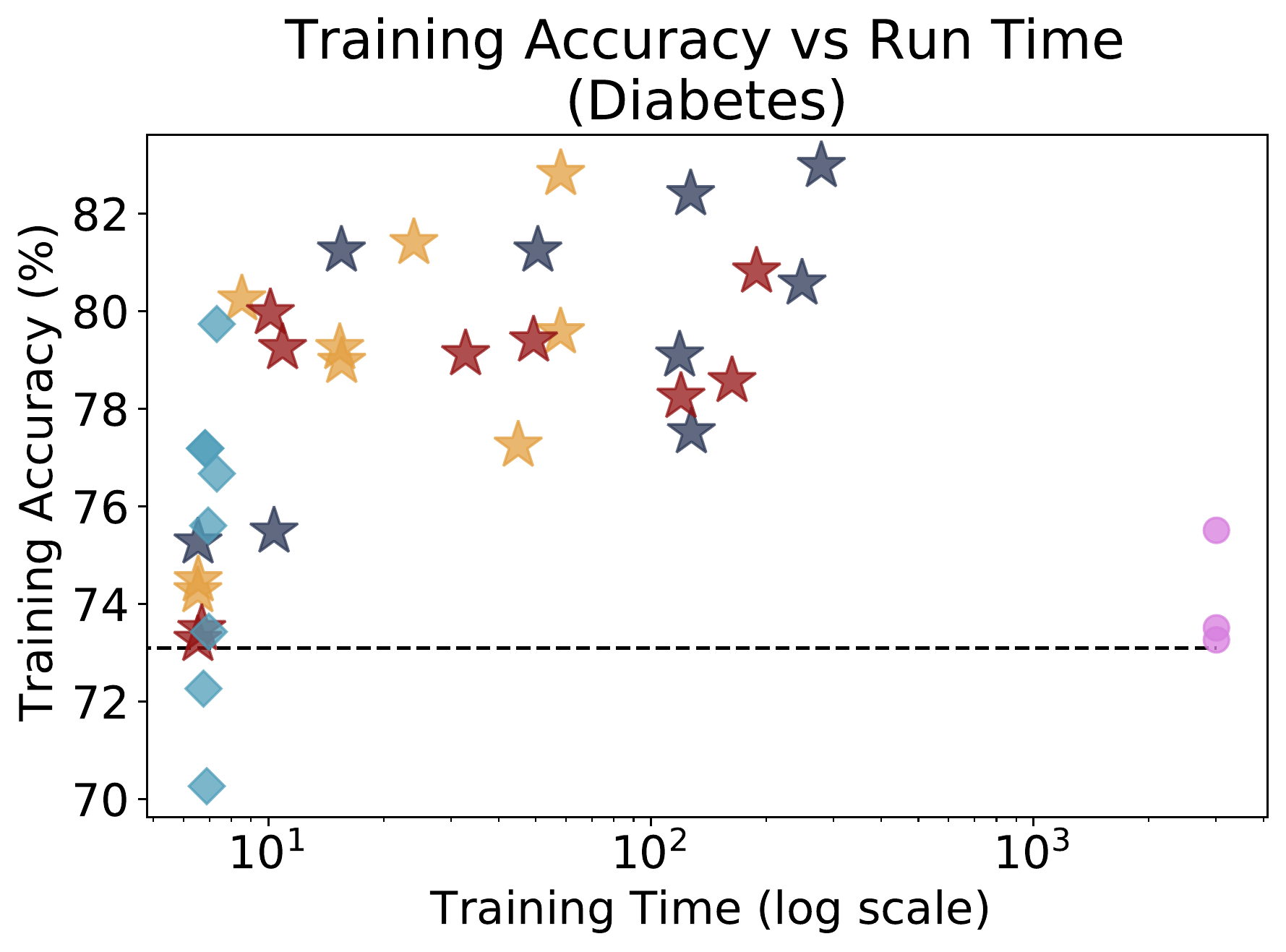}
\includegraphics[width=0.245\textwidth]{legend.pdf}
\vspace{-2ex}
\caption{\textbf{Training time vs. training accuracy:} All methods but CART and GBDT use guessed thresholds. GBDT and DL8.5 use data duplication. DL8.5 frequently times out, so there are fewer markers for it. While CART is the fastest algorithm, \GHOUL uses its additional runtime to produce models that produce higher accuracy.}
\vspace{-2ex}
\label{fig:acc-time-training}
\end{figure*}

\begin{figure*}
\includegraphics[width=0.245\textwidth]{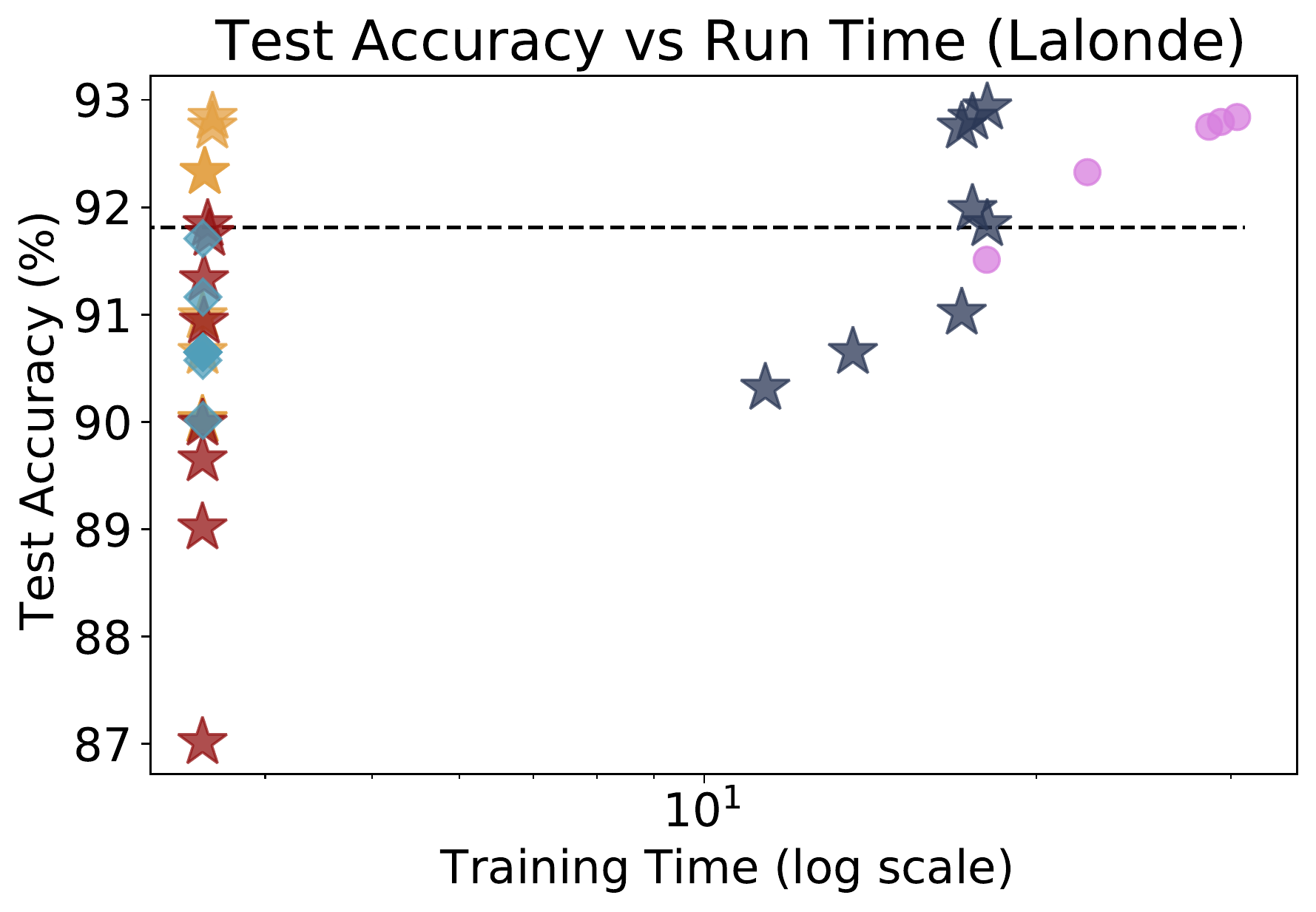}
\includegraphics[width=0.245\textwidth]{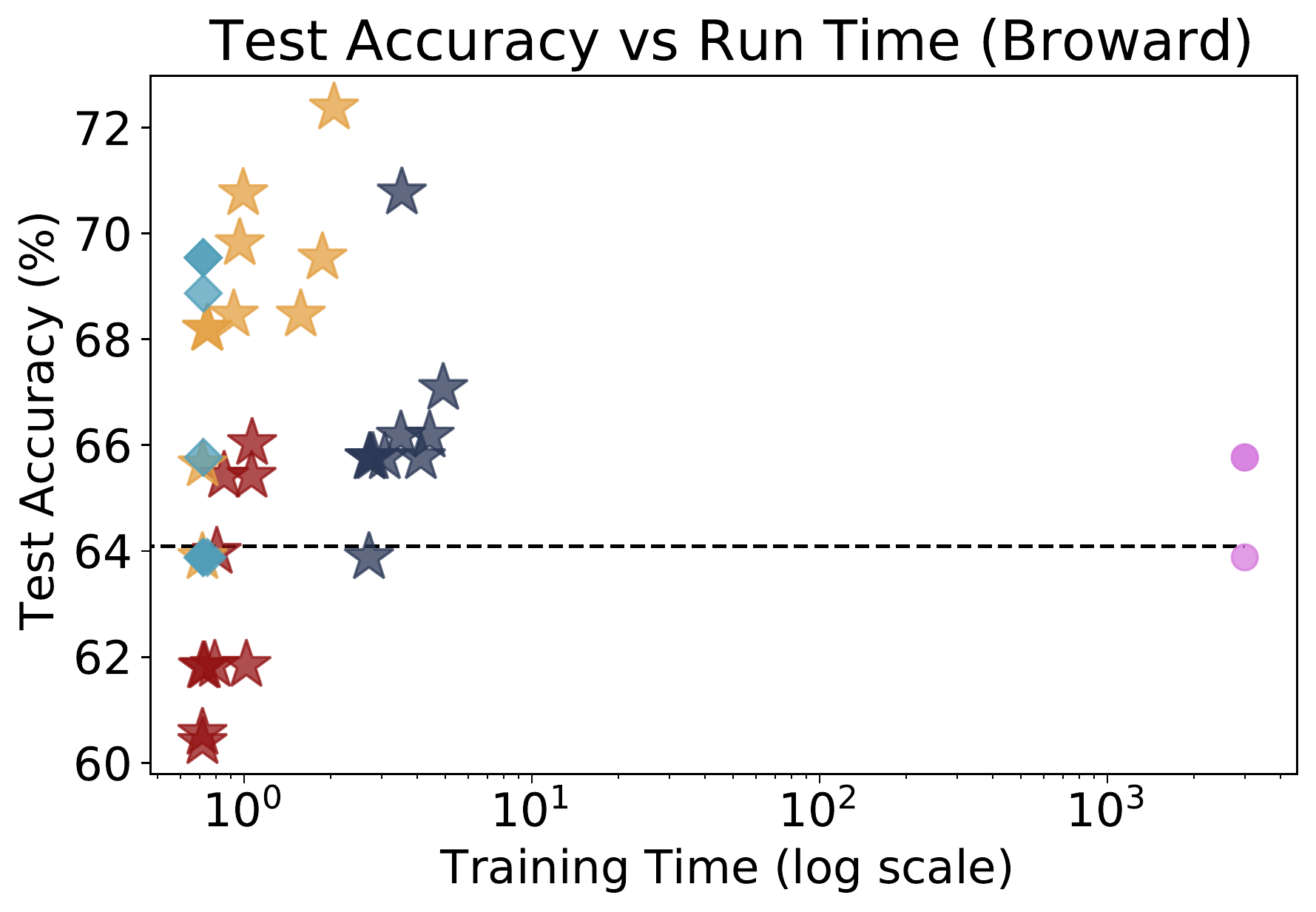}
\includegraphics[width=0.245\textwidth]{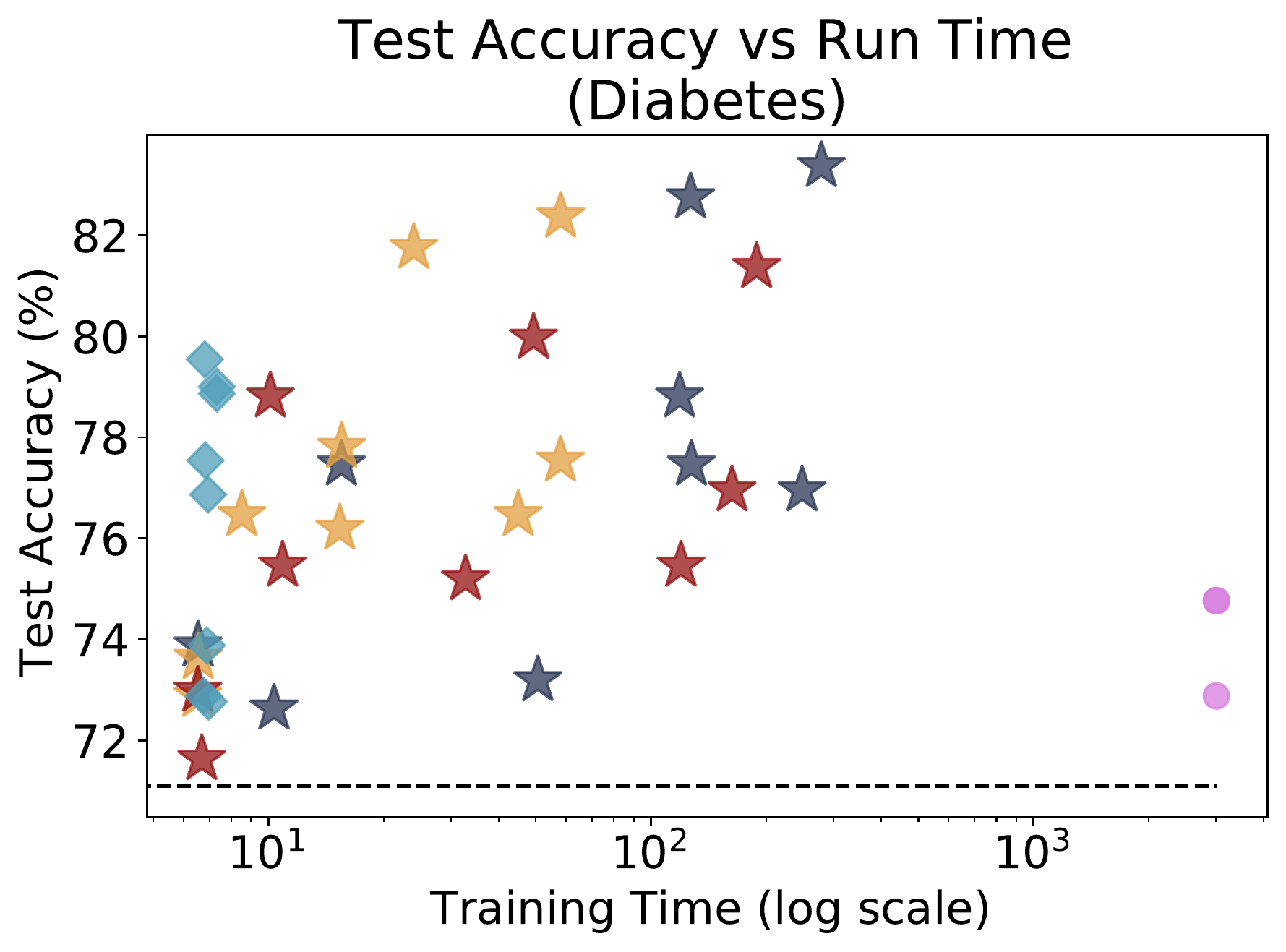}
\includegraphics[width=0.245\textwidth]{legend.pdf}
\vspace{-1ex}
\caption{\textbf{Training time vs$.$ test accuracy:} All methods but CART and GBDT use guessed thresholds. GBDT and DL8.5 use data duplication. DL8.5 frequently times out, so there are fewer markers for it. While CART is the fastest algorithm, \GHOUL uses its additional runtime to produce models with higher accuracy and generalize better.}
\vspace{-2ex}
\label{fig:acc-time-test}
\end{figure*}

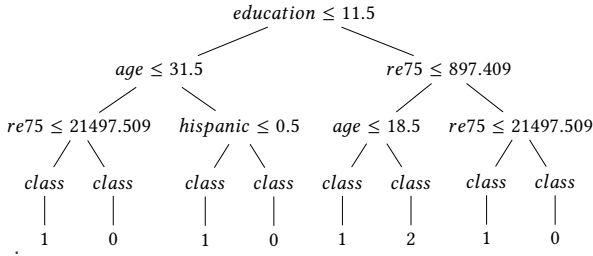
\begin{figure}[ht]
\centering
 \begin{frame}{}  
\Huge
\hspace{-2ex}
\scalebox{0.45}{
\begin{forest}
    [ $education \le 11.5$ [ $age \le 31.5$ [ $re75 \le 21497.509$ [ $class$ [ $1$ ] ] [ $class$ [ $0$ ] ] ] [ $hispanic \le 0.5$ [ $class$ [ $1$ ] ] [ $class$ [ $0$ ] ] ] ] [ $re75 \le 897.409$ [ $age \le 18.5$ [ $class$ [ $1$ ] ] [ $class$ [ $2$ ] ] ] [ $re75 \le 21497.509$ [ $class$ [ $1$ ] ] [ $class$ [ $0$ ] ] ] ] ]
\end{forest}
}
\end{frame}
    \vspace{-1ex}
    \caption{The tree generated by \GHOUL{} (depth limit 3) on the Lalonde dataset.}
    \label{fig:Lalonde_case_study}
\end{figure}

\head{Comparison of our approaches} We next compare the relative accuracy achieved using the direct, data-duplication, and weighted sampling approaches on all our evaluation datasets. The star-shaped points in Figures~\ref{fig:sparsity-train-accuracy} and \ref{fig:sparsity-test-accuracy} show the result of this comparison. These results suggest a trade-off between accuracy and running time. Weighted sampling is the fastest approach, but it has the worst accuracy among our approaches because it uses only subsets of the data. Data duplication, while slower than weighted sampling, is faster than the direct method, without losing much accuracy.

\head{Sparsity vs$.$ accuracy} If we now consider the dotted line and round and diamond shapes in Figures~\ref{fig:sparsity-train-accuracy} and \ref{fig:sparsity-test-accuracy}, we can see the accuracy-sparsity tradeoff for different decision tree models (the black line represents accuracy for GBDT). \GHOUL{} produces excellent training and test accuracy with a small number of leaves, and, compared to other decision tree models, achieves higher accuracy for every level of sparsity. Results of other~datasets~are~in~Appendix~\ref{app:sparse_accuracy}.

\head{Training time vs. accuracy} Figures~\ref{fig:acc-time-training} and \ref{fig:acc-time-test} show the training time and accuracy for different decision tree models. While the training times of \GHOUL{} and CART are almost the same, \GHOUL{} achieves the highest training and test accuracy in almost all cases. DL8.5 struggled with the 1 hour termination condition for all datasets except Lalonde; because DL8.5 did not solve to optimality, it was outperformed by both CART and \GHOUL. Results of other datasets are in Appendix~\ref{app:training_accuracy}.


\head{Lalonde Case Study}
 The Lalonde dataset is from the National Supported Work Demonstration \cite{dehejia1999causal, Lalonde_reference}, a labour market experiment in which participants were randomized between treatment (on-the-job training lasting between nine months and a year) and control groups. Accordingly, for each unit $U_i$, we have a pre-treatment covariate vector $X_i$ and observed assigned treatment $Z_i$. Let $Y^1_i$ be the potential outcome if unit $U_i$ received the treatment, and $Y^2_i$ be the potential outcome if it was not treated. When a unit is treated, we do not observe the potential outcome if it was not treated and vice versa. To address this issue, we use the MALTS model~\cite{parikh2018malts} that estimates these missing values by matching. MALTS gives us an estimate of the conditional average treatment effect, which we call $TE$. Since our weighted decision tree is designed for classification, we classify participants into three groups–-``should be treated,'' ``should be treated if budget allows,'' and ``should not be treated'' -– based on their conditional average treatment effect estimate. We selected features \textit{age, education, black, hispanic, married, nodegree, re75} for the pre-treatment covariate vector. Then we labelled the data points as $2, 1,$ and $0$ if the estimated treatment effect is larger than $2000$, between $-5000$ and $2000$, and less than $-5000$, respectively. Here, the penalty for each misclassification is defined as follows:
\vspace{-1ex}
\[   
\text{cost} = 
     \begin{cases}
       \text{0,} &\quad\text{correctly classified},\\
       200 + 3 \times age &\quad\text{label = 0 and misclassified,} \\
       100 + 3 \times age &\quad\text{label = 1 and misclassified,} \\
       300  &\quad\text{label = 2 and misclassified.} \\ 
     \end{cases}
\]
We linearly scale the above costs to the range from 1 to 100, and in the case of data-duplication, we round them to integers and treat them as weights of the dataset. Figure~\ref{fig:Lalonde_case_study} shows the generated tree by \GHOUL{} with a depth limit of 3. Generated trees with other depth limits can be found in Appendix~\ref{app:lalonde_trees}.

\section{Conclusions}
\label{sec:conclusions}
To find the optimal weighted decision tree, we first suggest directly optimizing a weighted loss function. However, this approach might be time-consuming for large weighted datasets. To improve efficiency, we present the data-duplication approach, which rounds all weights to integers and then duplicates each sample by its weights. Finally, to further improve efficiency, we present a stochastic process, where we sample an unweighted dataset from our weighted dataset. Our results suggest a trade-off of accuracy and running time among these approaches.

\begin{acks}
    We acknowledge the following grant support: NIH/NIDA under grant number DA054994 and NSF under grant number IIS-2147061. This research was enabled in part by support provided by WestGrid (\href{https://www.westgrid.ca}{https://www.westgrid.ca}) and The Digital Research Alliance (\href{https://alliancecan.ca/en}{https://alliancecan.ca/en}).
\end{acks}


\bibliography{main}
\bibliographystyle{ACM-Reference-Format}

\newpage
\appendix
\onecolumn
\section{Theorems and Proofs}

\subsection{Proof of Theorem~\ref{thm:glb}}
\label{app:theorem_1}
\setcounter{theorem}{0}
\begin{theorem} (Performance Guarantee).
Let $R(t_{\textrm{g}}, \tilde{\x}, \y)$ denote the objective of $t_\textrm{g}$ on the full binarized dataset $(\tilde{\x}, \y)$ for some per-leaf penalty $\lambda$. Then for any decision tree $t$ that satisfies the same depth constraint $d$, we have: \begin{align*}R(t_{\textrm{g}}, \tilde{\x}, \y) \leq &\frac{1}{\sum_{i = 1}^{N} w_i} \left(\sum_{i \in \MC} w_i + \sum_{i \in \CC} \mathds{1}[y_i \neq \hat{y}_i^{t}] \times w_i \right) + \lambda H_{t}.\end{align*}
That is, the objective of the guessing model is no worse than the union of errors made by the reference model and tree $t$.
\end{theorem}

\begin{proof}
We adapt the proof in the work of~\citet{McTavishZhong} to the weighted setting. Similar to them, we use the following notation to discuss lower bounds:
\begin{itemize} 

\item {$\lambda \geq 0$ is a regularizing term } 

\item For some decision tree $\t$, we calculate its risk as {${R}(\t,\tilde{\x},\y) = \frac{1}{\sum_{i = 1}^{N} w_i}\sum_{i=1}^N w_i \times \mathds{1}[y_i \neq \hat{y}_i^{\t}] + \lambda H_{\t}$}, where $H_{\t}$ refers to the number of leaves in tree $\t$.

\item $T$ is a reference model we use to guess lower bounds.

\item $\CC$ and $\MC$ are, respectively, the indices of the set of observations in the training set correctly classified by $T$ and the set of observations incorrectly classified by $T$.

\item $s_a$ is the set of training set observations captured in our current subproblem.

\item $d$ represents the maximum allowed depth for solutions to our current subproblem. If $d=0$, no further splits are allowed.

\item {$t_{g}(s_a, d, \lambda)$} is the lower-bound-guessing-algorithm's solution for subproblem $s_a$, depth limit $d$, and regularizer $\lambda$. When we just specify $t_g$ without arguments, we are referring to the lower-bound-guessing-algorithm's solution for the root subproblem (the whole dataset), with the depth limit argument provided for the tree as a whole.

\item Consider a subproblem $s_a$ corresponding to the full set of points passing through a specific internal or leaf node of the optimal tree $\t^*$ (call it node$_{\t^*, s_a}$). Define $H_{\t^*_a}$ as the number of leaves below node$_{s_a}$  (or 1 if node$_{s_a}$ is a leaf). Note that this is also the number of leaves needed in an optimal solution for subproblem $s_a$. Similarly define $H_{t_{g, a}}$ as the number of leaves below node$_{t_g,s_a}$ in $t_g$ (when $s_a$ corresponds to the full set of points passing through a node in $t_g$). Note that this does not necessarily correspond to the number of leaves needed in an optimal solution for subproblem $s_a$ because $t_g$ has not been fully optimized.

\item {$R_{g}(s_a, d, \lambda)$} is the objective of the solution found for subproblem $s_a$, depth limit $d$, and regularizer $\lambda$ when we use lower bound guessing:
\begin{align*} R_{g}(s_a, d, \lambda) := &\frac{1}{\sum_{i = 1}^{N} w_i}\sum_{i\in s_a} w_i \times \mathds{1}[y_i \neq \hat{y}_i^{t_{g}(s_a, d, \lambda)}]  + \lambda H_{t_{g, a}}.\end{align*} 

\item $lb_{g}(s_a)$ was defined as {$$lb_{g}(s_a) := \frac{1}{\sum_{i = 1}^{N} w_i}\sum_{i \in s_a} w_i \times \mathds{1}[y_i \neq \hat{y}_i^{T}] + \lambda,$$} which could be obtained at equality if we achieve the accuracy of $T$ in a single leaf. We add that this is equivalent (by definition) to {$$lb_{g}(s_a) = \frac{1}{\sum_{i = 1}^{N} w_i}\left(\sum_{i \in \MC \cap s_a} w_i\right) + \lambda.$$}

\item We additionally define {$t^*(s_a, d, \lambda)$} as an optimal solution for subproblem $s_a$, depth limit $d$, and regularizer $\lambda$ (that is, the solution found when we do not use lower bound guessing). When we just specify $t^*$ without arguments, we are referring to a solution for the root subproblem (the whole dataset), with the depth limit argument provided for the tree as a whole.

\item We also define {$R(s_a, d, \lambda)$} as the objective of the optimal solution found for subproblem $s_a$, depth limit $d$, and regularizer $\lambda$ (that is, the objective of the solution found when we do not use lower bound guessing):  {$$R(s_a, d, \lambda) := \frac{1}{\sum_{i = 1}^{N} w_i}\sum_{i\in s_a} w_i \times \mathds{1}[y_i \neq \hat{y}_i^{t^*(s_a, d, \lambda)}] + \lambda H_{t^* \cap s_a}.$$}
Here, by $t^* \cap s_a$ we refer to the part of the subproblem $s_a$ captured by a part of tree $t^*$.

\item {Define $lb_\textrm{max}(s_a, d, \lambda)$ as the highest lower bound estimate that occurs for a given subproblem $s_a$, depth budget $d$, and regularization $\lambda$, across the algorithm's whole execution when using lower bound guessing. Note that $$R_{g}(s_a, d, \lambda) \leq lb_\textrm{max}(s_a, d, \lambda)$$ because when a subproblem is solved, the current lower bound is made to match the objective of the solution returned for that subproblem. As a reminder, $R_{g}$ is computed after the subproblem is solved. When using lower bound guesses, it is possible for intermediate lower bound estimates (and therefore $lb_\textrm{max}$) to exceed $R_{g}(s_a, d, \lambda)$, and then the lower bound (but not $lb_\textrm{max}$) is decreased to match the objective of the best solution found when the subproblem is solved.
}
\end{itemize}

Now, without loss of generality, select a tree $t$ that is within some depth constraint $d$. We wish to prove that the {risk} on the full dataset (with some regularization $\lambda$) is bounded as: 
\begin{equation} \label{eq:lbguesstoprove}
R(t_{g}, \tilde{\x}, \y) \leq \frac{1}{\sum_{i = 1}^{N} w_i} \left(\sum_{i \in \MC} w_i + \sum_{i \in \CC} w_i \times \mathds{1}[y_i \neq \hat{y}_i^{t}]\right) + \lambda H_{t}.\end{equation}
or, equivalently, (defining $s_\textrm{all}$ as the set of all points in the dataset):
$$R_g(s_\textrm{all}, d, \lambda) \leq \frac{1}{\sum_{i = 1}^{N} w_i} \left(\sum_{i \in \MC} w_i + \sum_{i \in \CC} w_i \times \mathds{1}[y_i \neq \hat{y}_i^{t}]\right) + \lambda H_{t}.$$

To show this, it is sufficient to show a result that is strictly more general. Specifically, we show that for any subproblem $s_a$ (with $d, \lambda \geq 0$) that occurs as an internal or leaf node in $t$ (including the root), we can bound $lb_\textrm{max}(s_a, d, \lambda)$ as follows. Equation \eqref{eq:lbguesstoprove} is a direct consequence, using $s_a$ as the full dataset $s_{\textrm{all}}$.
\begin{equation} \label{eq:lbguesslambda}
\begin{aligned}
lb_\textrm{max}(s_a, d, \lambda) \leq \frac{1}{\sum_{i = 1}^{N} w_i} \left(\sum_{i \in \MC \cap s_a} w_i + \sum_{i \in \CC \cap s_a} w_i \times \mathds{1}[y_i \neq \hat{y}_i^{t}]\right) 
 + \lambda H_{t_a}. 
\end{aligned}
\end{equation}
What we originally wished to prove then follows from noting that, by definition of $lb_\textrm{max}$, we have
$R_g(s_{\textrm{all}}, d, \lambda) \leq lb_\textrm{max}(s_\textrm{all}, d, \lambda)$. 
Here, $d, \lambda$ are the depth limit and regularization provided for $t_g$ (where $d$ matches or exceeds the depth of $t$).

We prove this sufficient claim (hereafter referred to as Equation \ref{eq:lbguesslambda}) for all subproblems in $t$ using induction.
\\\\
\textbf{Base Case}: Let us take any subset of data $s_a$, depth constraint $d$, and regularization $\lambda$ which corresponds to a subproblem in tree $t$ whose solution in $t$ is a leaf node.

Because the solution to $s_a$ was a leaf in $t$, then its objective (without making further splits) is $ub = \frac{1}{\sum_{i = 1}^{N} w_i}(\sum_{i \in s_a} w_i \times \mathds{1}[y_i \neq \hat{y}_i^{t}]) + \lambda\cdot 1$. 
We want to show, in this case, that  \eqref{eq:lbguesslambda} holds.

Our initial lower bound guess is $lb_g$. Either $lb_g > ub$, or $lb_g \leq ub$. If $lb_g > ub$, we are done with the subproblem as per the Branch-and-Bound algorithm in \cite{McTavishZhong}. Otherwise, when $lb_g \leq ub$, we know from their algorithm that the lower bound can never increase above $ub$. Therefore, the highest value of the lower bound during execution,  $lb_\textrm{max}(s_a, d, \lambda)$, obeys  $lb_\textrm{max}(s_a, d, \lambda)\leq \max (ub, lb_{g})$. Then,
\begin{equation*}
    \begin{aligned}
    lb_\textrm{max}(s_a, d, \lambda) &\leq \max (ub, lb_{g}) \;\;\;\;\textrm{(by argument just above)}\\
    &= \max \left(\frac{1}{\sum_{i = 1}^{N} w_i}\sum_{i \in s_a} (w_i \times \mathds{1}[y_i \neq \hat{y}_i^{t}]) + \lambda, \frac{\sum_{i \in \MC \cap s_a}w_i}{\sum_{i = 1}^{N} w_i} + \lambda \right) \;\;\;\;\textrm{(by definition of $ub$ and $lb_g$)}\\
    &= \frac{1}{\sum_{i = 1}^{N} w_i} \max \left(\sum_{i \in s_a} w_i \times \mathds{1}[y_i \neq \hat{y}_i^{t}],  \sum_{i \in \MC \cap s_a}w_i\right) + \lambda\\
    &= \frac{1}{\sum_{i = 1}^{N} w_i} \max \left(\sum_{i \in \CC \cap s_a} w_i \times \mathds{1}[y_i \neq \hat{y}_i^{t}]  + \sum_{i \in \MC \cap s_a} w_i \times \mathds{1}[y_i \neq \hat{y}_i^{t}], \sum_{i \in \MC \cap s_a}w_i\right) + \lambda.\\
\end{aligned}
\end{equation*}
Note that both terms inside the max are at most $\sum_{i \in \CC \cap s_a}w_i \times \mathds{1}[y_i \neq \hat{y}_i^{t}] + \sum_{i \in \MC \cap s_a}w_i$. Therefore, 
\begin{align*} 
lb_\textrm{max}(s_a, d, \lambda) \leq \frac{1}{\sum_{i = 1}^{N} w_i} \left( \sum_{i \in \CC \cap s_a}w_i \times \mathds{1}[y_i \neq \hat{y}_i^{t}] + \sum_{ i \in \MC \cap s_a} w_i\right)  + \lambda.\end{align*}
Moreover, the number of leaves in the optimal tree for subproblem $s_a$ is 1, i.e., $H_{t_a} = 1$ (since $s_a$ corresponds to a leaf in $t$), so 
\begin{equation*}
    \begin{aligned}
    lb_\textrm{max}(s_a, d, \lambda) \leq  \frac{1}{\sum_{i = 1}^{N} w_i} \left(\sum_{i \in \MC \cap s_a} w_i +\sum_{i \in \CC \cap s_a} w_i \times \mathds{1}[y_i \neq \hat{y}_i^{t}]\right) + \lambda H_{t_a}. 
    \end{aligned}
\end{equation*}
And this matches Equation \eqref{eq:lbguesslambda}, as required. Thus, we have shown that the base case obeys the statement of the theorem.
\\\\
\textbf{Inductive Step}: Let the set of points $s_a$, depth constraint $d >0$, and regularization $\lambda \geq 0$ be a subproblem that corresponds to an internal node in $t$. Let $j$ indicate the feature that was split on in $t$ for this node, and define $s_j$ as the set of data points $\tilde{\x}_i$ such that $\tilde{x}_{ij} = 1$ and $s_j^c$ as the set of data points $\tilde{\x}_i$ such that $\tilde{x}_{ij} = 0$. We assume \eqref{eq:lbguesslambda} holds for both the left and right child subproblems and aim to show that it holds for their parent subproblem. The left subproblem is $s_a \cap s_j$ with depth $d-1$ and the right subproblem is $s_a \cap s_j^c$ with depth $d-1$. Thus, assuming (as per the inductive hypothesis) that \eqref{eq:lbguesslambda} holds, i.e., 
\begin{equation*}
    \begin{aligned}
    lb_\textrm{max}(s_a \cap s_j, d-1, \lambda) \leq & \frac{1}{\sum_{i = 1}^{N} w_i} \left(\sum_{i \in \MC \cap s_a \cap s_j} w_i + \sum_{i \in \CC \cap s_a \cap s_j} w_i \times \mathds{1}[y_i \neq \hat{y}_i^{t}]\right) + \lambda H_{t_a \cap s_j}
    \end{aligned}
\end{equation*}
and
\begin{equation*}
    \begin{aligned}
    lb_\textrm{max}(s_a \cap s_j^c, d-1, \lambda) \leq \frac{1}{\sum_{i = 1}^{N} w_i} \left(\sum_{i \in \MC \cap s_a \cap s_j^c} w_i + \sum_{i \in \CC \cap s_a \cap s_j^c} w_i \times \mathds{1}[y_i \neq \hat{y}_i^{t}]\right) + \lambda H_{t_{a \cap j^c}},
    \end{aligned}
\end{equation*}
it remains to show that Equation \eqref{eq:lbguesslambda} holds for $s_a$:
\begin{equation*}
    \begin{aligned}
    lb_\textrm{max}(s_a, d, \lambda) \leq & \frac{1}{\sum_{i = 1}^{N} w_i} \left(\sum_{i \in \MC \cap s_a} w_i   + \sum_{i \in \CC \cap s_a} w_i \times \mathds{1}[y_i \neq \hat{y}_i^{t}]\right) + \lambda H_{t_a}.
    \end{aligned}
\end{equation*}

We prove the inductive step by cases:
\begin{enumerate}
    \item If $ub \leq lb_g + \lambda$, then 
    as per the branch-and-bound algorithm in \cite{McTavishZhong},
    $s_a$ corresponds to a leaf in $t_g$, with a loss for this subproblem equal to $ub$. Since the algorithm returns immediately after changing the lower bound to $ub$, the maximum value the lower bound takes (that is, $lb_\textrm{max}$) is whichever of $lb_g$ or $ub$ is higher. We have: 
    \begin{equation*}
        \begin{aligned}
        lb_\textrm{max}(s_a, d, \lambda) &\leq \max(lb_g, ub) \\
        &\leq \max(lb_g, lb_g + \lambda) \; \; \; \; \text{(Since we conditioned on $ub \leq lb_g + \lambda$)} \\
         &\leq lb_g + \lambda \\
         &= \frac{1}{\sum_{i = 1}^{N} w_i} \sum_{i \in \MC \cap s_a} w_i + 2 \lambda \\& \leq \frac{1}{\sum_{i = 1}^{N} w_i}\left(\sum_{i \in \MC \cap s_a}w_i + \sum_{i \in \CC \cap s_a} w_i \times \mathds{1}[y_i \neq \hat{y}_i^{t}]\right) + 2 \lambda.
        \end{aligned}
    \end{equation*}
    Noting that because $s_a$ corresponds to an internal node in $t$, there are at least two leaves below it, so $H_{t_a} \geq 2$. Thus,
    \begin{equation*}
        \begin{aligned}
            lb_\textrm{max}(s_a, d, \lambda) \leq  \frac{1}{\sum_{i = 1}^{N} w_i}\left(\sum_{i \in \MC \cap s_a} w_i + \sum_{i \in \CC \cap s_a} w_i \times \mathds{1}[y_i \neq \hat{y}_i^{t}]\right) + \lambda H_{t_a}.
        \end{aligned}
    \end{equation*}
    This equation matches Equation \ref{eq:lbguesslambda}, as required.
    
    \item Else, 
    the lower bound (and therefore $lb_\textrm{max}(s_a, d, \lambda)$) cannot exceed the combined lower bounds of the left and right subproblems from splitting on feature $j$. We know the split for feature $j$ will lead to a lower bound estimate no more than  $lb_\textrm{max}(s_a \cap s_j, d-1, \lambda) + lb_\textrm{max}(s_a \cap s^c_j, d-1, \lambda)$. Thus we have: 
    \begin{equation*}
            lb_\textrm{max}(s_a, d, \lambda) \leq  lb_\textrm{max}(s_a \cap s_j, d-1, \lambda) + lb_\textrm{max}(s_a \cap s_j^c, d-1, \lambda).
    \end{equation*}
    Using the inductive hypothesis, this reduces to 
    \begin{align*}
    lb_\textrm{max}(s_a, d, \lambda) &\leq  \frac{1}{\sum_{i = 1}^{N} w_i} \left(\sum_{i \in \MC \cap s_a \cap s_j} w_i + \sum_{i \in \CC \cap s_a \cap s_j} w_i \times \mathds{1}[y_i \neq \hat{y}_i^{t}]\right) + \lambda H_{t_{a \cap j}} \\ & \ \ \ \ \ + \frac{1}{\sum_{i = 1}^{N} w_i} \left(\sum_{i \in \MC \cap s_a \cap s_j^c} w_i + \sum_{i \in \CC \cap s_a \cap s_j^c} w_i \times \mathds{1}[y_i \neq \hat{y}_i^{t}]\right) +  \lambda H_{t_{a \cap j^c}}.
    \end{align*}
Noting that $s_a \cap s_j$ and $s_a \cap s_j^c$ partition $s_a$:
    \begin{align*}
    lb_\textrm{max}(s_a, d, \lambda) &\leq  \frac{1}{\sum_{i = 1}^{N} w_i} \left(\sum_{i \in \MC \cap s_a} w_i + \sum_{i \in \CC \cap s_a} w_i \times \mathds{1}[y_i \neq \hat{y}_i^{t}]\right) + \lambda H_{t_a}.
    \end{align*}
    This equation matches Equation \eqref{eq:lbguesslambda}, as required. 
\end{enumerate}

Thus we have proved the inductive step holds for all possible cases. By induction, then, we have proved Equation \eqref{eq:lbguesslambda} holds for any internal or leaf node in $t$. Since the root node is an internal node of $t$, we have also proven Equation \eqref{eq:lbguesslambda} holds for the root problem. As per the justification given when claiming Equation \eqref{eq:lbguesslambda} was sufficient, that also means
$$R(t_{g}, \tilde{\x}, \y) \leq \frac{1}{\sum_{i = 1}^{N} w_i} \left(\sum_{i \in \MC} w_i + \sum_{i \in \CC} w_i \times \mathds{1}[y_i \neq \hat{y}_i^{t}]\right) + \lambda H_{t}$$
which is what we wished to show.
\end{proof}

\subsection{Proof that Approximate Min Loss is not that far from True Min Loss}
\label{app:theorem_2}
Recall the objective 
\[R(\t) := \frac{1}{\sum_{i = 1}^{N} w_i}\sum_i w_i I_i(\t) + \lambda \#\textrm{leaves}.\] Define the objective with the approximate weights as
\[\Rt(\t) := \frac{1}{\sum_{i = 1}^{N} \tilde{w}_i}\sum_i \tilde{w}_i I_i(\t) + \lambda \#\textrm{leaves}.\] When we rounded, we ensured that the weights did not change much. That is, we know by design that $\|\mathbf{w}-\mathbf{\tilde{w}}\|_\infty\leq \epsilon$. 

\begin{theorem} Let $t^*$ be a minimizer of the objective as $t^* \in \argmin_t R(t)$, and $\tt$ be a minimizer of the approximate loss function as $\tt\in \argmin_t \Rt(t)$. If $\|\mathbf{w}-\mathbf{\tilde{w}}\|_\infty\leq \epsilon$, we have:
\begin{equation*}
|R(t^*) - \Rt(\tt)|\leq \\max\left\{ \frac{(\zeta - 1) \psi + \epsilon}{\zeta}, \frac{(\eta - 1) \psi + \epsilon}{\eta} \right\},
\end{equation*}
where $\eta = \max_{1 \leq i \leq N} \left\{\frac{w_i}{\tilde{w}_i} \right\}$, $\zeta = \max_{1 \leq i \leq N} \left\{\frac{\tilde{w}_i}{w_i} \right\}$, and $\psi = \frac{\max_i\left\{w_i, \tilde{w}_i\right\}}{\min_i \left\{ w_i, \tilde{w}_i \right\}}$.
\end{theorem}

\begin{proof}
Since we know that $\forall 1 \leq i \leq N\: :\: \frac{1}{\zeta} \leq \frac{w_i}{\tilde{w_i}} \leq \eta$, and also $\forall 1 \leq i \leq N \: :\:$ $|w_i - \tilde{w}_i | \leq \epsilon$, then for any tree $\t$,
\begin{eqnarray}\nonumber
|R(\t)-\Rt(\t)| &=& \left|\frac{1}{\sum_{i = 1}^{N} w_i}\sum_i w_i I_i(\t)+\lambda\#\textrm{leaves} - \frac{1}{\sum_{i = 1}^{N} \tilde{w}_i}\sum_{i}\tilde{w}_iI_i(\t)-\lambda\#\textrm{leaves} \right| \\\nonumber
&=& \left|\sum_i \left(\frac{w_i}{\sum_{i = 1}^{N} w_i} - \frac{\tilde{w}_i}{\sum_{i = 1}^{N} \tilde{w}_i}\right) I_i(\t)\right|\\\nonumber
&\leq& \sum_i \left|\frac{w_i}{\sum_{i = 1}^{N} w_i} - \frac{\tilde{w}_i}{\sum_{i = 1}^{N} \tilde{w}_i}\right| I_i(\t) \leq \sum_i \max \left\{  \frac{w_i}{\sum_i w_i} - \frac{\tilde{w}_i}{\zeta \sum_i w_i},  \frac{\tilde{w}_i}{\sum_i \tilde{w}_i} - \frac{w_i}{\eta \sum_i \tilde{w}_i}  \right\} I_i(\t)\\\nonumber
&\leq& \sum_i  \max\left\{ \frac{(\zeta - 1) w_i + \epsilon}{\zeta\sum_{i = 1}^{N} w_i}, \frac{(\eta - 1) \tilde{w}_i + \epsilon}{\eta\sum_{i = 1}^{N} \tilde{w}_i} \right\} I_i(\t) \\ \nonumber
&\leq&  \sum_i  \max\left\{ \frac{(\zeta - 1) \max_i\{w_i\} + \epsilon}{\zeta \times N \times \min_i\{w_i\}}, \frac{(\eta - 1) \max_i\{\tilde{w}_i\}+ \epsilon}{\eta \times N \times \min_i\{\tilde{w}_i\}} \right\} I_i(\t) \;\;\;\; \textrm{(since}\: N \times \max_i\{w_i\} \geq \sum_i^N w_i \geq N \times \min_i\{w_i\})\\ \nonumber
&=& \sum_i \frac{\max\left\{ \frac{(\zeta - 1) \psi + \epsilon}{\zeta}, \frac{(\eta - 1) \psi + \epsilon}{\eta} \right\}}{N} I_i(\t) \;\;\;\; \textrm{(where we used that } \psi = \frac{\max_i\left\{w_i, \tilde{w}_i\right\}}{\min_i \left\{ w_i, \tilde{w}_i \right\}})\\
&=& \max \left\{ \frac{(\zeta - 1) \psi + \epsilon}{\zeta}, \frac{(\eta - 1) \psi + \epsilon}{\eta} \right\} \mathcal{L}(\t). \label{RRtBound}
\end{eqnarray}
For simplicity, we let $K = \max\left\{ \frac{(\zeta - 1) \psi + \epsilon}{\zeta}, \frac{(\eta - 1) \psi + \epsilon}{\eta} \right\}$. We will use proof-by-contradiction. Let us assume that the statement of the theorem is false. This implies:
\begin{equation}\label{eq:assume}
 K <|R(t^*)-\Rt(\tt)|=|\Rt(\tt)-R(t^*)|.
\end{equation}
Since the error rate $\mathcal{L}(\t)$ is always less than 1 for any tree, we have that when $\Rt(\tt)\geq R(t^*)$,
\begin{eqnarray}\nonumber
K \times \mathcal{L}(t^*)&<&  K <|\Rt(\tt) - R(t^*)| = \Rt(\tt) - R(t^*)\\\label{e2}
R(t^*) + K  \mathcal{L}(t^*) &<&  \Rt(\tt) 
\end{eqnarray}
and conversely, when $R(t^*)\geq \Rt(\tt)$,
\begin{eqnarray}\nonumber
 K \mathcal{L}(\tt)&<& K <|R(t^*)-\Rt(\tt)| = R(t^*)-\Rt(\tt)\\
\label{e3}
\Rt(\tt)+ K \mathcal{L}(\tt) &<& R(t^*).
\end{eqnarray}

Using \eqref{RRtBound} for $t^*$, we have:
\begin{eqnarray*}
\Rt(t^*)-R(t^*)&\leq& |R(t^*)-\Rt(t^*)|\leq  K \times \mathcal{L}(t^*)\\
\Rt(t^*)&\leq& R(t^*)+ K \mathcal{L}(t^*).
\end{eqnarray*}
Since $\tt$ is the minimizer of $\Rt$, we know $\Rt(\tt)\leq \Rt(t^*)$. Combining that with the equation above, we have: \begin{equation}\label{e4}
    \Rt(\tt)\leq R(t^*)+  K \mathcal{L}(t^*).
\end{equation}
Analogously, using \eqref{RRtBound} for $\tt$, we have:
\begin{eqnarray*}
R(\tt)-\Rt(\tt)&\leq& |R(\tt)-\Rt(\tt)|\leq  K \mathcal{L}(\tt)\\
R(\tt)&\leq& \Rt(\tt)+  K \mathcal{L}(\tt).
\end{eqnarray*}
Since $t^*$ is the minimizer of $R$, we know $R(t^*)\leq R(\tt)$. Combining that with the equation above, we have: \begin{equation}\label{e5}
    R(t^*)\leq \Rt(\tt)+  K \mathcal{L}(\tt).
\end{equation}
Combining \eqref{e2} with \eqref{e4}, we have
\begin{equation}
R(t^*) +  K \mathcal{L}(t^*) <  \Rt(\tt) \leq R(t^*)+  K \mathcal{L}(t^*),
\end{equation}
which is a contradiction since one cannot have a number strictly less than itself. Similarly, combining \eqref{e3} and \eqref{e5}, we have 
\begin{equation}
    \Rt(\tt)+ K \mathcal{L}(\tt) < R(t^*)\leq \Rt(\tt)+ K \mathcal{L}(\tt),
\end{equation}
which again leads to a contradiction. 

These contradictions, which occur whether $\Rt(\tt)\geq R(t^*)$ or $R(t^*)\geq \Rt(\tt)$, imply that our assumption, namely \eqref{eq:assume}, is incorrect, thus, we have its opposite, 
$|R(t^*)-\Rt(\tt)|\leq K$.
\end{proof}

\subsection{Proof of Theorem~\ref{thm:weighted_sampling}}
\begin{theorem}
Given a weighted dataset $D = \{(\x_i, y_i, w_i)\}_{i=1}^N$, an arbitrary positive real number $r > 0$, an arbitrary positive real number $\varepsilon > 0$, and a tree $\t$, if we sample $S = r \times N$ data points from $D$, $\tilde{D} = \{(\tilde{\x}_i, \tilde{y}_i)\}_{i=1}^S$, we have:
\begin{equation*}
    \mathbb{P}\left( | \tilde{\mathcal{L}}(\t, \tilde{\x}, \tilde{\y}) -  \mathcal{L}_{\w}(\t, \x, \y)| \geq \varepsilon \right) \leq 2 \exp\left( - \frac{2\varepsilon^2}{S^2} \right).
\end{equation*}
\end{theorem}

\begin{proof}
    For a given tree $\t$, let $X_i$ be a random variable such that $X_i = \mathds{1}[y_i \neq \hat{y}_i^T]$. Accordingly, for the sampled dataset, we have: 
    \begin{equation*}
        \tilde{\mathcal{L}} = \sum_{i = 1}^S X_i. 
    \end{equation*}
    Since for each $1 \leq i \leq S$ we have $0 \leq X_i \leq 1$, based on Hoeffding's inequality, we have:
    \begin{equation*}
        \mathbb{P}\left( | \tilde{\mathcal{L}}(\t, \tilde{\x}, \tilde{\y}) -  \mathbb{E}\left[\tilde{\mathcal{L}}(\t, \tilde{\x}, \tilde{\y}) \right]| \geq \varepsilon \right) \leq 2 \exp\left( - \frac{2\varepsilon^2}{S} \right).
    \end{equation*}
    As we discussed, $\mathbb{E}\left[\tilde{\mathcal{L}}(\t, \tilde{\x}, \tilde{\y}) \right] = \mathcal{L}_{\w}(\t, \x, \y)$. So we can easily conclude the theorem.
\end{proof}

\section{Experimental Details}

\subsection{Datasets}\label{app:datasets}
\noindent\textbf{Lalonde}:  The Lalonde dataset is from the National Supported Work Demonstration \cite{dehejia1999causal, Lalonde_reference}, a labour market experiment in which participants were randomized between treatment (on-the-job training lasting between nine months and a year) and control groups. Accordingly, for each unit $U_i$, we have a pre-treatment covariate vector $X_i$ and observed assigned treatment $Z_i$. Let $Y^1_i$ be the potential outcome if unit $U_i$ received the treatment, and $Y^2_i$ be the potential outcome if it was not treated. When a unit is treated, we do not observe the potential outcome if it was not treated and vice versa. To address this issue, we use the MALTS model~\cite{parikh2018malts} that estimates these missing values by matching. MALTS gives us an estimation of the conditional average treatment effect, which we call $TE$. Since our weighted decision tree is designed for classification, we classify participants into three groups -- ``should be treated,'' ``should be treated if budget allows,'' and ``should not be treated'' -- based on their estimated conditional average treatment effect ($TE$). We selected features \textit{age, education, black, hispanic, married, nodegree, re75} for the pre-treatment covariate vector. Here, re75 is earnings in 1975. Then we labelled the data points as $2, 1,$ and $0$ if the estimated treatment effect is larger than $2000$, between $-5000$ and $2000$, and less than $-5000$, respectively. Here, the penalty for each misclassification is defined as follows:

\[   
\text{cost} = 
     \begin{cases}
       \text{0,} &\quad\text{if we correctly classify them},\\
       200 + 3 \times age &\quad\text{if their label is 0 and we misclassify,} \\
       100 + 3 \times age &\quad\text{if their label is 1 and we misclassify,} \\
       300  &\quad\text{if their label is 2 and we misclassify.} \\ 
     \end{cases}
\]

We linearly scale the above costs to the range from 1 to 100, and in the case of data-duplication, we round them to integers and treat them as weights of the dataset.  Figure~\ref{fig:hist_weights} shows the distribution of weights. 

\begin{figure*}[t]
    \centering
    \includegraphics[width=0.32\textwidth]{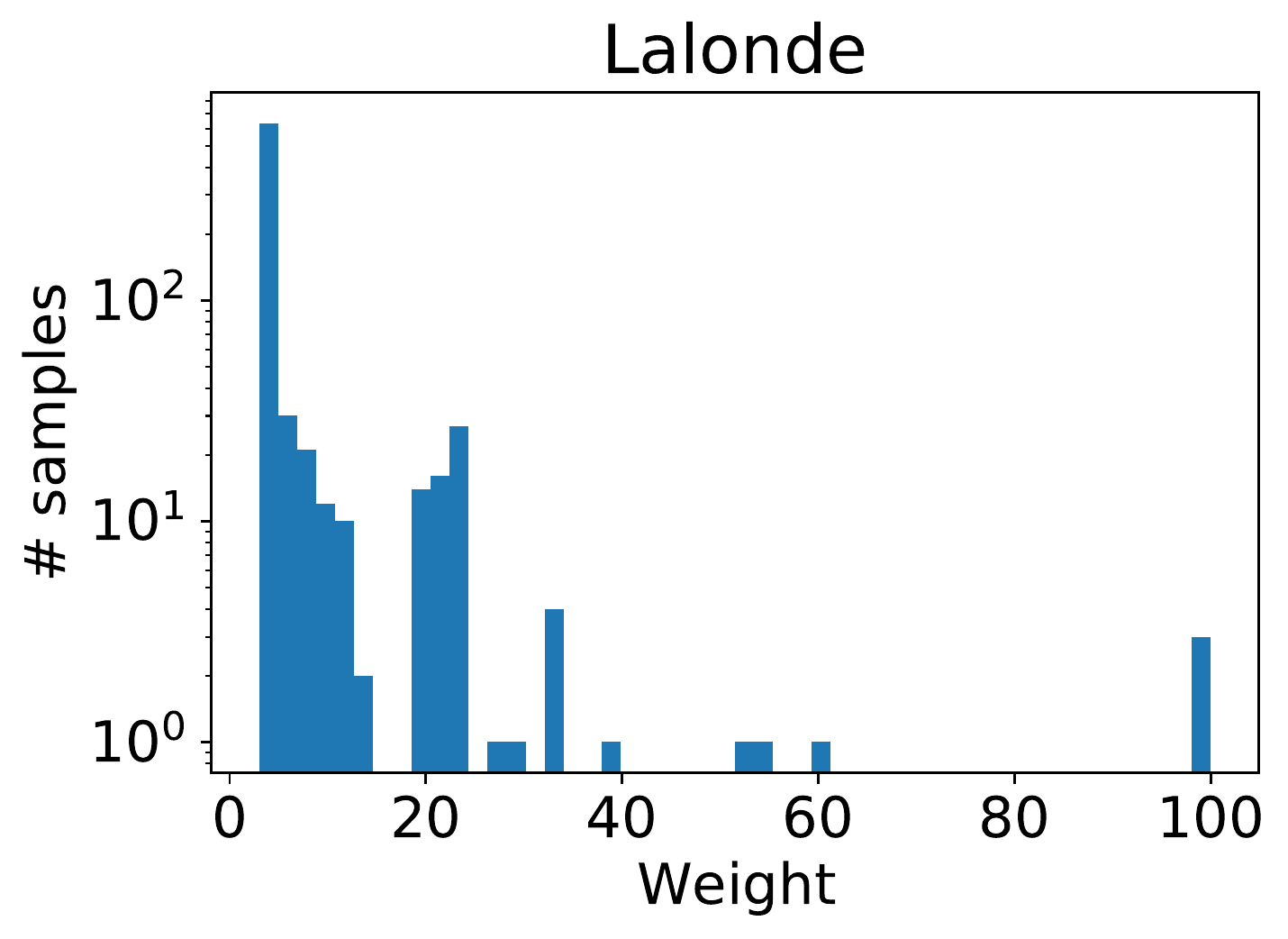}
    \includegraphics[width=0.32\textwidth]{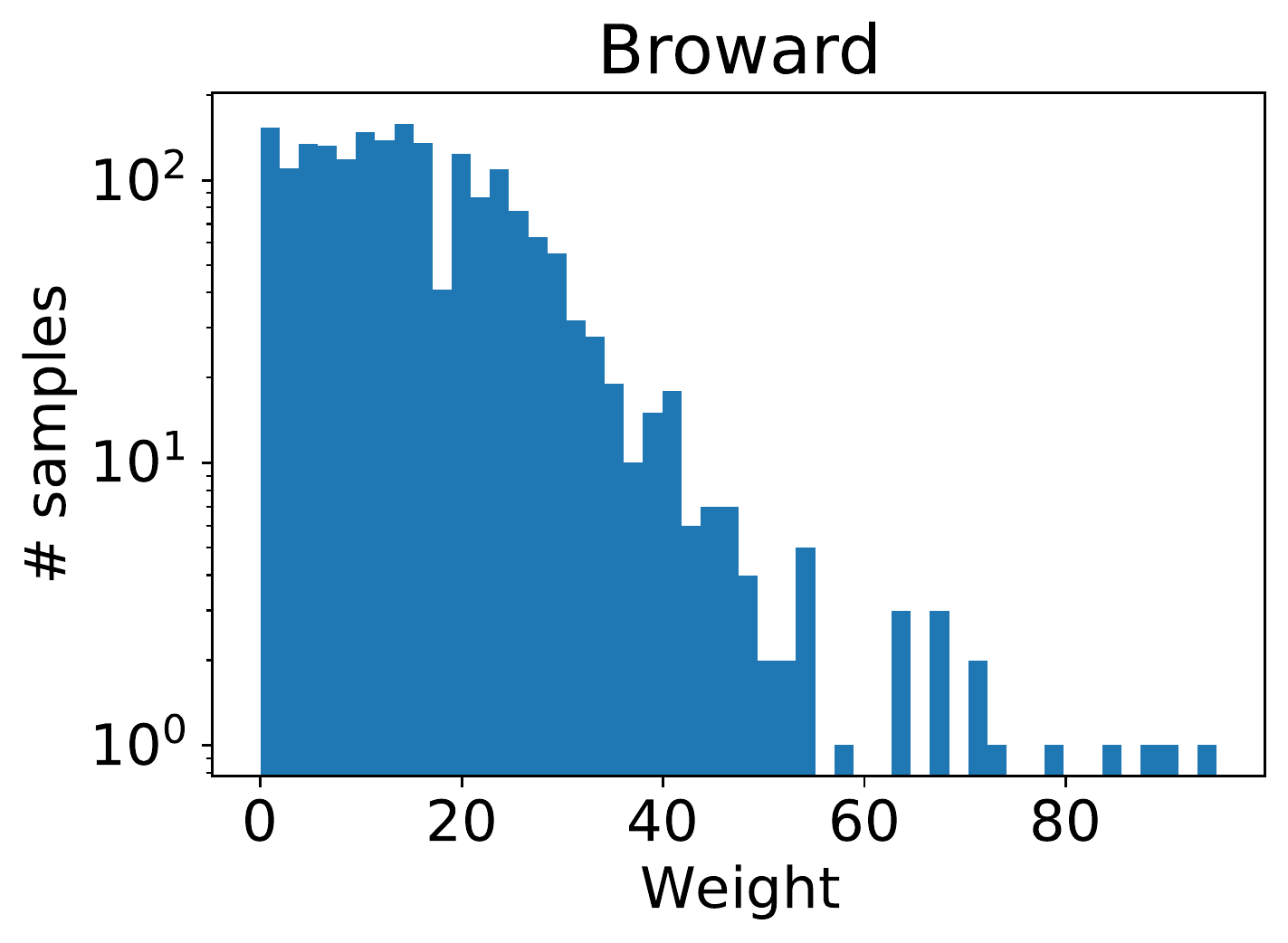}
    \includegraphics[width=0.32\textwidth]{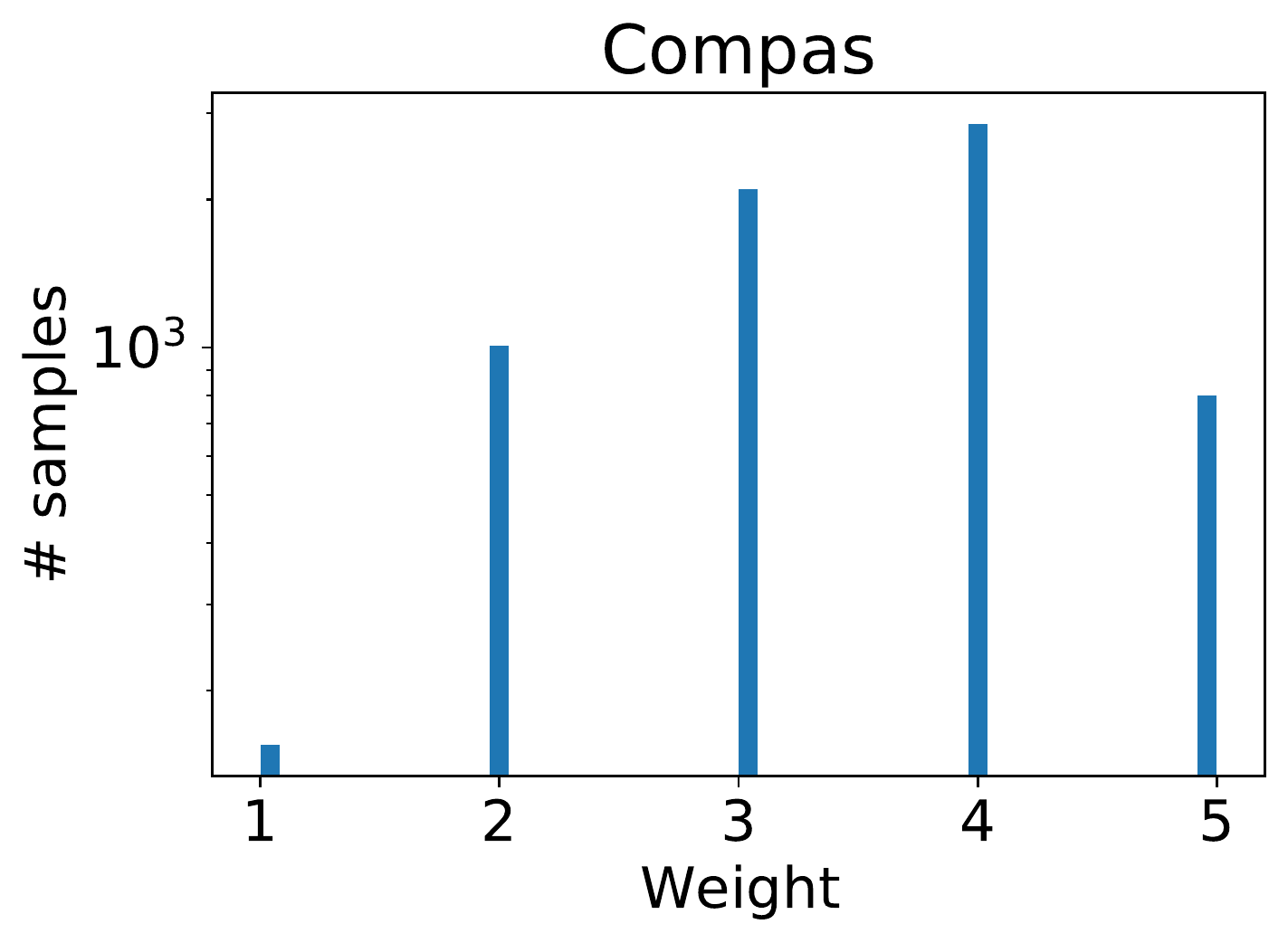}\\
    \includegraphics[width=0.32\textwidth]{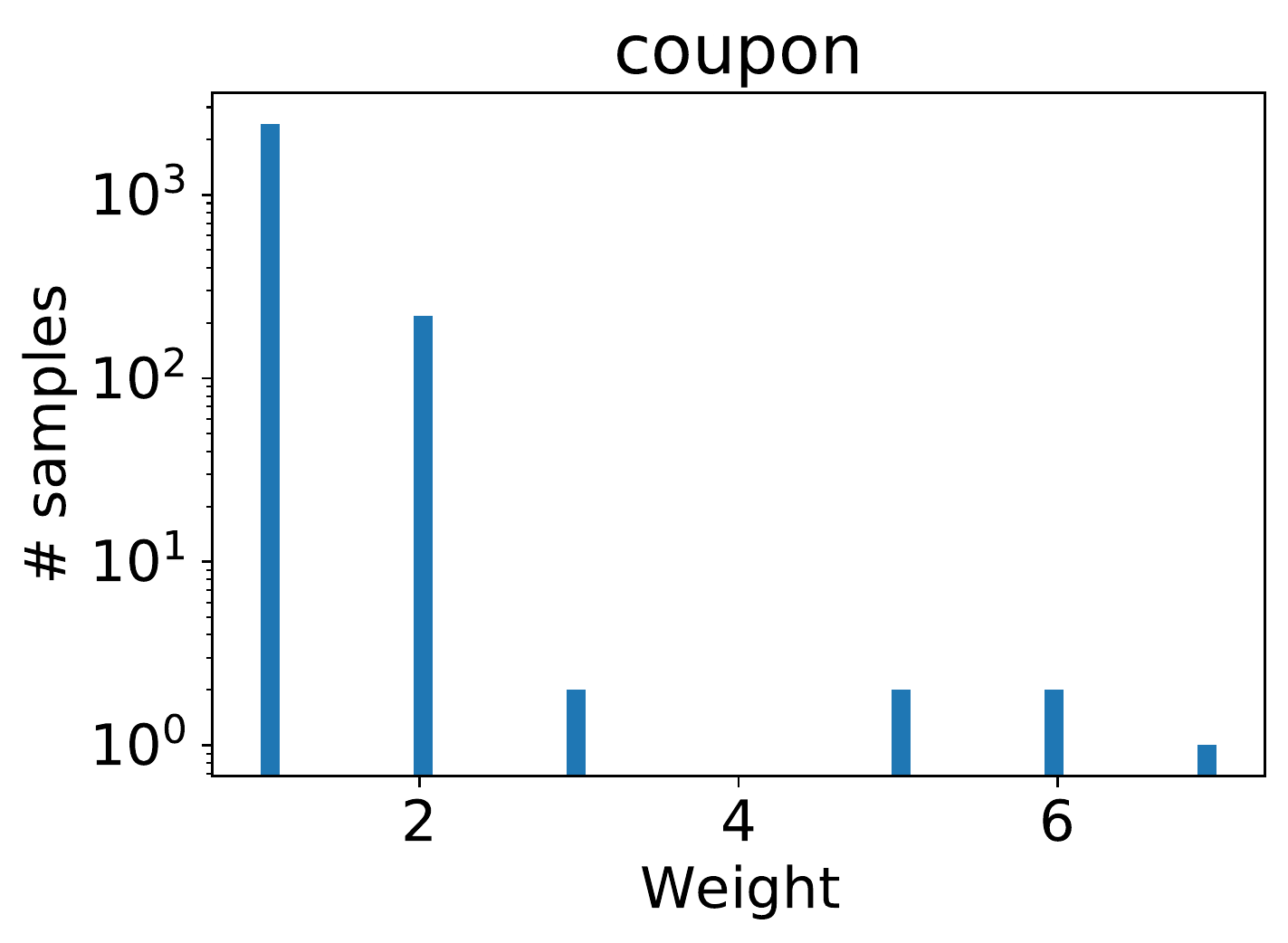}
    \includegraphics[width=0.32\textwidth]{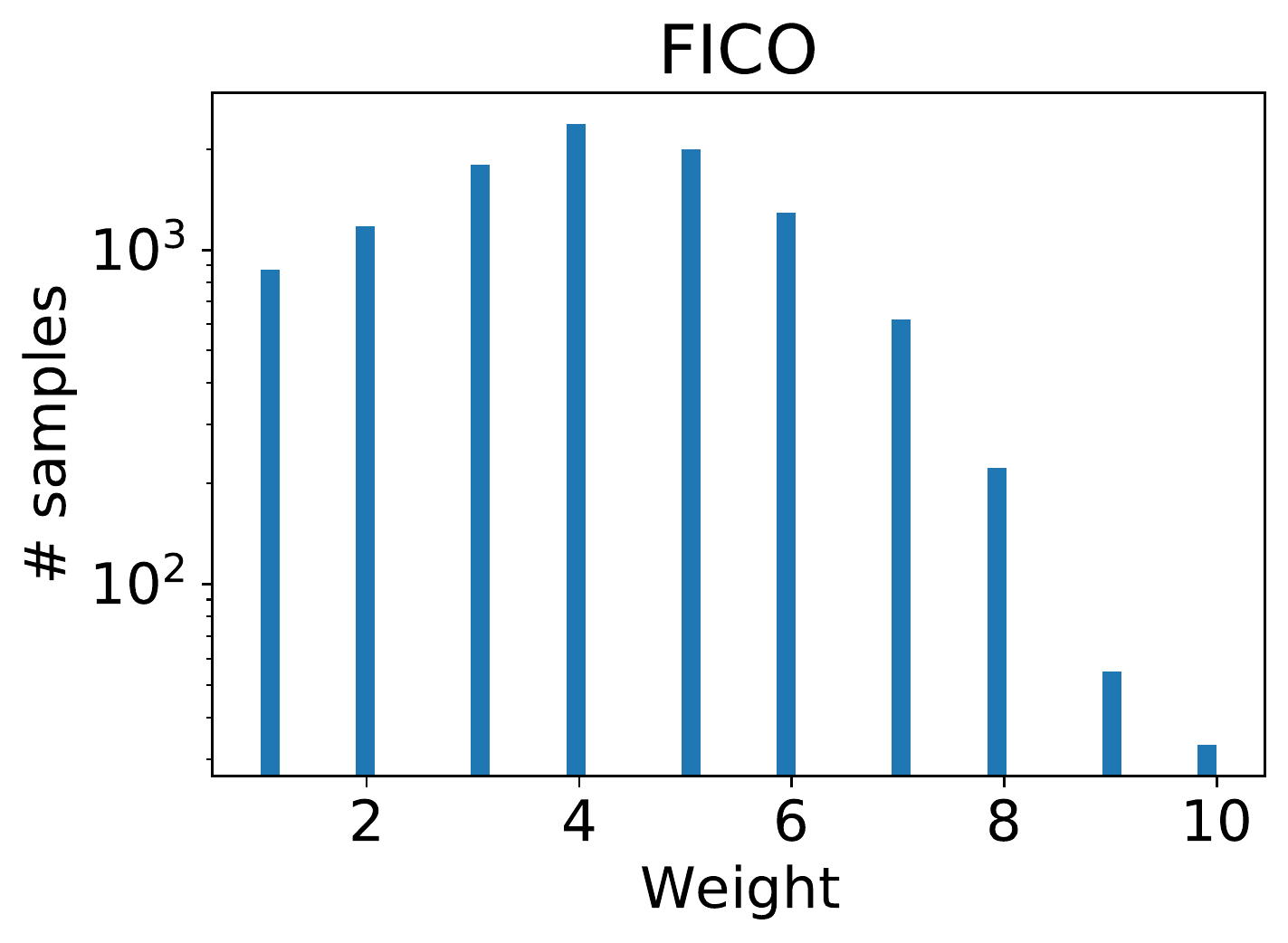}
    \includegraphics[width=0.32\textwidth]{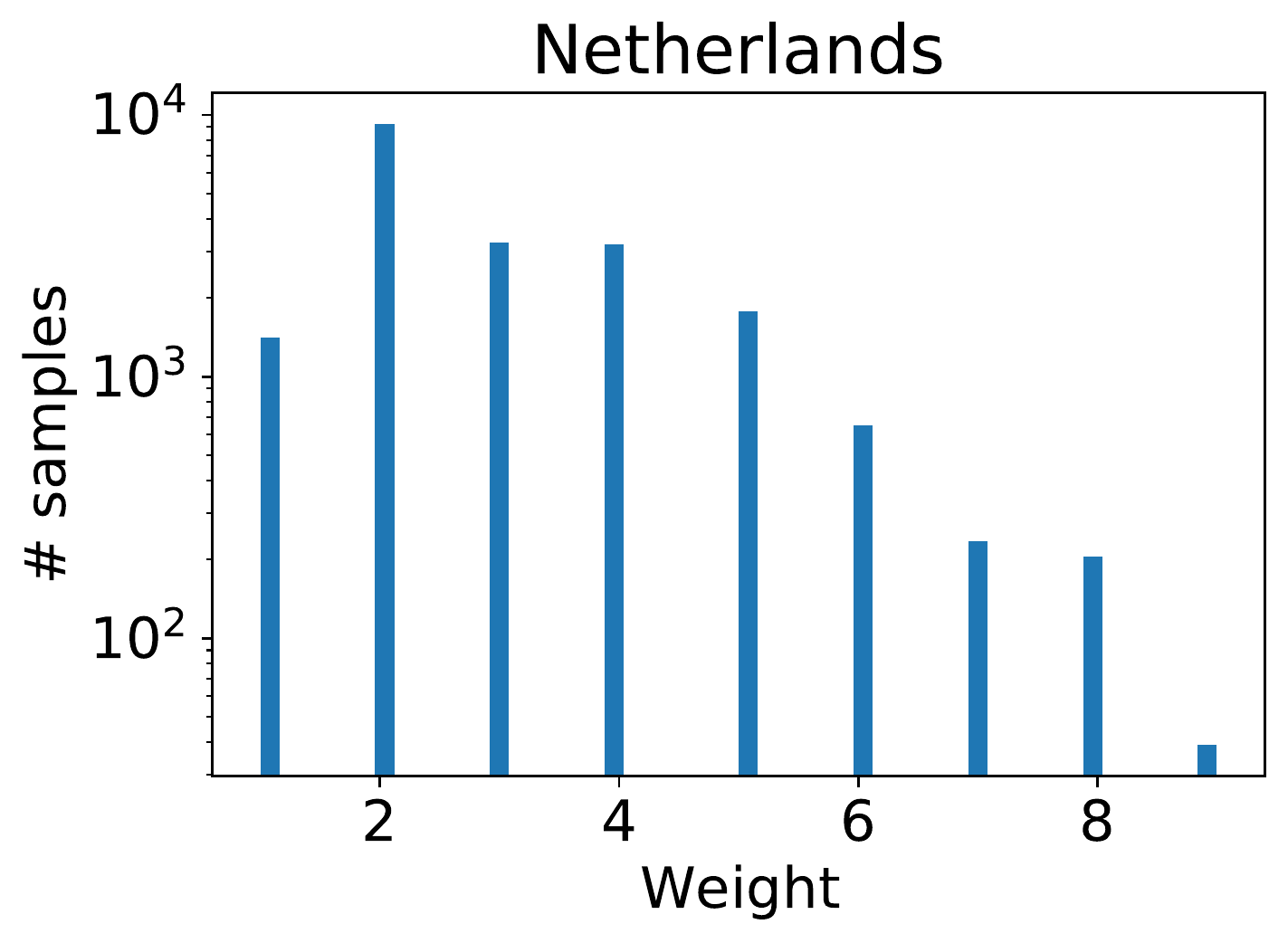}
    \includegraphics[width=0.32\textwidth]{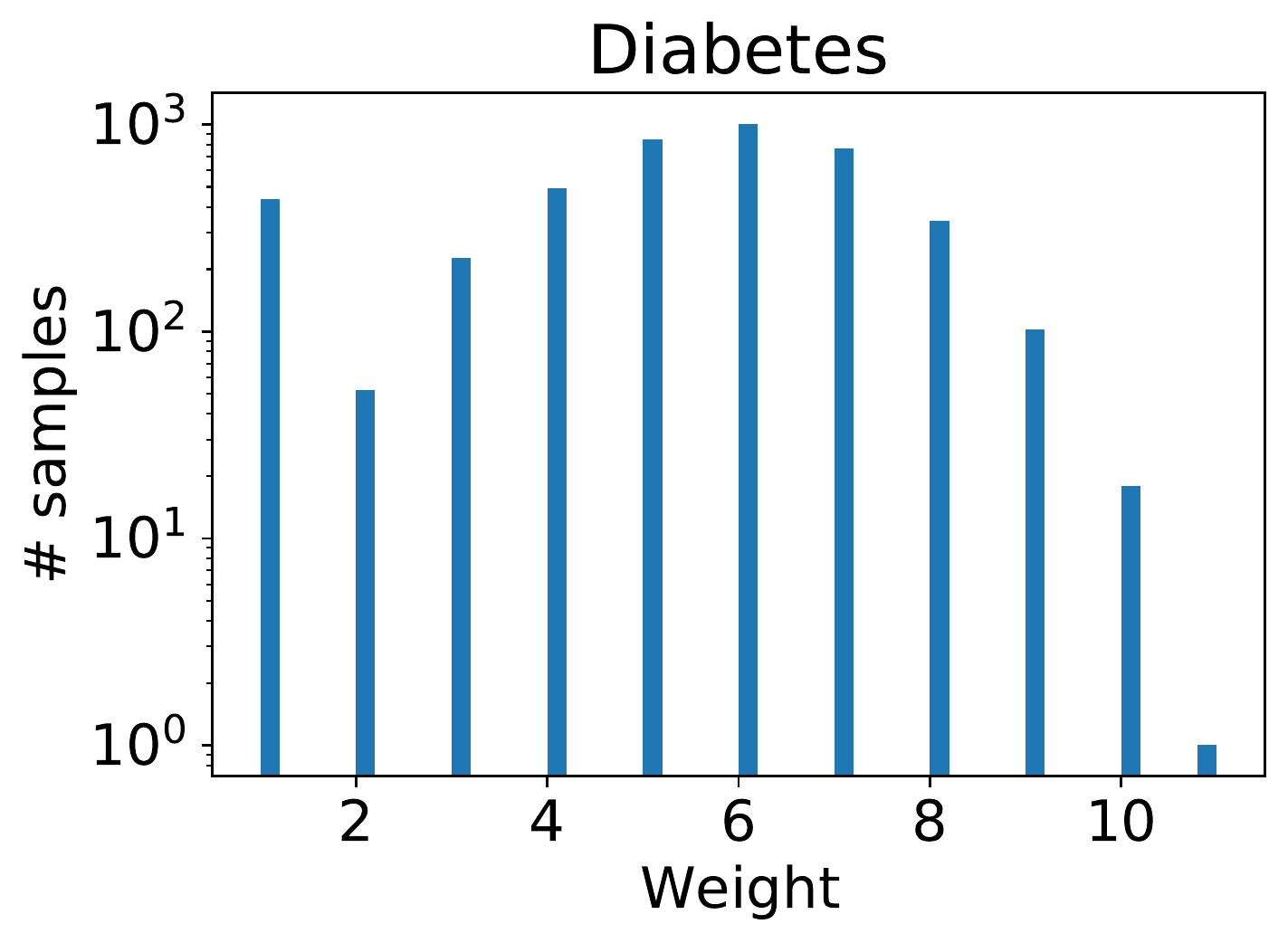}
    \caption{Distribution of weights.}
    \label{fig:hist_weights}
\end{figure*}

\noindent\textbf{Broward}: In this dataset \cite{wang2020pursuit}, we predict whether defendants have any type of charge (for which they were eventually convicted) within two years from the current charge/release date. We selected features \textit{sex, age\_at\_current\_charge, age\_at\_first\_charge,
p\_charges, p\_incarceration, p\_probation, p\_juv\_fel\_count, p\_felprop\_viol, p\_murder, p\_felassault, p\_misdeassault, p\_famviol, p\_sex\_offense, p\_weapon, p\_fta\_two\_year, p\_fta\_two\_year\_plus, current\_violence, current\_violence20, p\_pending\_charge, p\_felony, p\_misdemeanor, p\_violence, total\_convictions, p\_arrest, p\_property, p\_traffic, p\_drug, p\_dui, p\_domestic, p\_stalking, p\_voyeurism, p\_fraud, p\_stealing, p\_trespass, six\_month, one\_year, three\_year, and five\_year} and the label \textit{general\_two\_year}. We use the inverse propensity scores as the weights to balance the number of samples for each value of feature ``sex.'' Figure~\ref{fig:hist_weights} shows the distribution of weights. 

\noindent\textbf{COMPAS}: In this dataset~\cite{LarsonMaKiAn16}, we predict whether individuals are arrested within two years of release. We selected features \textit{sex, age, juv\_fel\_count, juv\_misd\_count, juv\_other\_count, priors\_count, and c\_charge\_degree} and the label \textit{two\_year\_recid}. We use inverse propensity scores as the weights to balance the number of samples for each value of feature ``sex.'' 

\noindent\textbf{Coupon}: In this dataset~\cite{wang2015or}, we predict whether a customer will accept a coupon for takeaway food or a cheap restaurant depending on their coupon usage history, current conditions while driving, and coupon expiration time. We selected features \textit{destination, passanger, weather, temperature, time, expiration, gender, age, maritalStatus, Childrennumber, education, occupation, income, Bar, CoffeeHouse, CarryAway, RestaurantLessThan20, Restaurant20To50, toCoupon\_GEQ15min, toCoupon\_GEQ25min, direction\_same} and the label \textit{Y}, and removed observations with missing values. We then used one-hot encoding to transform these categorical features into binary features. We use inverse propensity scores as the weights to balance the number of samples for each value of feature ``destination.''

\noindent\textbf{FICO}: In the FICO dataset \cite{competition}, we predict whether an individual will default on a loan. We use this dataset without preprocessing, and use propensity scores as the weights to balance the number of samples for each value of feature ``ExternalRiskEstimate.''

\noindent\textbf{Netherlands}: The prediction task in this dataset is to see whether defendants have any type of charge within four years \cite{tollenaar2013method}. We translated the feature names from Dutch to English and then selected features \textit{sex, country of birth, log \# of previous penal cases, age in years, age at first penal case, offence type, 11-20 previous case, $>$20 previous case, and age squared} and the label recidivism\_in\_4y. We use inverse propensity scores as the weights to balance the number of samples for each value of feature "sex".

\noindent\textbf{Diabetes}:
The Diabetes dataset~\cite{Diabetes_dataset} includes the data of 10 years of clinical care at 130 US hospitals, that has over 50 features representing patient and hospital outcomes. In this dataset we predict whether an individual will readmit to the hospital or not. We selected features \textit{race, gender, age, medical-specialty, diag1, diag2, diag3, max-glu-serum, A1Cresult, metformin, repaglinide, nateglinide, chlorpropamide, glimepiride, acetohexamide, glipizide, glyburide, tolbutamide, pioglitazone, rosiglitazone, acarbose, miglitol, troglitazone, tolazamide, examide, citoglipton, insulin, glyburide-metformin, glipizide-metformin, glimepiride-pioglitazone, metformin-rosiglitazone, metformin-pioglitazone, change, diabetesMed}. 
We use inverse propensity scores as the weights to balance the number of samples for each value of feature "gender".

\subsection{Configurations}\label{app:exp_detail}
Table~\ref{tab:tree_config} lists the configurations used for each dataset when training decision trees. For \GHOUL, DL8.5, and CART, we set the depth limit from 2 to 5. Also, the last column shows the different values of the regularizer, $\lambda$, for different datasets. Also, we have used different numbers of samples in our weighted samplinng approach. We vary the percentage of samplig from $100\%$ to $600\%$.

\begin{table}[t]
    \caption{Configurations used to train decision tree models}
\resizebox{0.5\textwidth}{!} {
    \centering
    \begin{tabular}{c|c|l}
    \toprule
    Dataset & Depth Limit & Regularizer ($\lambda$)\\\hline\hline
    Lalonde & 2, 3, 4, 5, 6  & 0.1, 0.05, 0.02, 0.01 \\ \hline 
    Broward & 2, 3, 4, 5, 6 & 0.005, 0.002, 0.001, 0.0005 \\\hline
    Coupon & 2, 3, 4, 5, 6 & 0.005, 0.002, 0.001, 0.0005\\\hline
    Diabetes & 2, 3, 4, 5, 6 & 0.005, 0.002, 0.001, 0.0005\\\hline
    COMPAS & 2, 3, 4, 5, 6 & 0.005, 0.002, 0.001, 0.0005, 0.0002\\\hline
    FICO & 2, 3, 4, 5, 6 & 0.01, 0.005, 0.002, 0.001, 0.0005\\\hline
    Netherlands & 2, 3, 4, 5, 6 & 0.01, 0.005, 0.002, 0.001 \\
    \bottomrule
    \end{tabular}
}
\vspace{5ex}
    \label{tab:tree_config}
\end{table}

\subsection{Evaluation Platform}\label{app:machine_property}
 All reported times are from a 32-core dual Intel E5-2683 v4 Broadwell processor running at 2.1 Ghz, with approximately 125 GB of available memory. We ran all tests single-threaded (i.e., we used only one of the 32 cores) on the Cedar cluster of Compute Canada. For the experiment reported in Figure~\ref{fig:run_time_p}, we also ran on a 2.6GHz 6-core Intel Core i7 processor and 16GB of 2400MHz DDR4 onboard memory.

\subsection{Software Packages Used}\label{app:software}
\noindent\textbf{GOSDT and DL8.5 with Guessing}: 
For guessing technique proposed by \citet{McTavishZhong}, we use their publicly released code (https://github.com/ubc-systopia/gosdt-guesses) and (https://github.com/ubc-systopia/pydl8.5-lbguess).

\noindent\textbf{DL8.5}
For DL8.5 of \citet{aglin2020learning}, we used the implementation in their main repository (https://github.com/aia-uclouvain/pydl8.5).

\noindent\textbf{CART}: We run CART using the Python implementation from Sci-Kit Learn.

\clearpage
\section{More Experimental Results}

\subsection{Sparsity vs$.$ Test Accuracy}
\label{app:sparse_accuracy}

In Section~\ref{sec:experiments}, we discussed the results for three datasets. Figure~\ref{fig:sparsity-test-accuracy_appendix} shows the results of sparsity vs$.$ test accuracy for all datasets.  \GHOUL{} produced excellent test accuracy with a small number of leaves and compared to other decision tree models, achieves higher accuracy for every level of sparsity.

\begin{figure*}
\includegraphics[width=0.247\textwidth]{main_Lalonde_NumberofLeaves_TestAccuracy.pdf}
\includegraphics[width=0.247\textwidth]{main_Broward_NumberofLeaves_TestAccuracy.pdf}
\includegraphics[width=0.247\textwidth]{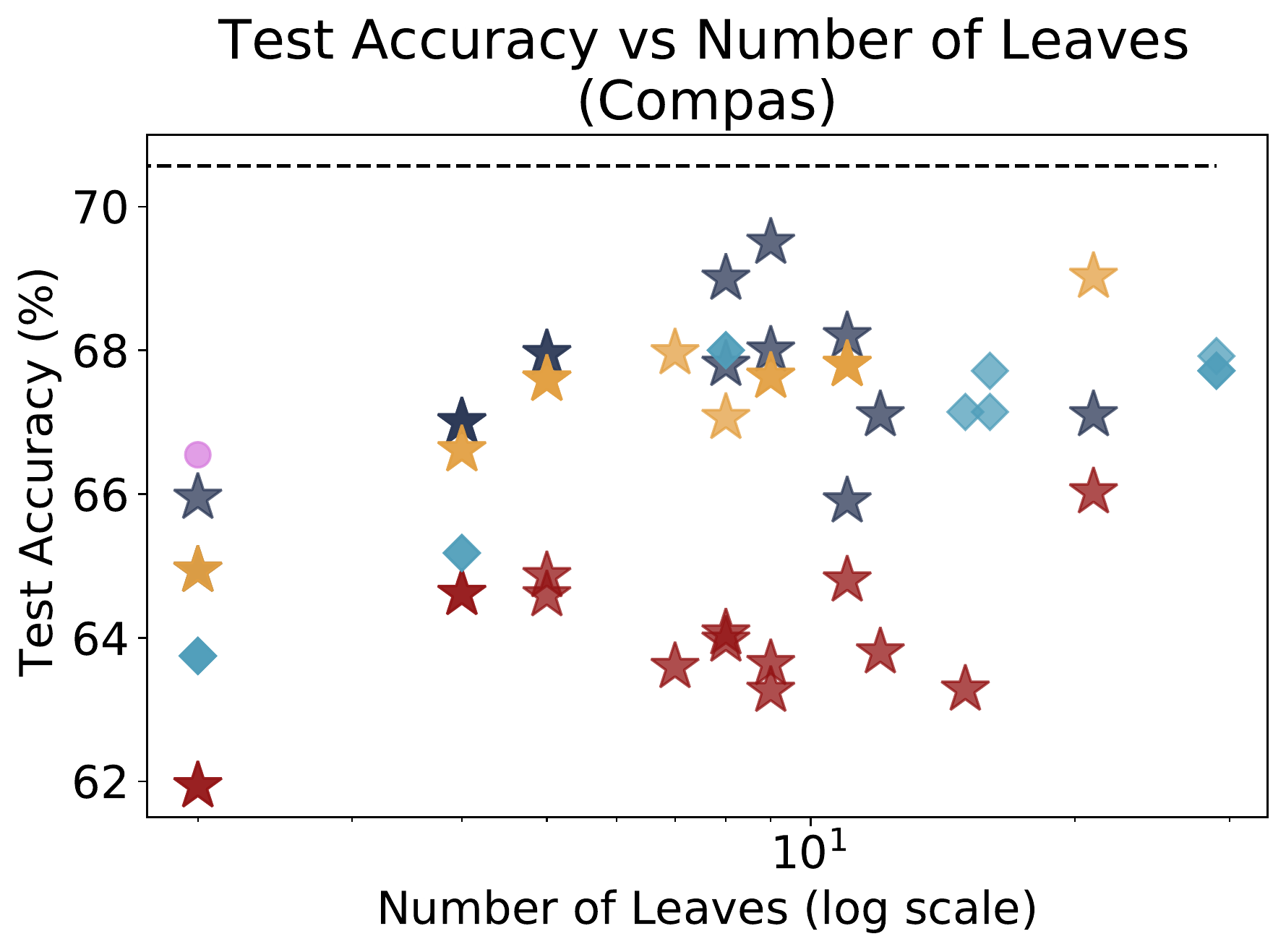}
\includegraphics[width=0.247\textwidth]{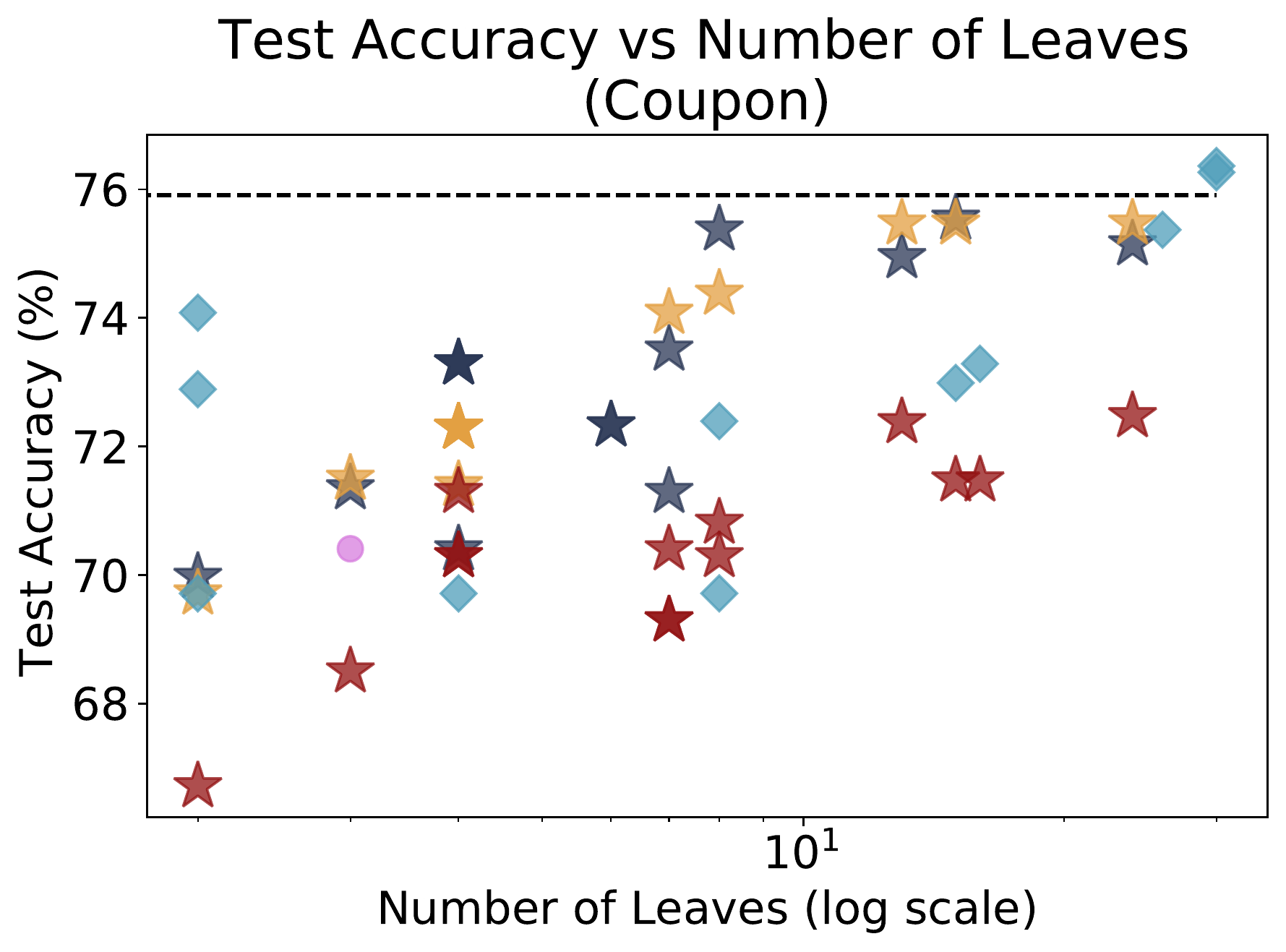}\\
\includegraphics[width=0.247\textwidth]{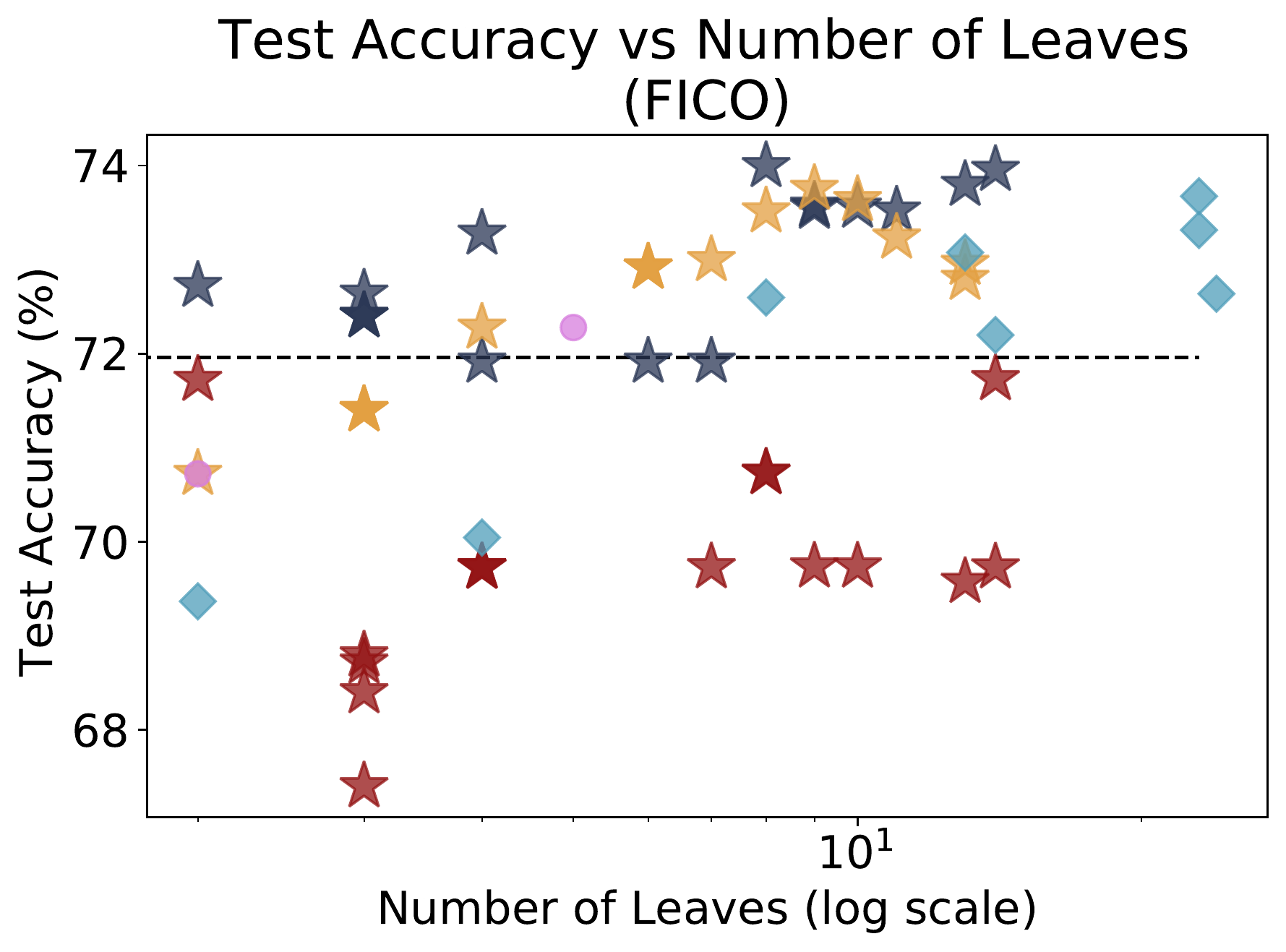}
\includegraphics[width=0.247\textwidth]{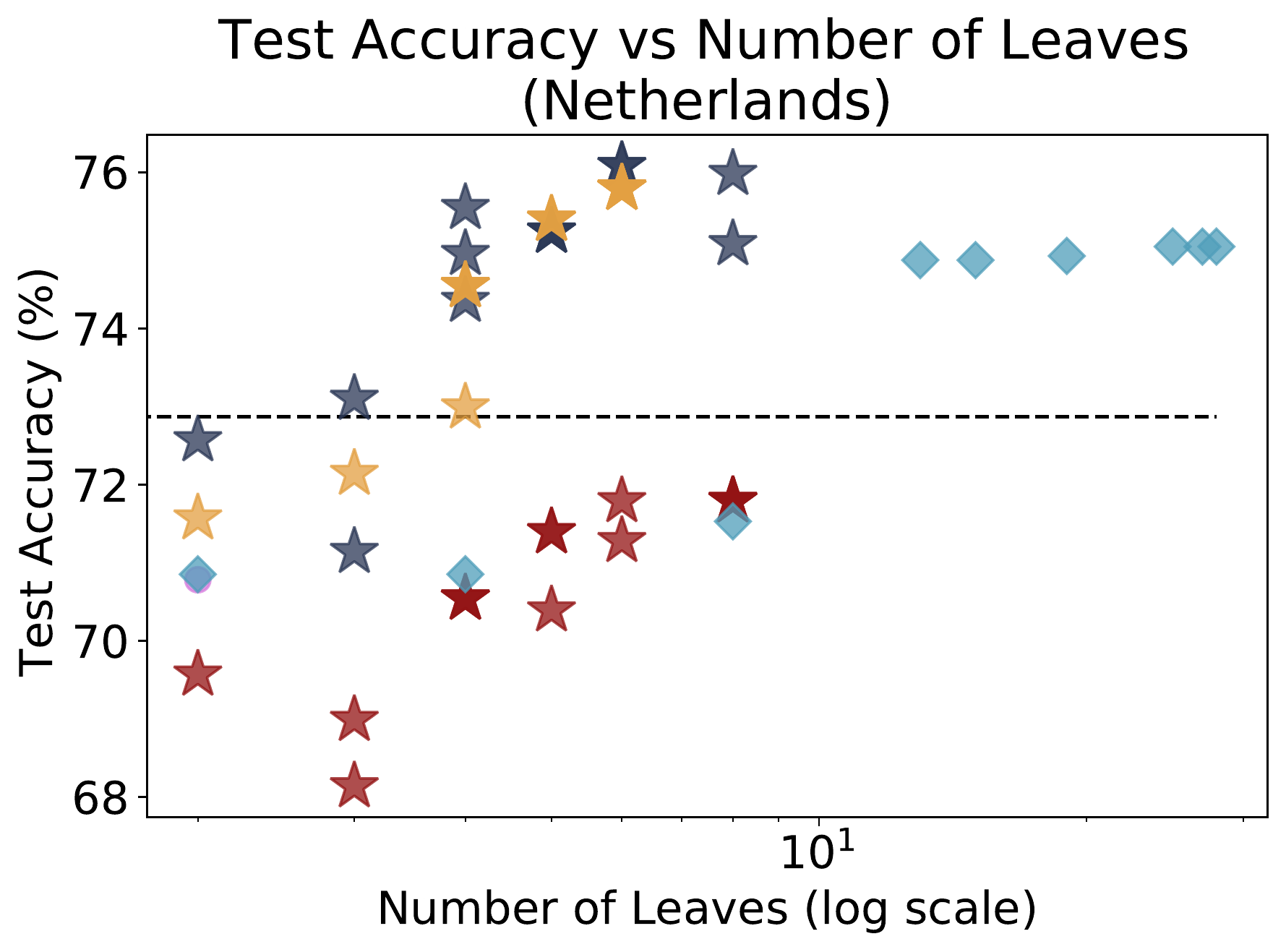}
\includegraphics[width=0.247\textwidth]{main_Diabetes_NumberofLeaves_TestAccuracy.pdf}
\includegraphics[width=0.247\textwidth]{legend.pdf}
\caption{\textbf{Sparsity vs. test accuracy: } All methods but CART and GBDT use guessed thresholds. GBDT and DL8.5 use data duplication. DL8.5 frequently times out, so there are fewer markers for it. \GHOUL{} achieves the highest training accuracy for every level of sparsity. \GHOUL{} also achieves the highest test accuracy for \textbf{\textit{almost}} every level of sparsity. }
\label{fig:sparsity-test-accuracy_appendix}
\end{figure*}

\subsection{Training Time vs$.$ Test Accuracy}
\label{app:training_accuracy}

In Section~\ref{sec:experiments}, we discussed the results of training time and test accuracy of our methods and baselines on three datasets. Figure~\ref{fig:acc-time-test-appendix} reports the results of training time vs$.$ test accuracy on all datasets. While the training times of \GHOUL{} and CART are almost the same, \GHOUL{} achieves the highest test accuracy in almost all cases. DL8.5 struggled with the 1-hour termination condition for all datasets except Lalonde; because DL8.5 did not solve to optimality, it was outperformed by both CART and \GHOUL. 

\begin{figure*}
\includegraphics[width=0.247\textwidth]{main_Lalonde_TrainingTime_TestAccuracy.pdf}
\includegraphics[width=0.247\textwidth]{main_Broward_TrainingTime_TestAccuracy.pdf}
\includegraphics[width=0.247\textwidth]{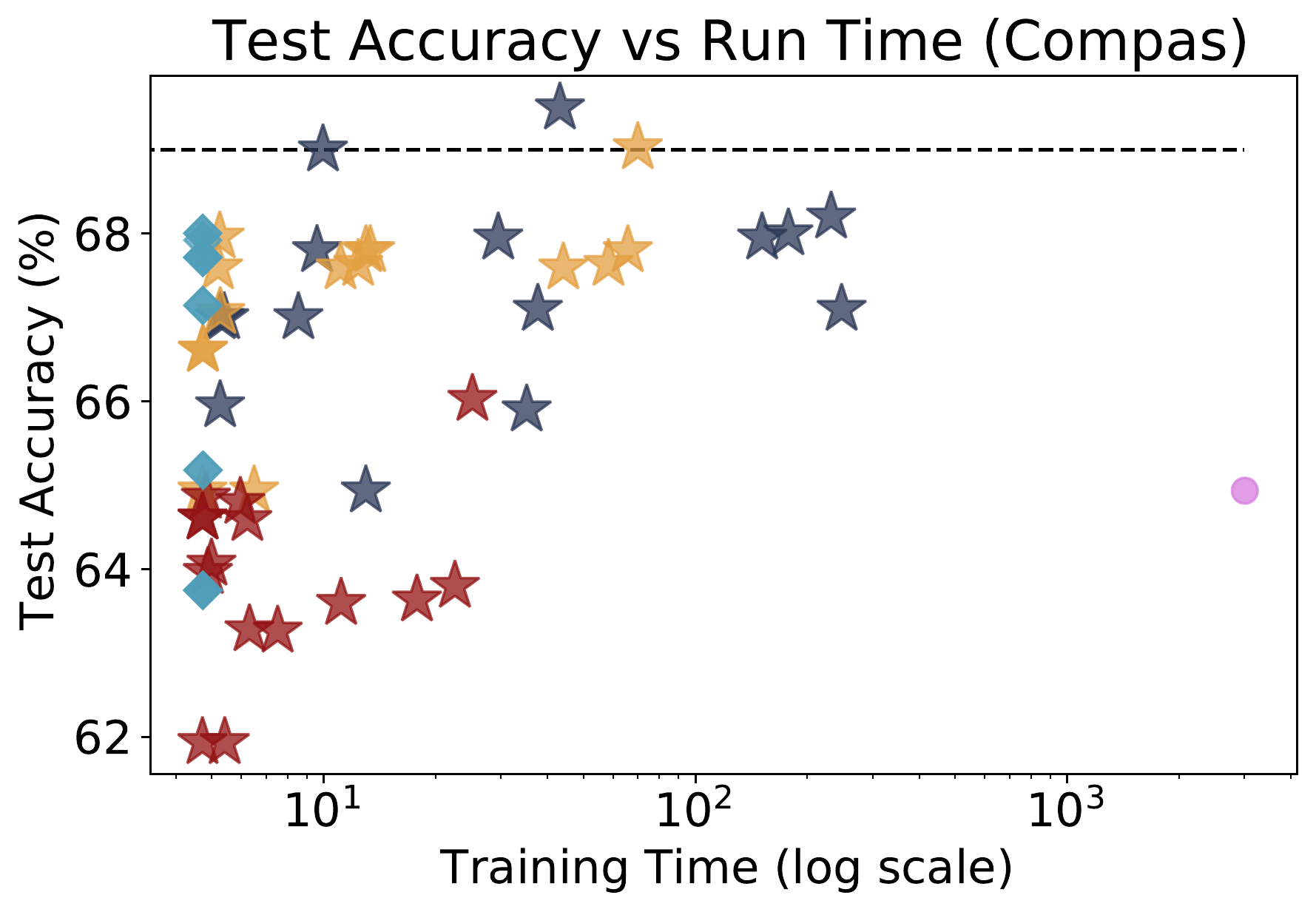}
\includegraphics[width=0.247\textwidth]{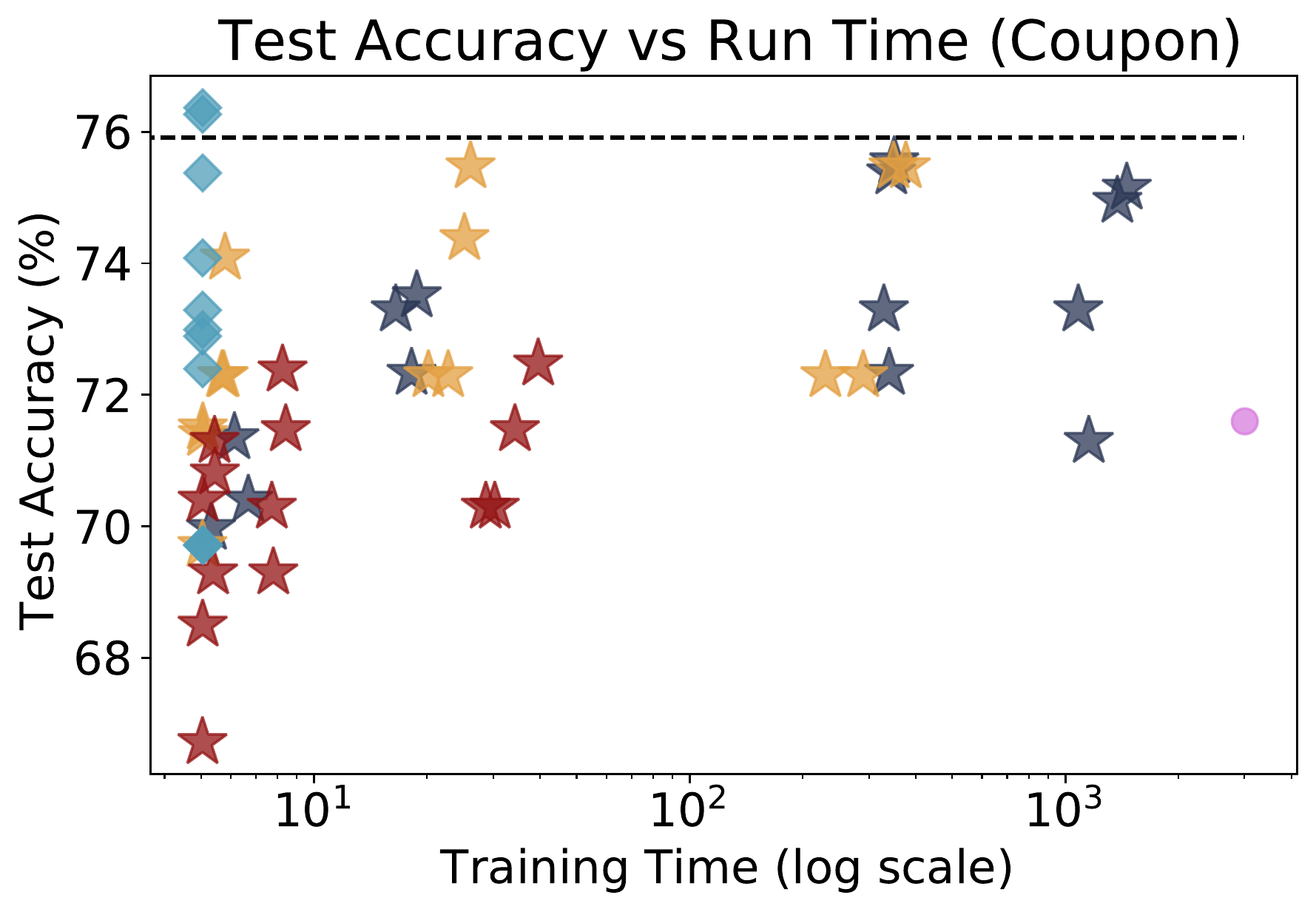}\\
\includegraphics[width=0.247\textwidth]{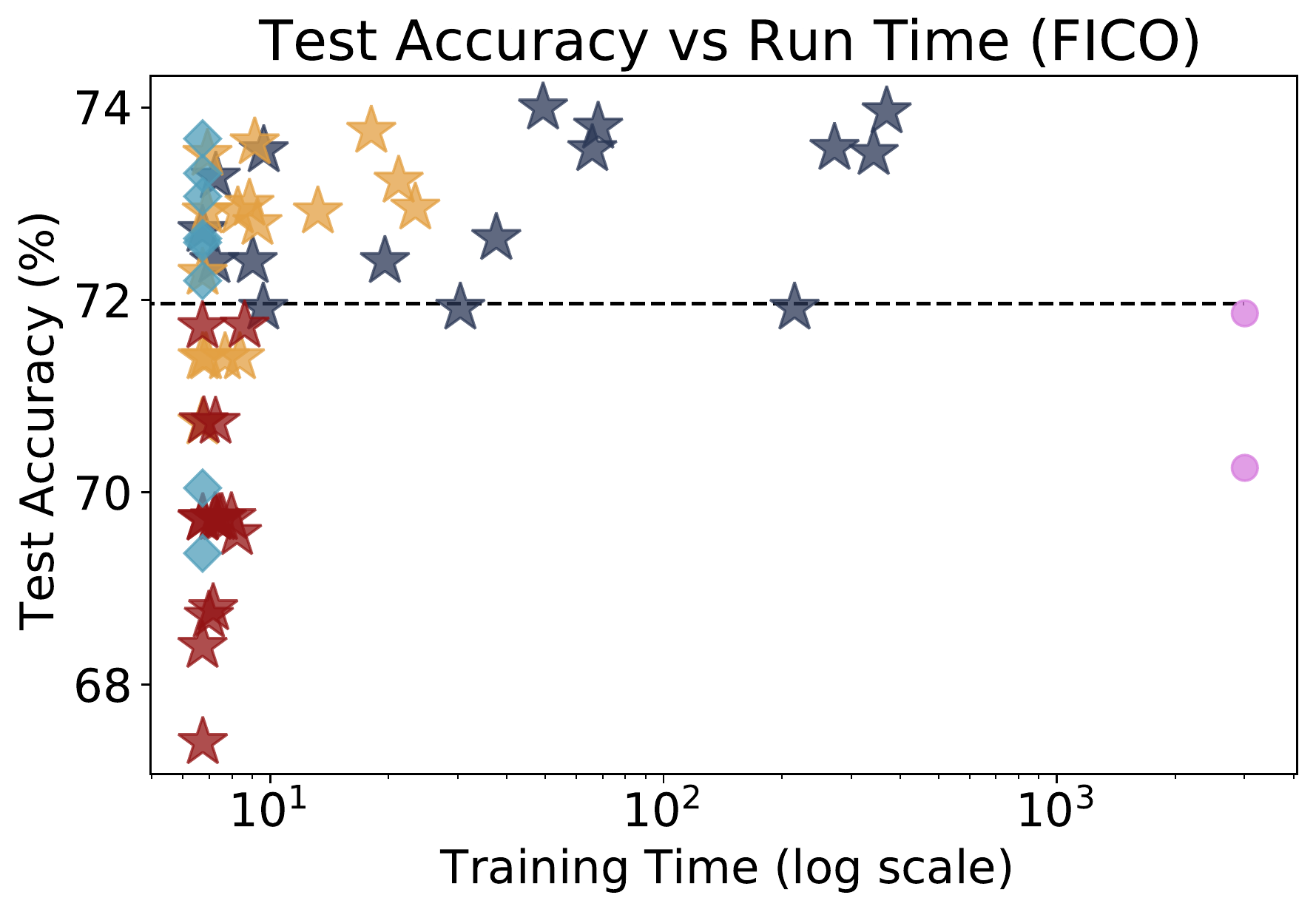}
\includegraphics[width=0.247\textwidth]{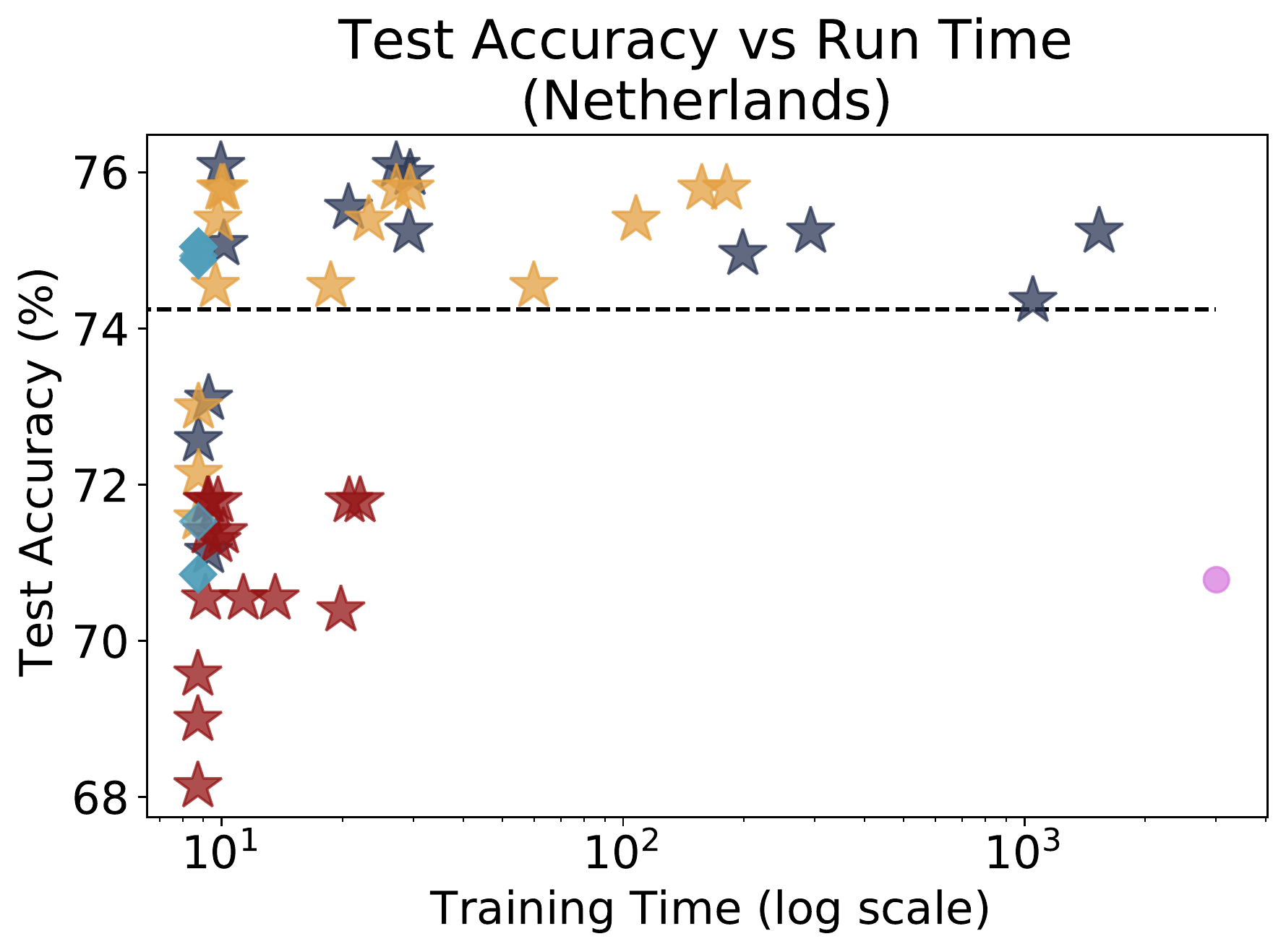}
\includegraphics[width=0.247\textwidth]{main_Diabetes_TrainingTime_TestAccuracy.pdf}
\includegraphics[width=0.247\textwidth]{legend.pdf}
\caption{\textbf{Training time vs$.$ test accuracy:} All methods but CART and GBDT use guessed thresholds. GBDT and DL8.5 use data duplication. DL8.5 frequently times out, so there are fewer markers for it. While CART is the fastest algorithm, \GHOUL uses its additional runtime to produce models with higher accuracy and generalize better.}
\label{fig:acc-time-test-appendix}
\end{figure*}

\subsection{Effect of the Number Samples}
\label{app:effect-samples}
Intuitavely, the accuracy of weighted sampling approach increase with the sample size. To show this effect, Figure~\ref{fig:sample-size} shows the effect of sample size on accuracy. Results indicates that increasing the sample size lets us to have a more accurate model.

\begin{figure*}
\includegraphics[width=0.245\textwidth]{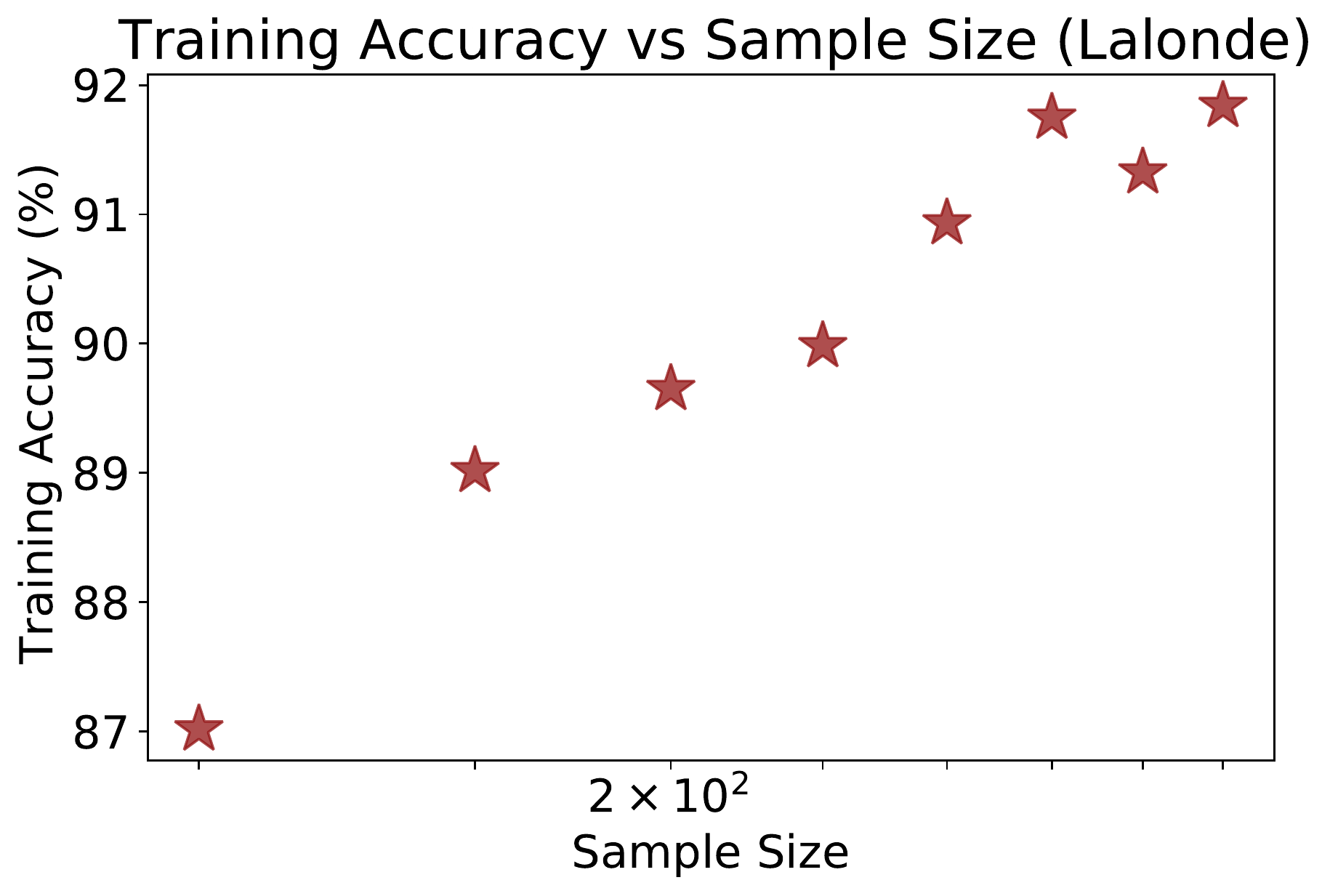}
\includegraphics[width=0.245\textwidth]{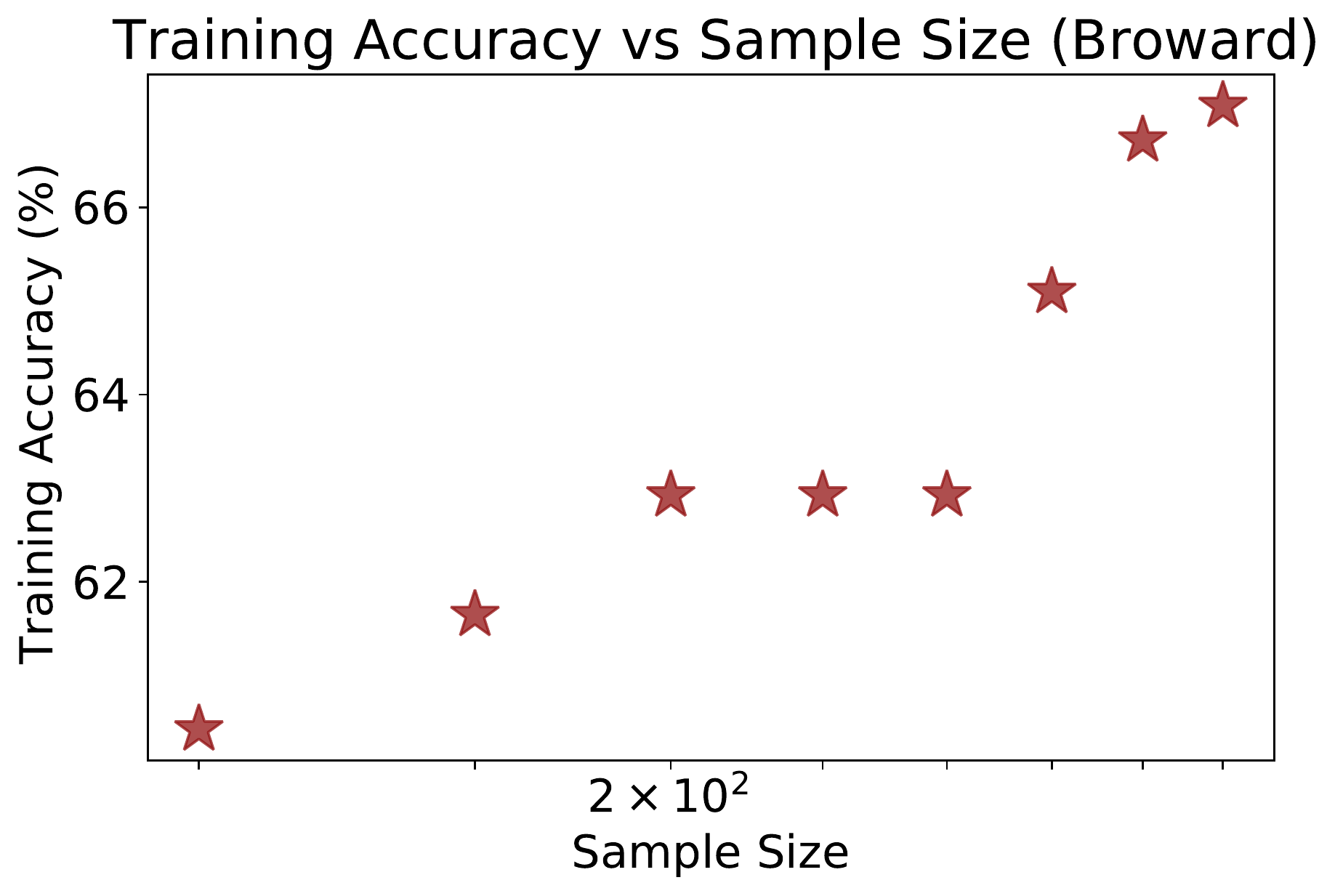}
\includegraphics[width=0.247\textwidth]{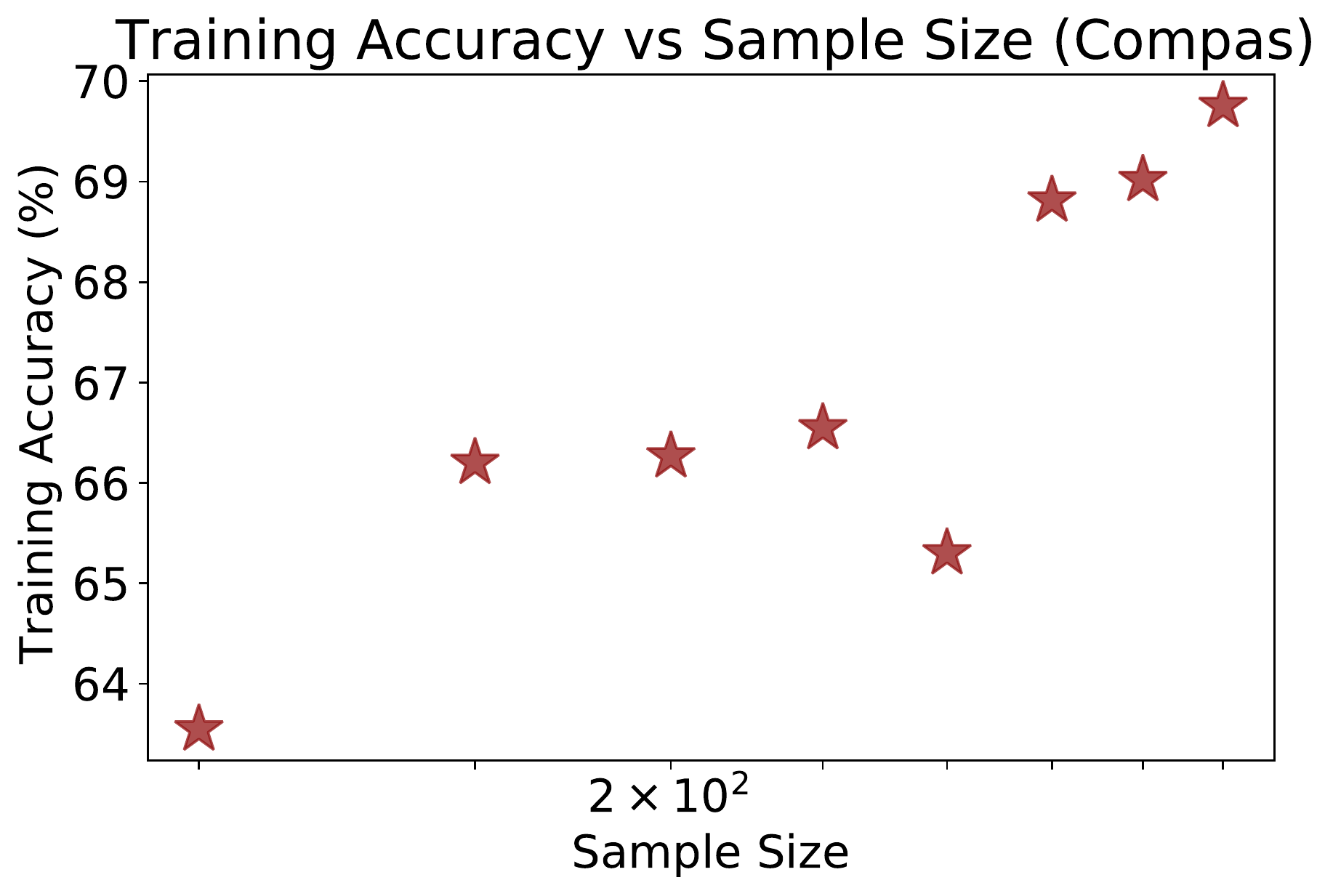}
\includegraphics[width=0.247\textwidth]{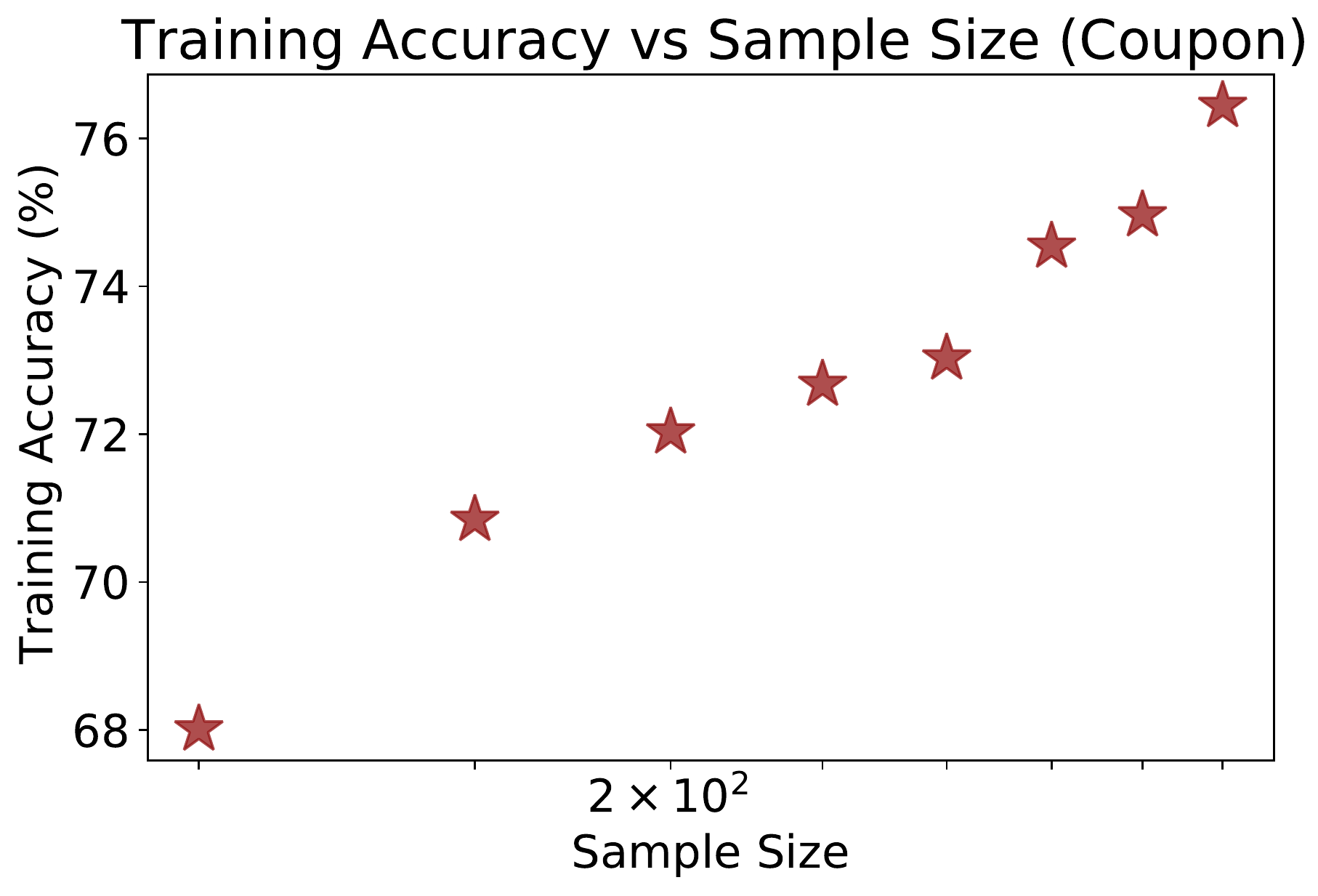}\\
\includegraphics[width=0.247\textwidth]{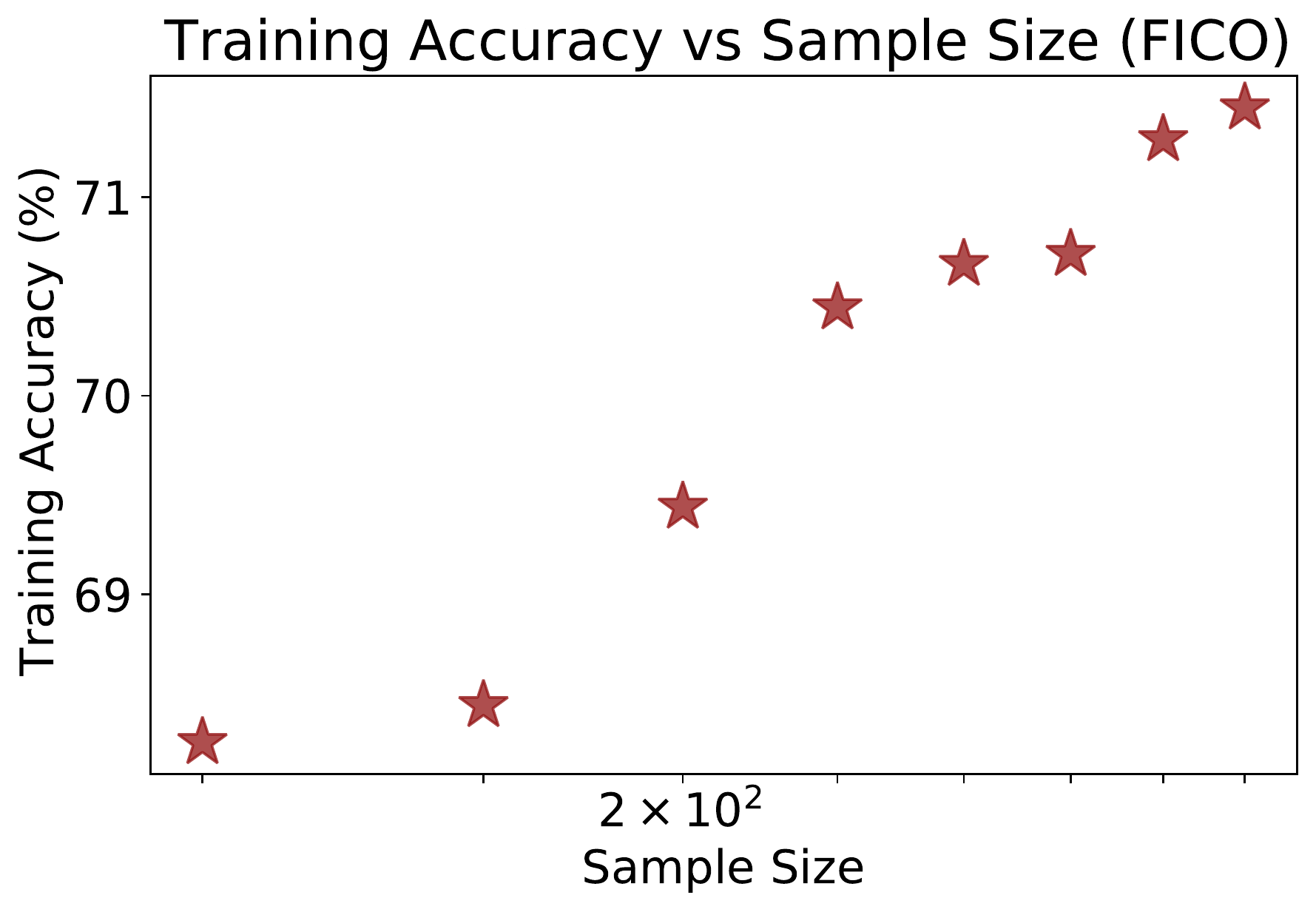}
\includegraphics[width=0.247\textwidth]{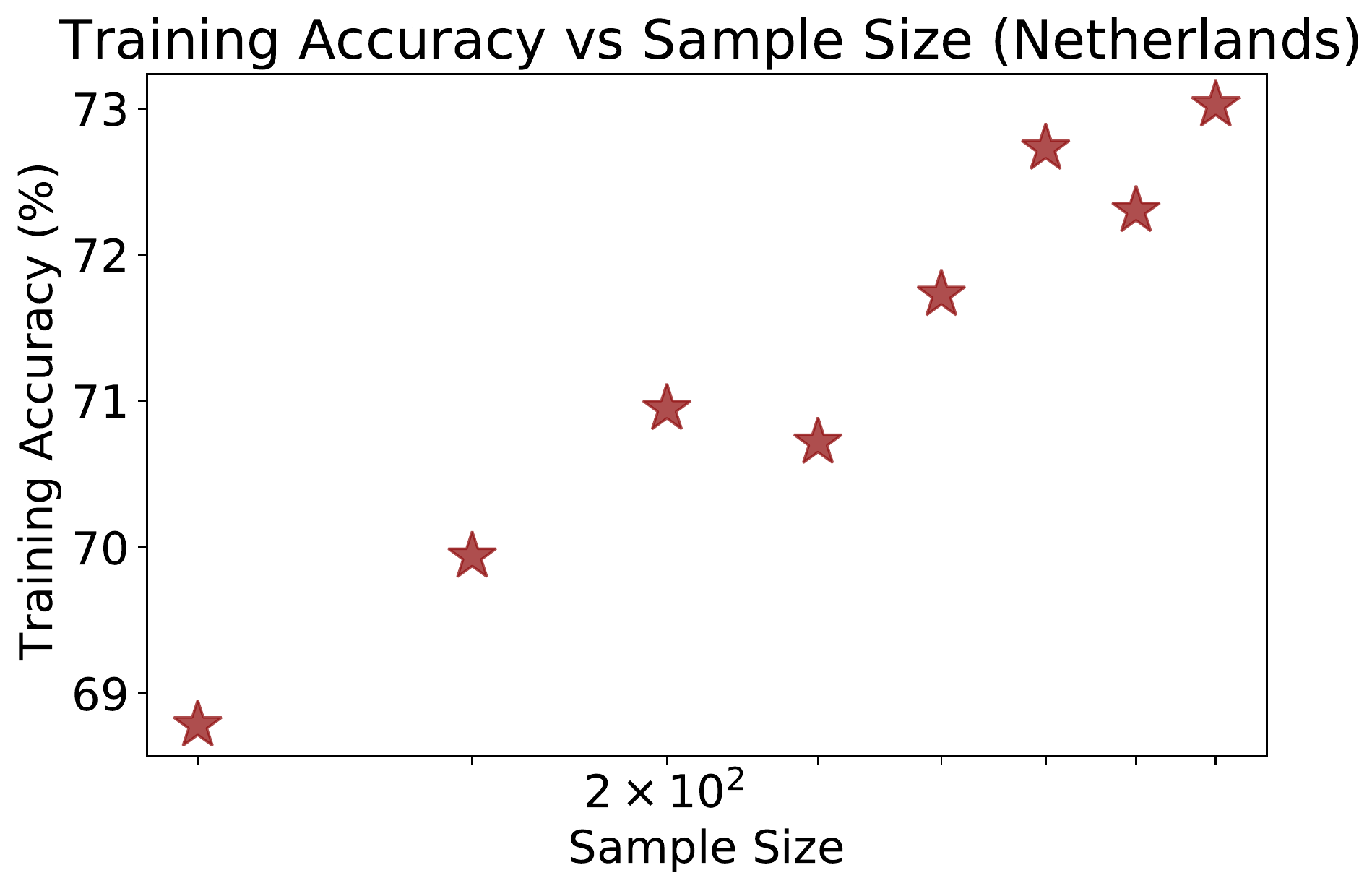}
\includegraphics[width=0.245\textwidth]{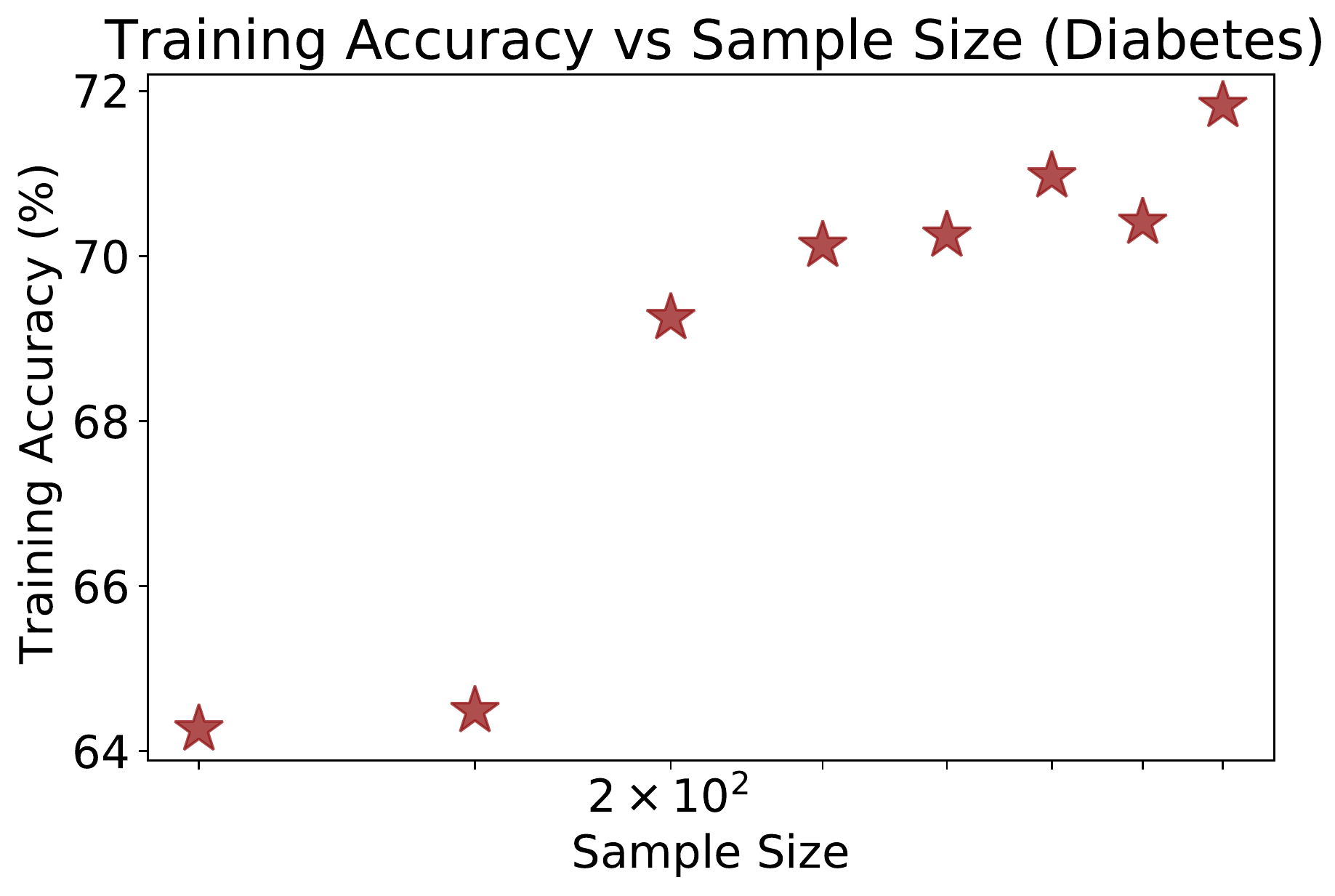}
\vspace{-1ex}
\caption{\textbf{Training accuracy vs$.$ sample size:} the accuracy of weighted sampling approach increase with the sample size.}
\label{fig:sample-size}
\end{figure*}

\subsection{Lalonde Case Study}\label{app:lalonde_trees}
As discussed in Section~\ref{sec:experiments} (Lalonde Case Study), to show the effectiveness of our approach for real-world problems, we conducted a case study on the Lalonde dataset. Trees produced by \GHOUL{} with depth limits of 2, 3, 4, and 5 are found in Figures \ref{fig:lalonde_depth2}-\ref{fig:lalonde_depth5}.

\begin{figure}
    \centering
\begin{frame}{}  
\scalebox{0.75}{
\begin{forest}
    [ $re75 \le 897.4097$ [ $education \le 11.5$ [ $class$ [ $1$ ] ] [ $class$ [ $2$ ] ] ] [ $re75 \le 21497.51$ [ $class$ [ $1$ ] ] [ $class$ [ $0$ ] ] ] ]
\end{forest}
}
\end{frame}
    \caption{The tree generated by weighted \GHOUL{} (depth limit 2) on the Lalonde dataset.}
    \label{fig:lalonde_depth2}
\end{figure}
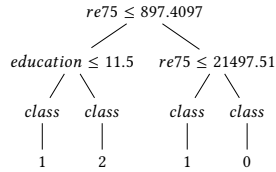

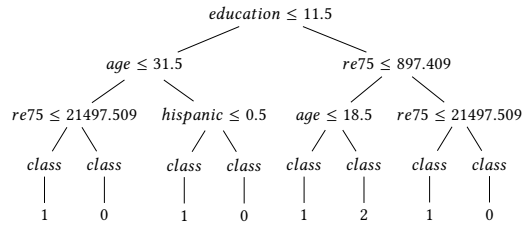
\begin{figure}
    \centering
\begin{frame}{}  
\scalebox{0.75}{
\begin{forest}
    [ $education \le 11.5$ [ $age \le 31.5$ [ $re75 \le 21497.509$ [ $class$ [ $1$ ] ] [ $class$ [ $0$ ] ] ] [ $hispanic \le 0.5$ [ $class$ [ $1$ ] ] [ $class$ [ $0$ ] ] ] ] [ $re75 \le 897.409$ [ $age \le 18.5$ [ $class$ [ $1$ ] ] [ $class$ [ $2$ ] ] ] [ $re75 \le 21497.509$ [ $class$ [ $1$ ] ] [ $class$ [ $0$ ] ] ] ] ]
\end{forest}
}
\end{frame}
    \caption{The tree generated by weighted \GHOUL{} (depth limit 3) on the Lalonde dataset.}
    \label{fig:lalonde_depth3}
\end{figure}

\begin{figure}
    \centering
\begin{frame}{}  
\scalebox{0.75}{
\begin{forest}
    [ $re75 \le 897.4097$ [ $age \le 32.0$ [ $age \le 18.5$ [ $class$ [ $1$ ] ] [ $education \le 10.5$ [ $class$ [ $1$ ] ] [ $class$ [ $2$ ] ] ] ] [ $education \le 11.5$ [ $married \le 0.5$ [ $class$ [ $1$ ] ] [ $class$ [ $2$ ] ] ] [ $class$ [ $2$ ] ] ] ] [ $hispanic \le 0.5$ [ $married \le 0.5$ [ $re75 \le 21497.509$ [ $class$ [ $1$ ] ] [ $class$ [ $0$ ] ] ] [ $re75 \le 32975.189$ [ $class$ [ $1$ ] ] [ $class$ [ $0$ ] ] ] ] [ $age \le 31.5$ [ $re75 \le 13916.55$ [ $class$ [ $1$ ] ] [ $class$ [ $0$ ] ] ] [ $education \le 11.5$ [ $class$ [ $0$ ] ] [ $class$ [ $1$ ] ] ] ] ] ]
\end{forest}
}
\end{frame}
    \caption{The tree generated by weighted \GHOUL{} (depth limit 4) on the Lalonde dataset.}
    \label{fig:lalonde_depth4}
\end{figure}
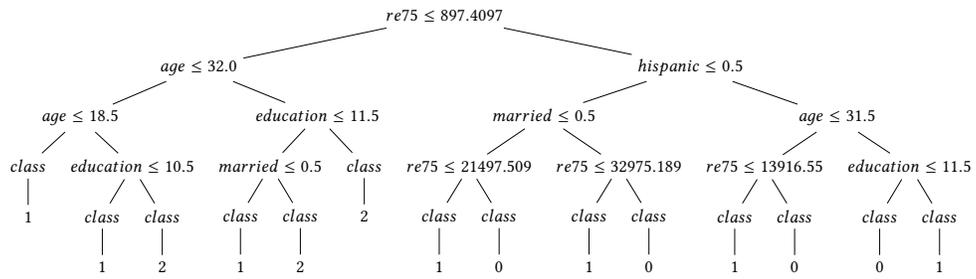

\begin{figure}
    \centering
  \hspace{-3ex}
\begin{frame}{}  
\scalebox{0.53}{
\begin{forest}
    [ $married \le 0.5$ [ $education \le 10.5$ [ $education \le 4.5$ [ $black \le 0.5$ [ $class$ [ $1$ ] ] [ $re75 \le 897.409$ [ $class$ [ $1$ ] ] [ $class$ [ $2$ ] ] ] ] [ $age \le 31.5$ [ $re75 \le 13916.55$ [ $class$ [ $1$ ] ] [ $class$ [ $0$ ] ] ] [ $hispanic \le 0.5$ [ $class$ [ $1$ ] ] [ $class$ [ $2$ ] ] ] ] ] [ $age \le 33.5$ [ $re75 \le 897.409$ [ $age \le 18.5$ [ $class$ [ $1$ ] ] [ $class$ [ $2$ ] ] ] [ $re75 \le 21497.50$ [ $class$ [ $1$ ] ] [ $class$ [ $0$ ] ] ] ] [ $education \le 11.5$ [ $black \le 0.5$ [ $class$ [ $0$ ] ] [ $class$ [ $1$ ] ] ] [ $re75 \le 897.40$ [ $class$ [ $2$ ] ] [ $class$ [ $1$ ] ] ] ] ] ] [ $re75 \le 2594.9949$ [ $age \le 24.0$ [ $education \le 10.5$ [ $class$ [ $1$ ] ] [ $black \le 0.5$ [ $class$ [ $1$ ] ] [ $class$ [ $2$ ] ] ] ] [ $education \le 9.5$ [ $age \le 31.5$ [ $class$ [ $1$ ] ] [ $class$ [ $2$ ] ] ] [ $re75 \le 1436.94$ [ $class$ [ $2$ ] ] [ $class$ [ $1$ ] ] ] ] ] [ $re75 \le 8066.70$ [ $age \le 31.5$ [ $education \le 9.5$ [ $class$ [ $2$ ] ] [ $class$ [ $1$ ] ] ] [ $class$ [ $1$ ] ] ] [ $re75 \le 13916.55$ [ $age \le 24.0$ [ $class$ [ $1$ ] ] [ $class$ [ $2$ ] ] ] [ $re75 \le 32975.18$ [ $class$ [ $1$ ] ] [ $class$ [ $0$ ] ] ] ] ] ] ]
\end{forest}
}
\end{frame}
    \caption{The tree generated by weighted \GHOUL{} (depth limit 5) on the Lalonde dataset.}
    \label{fig:lalonde_depth5}
\end{figure}
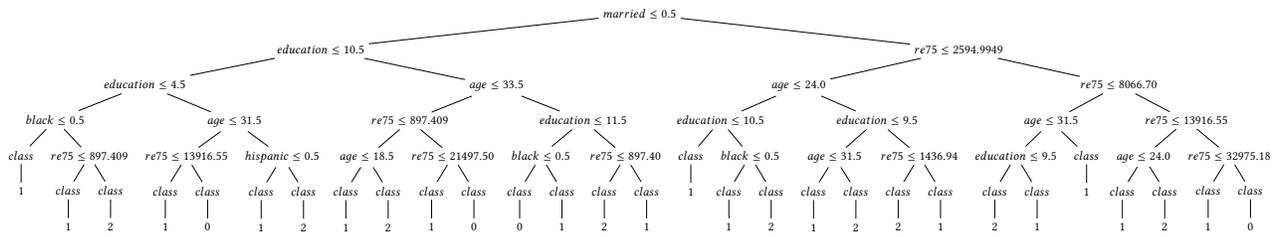

\end{document}